\documentclass{article}

\usepackage{microtype}
\usepackage{graphicx}
\usepackage{subfigure}
\usepackage{booktabs} %

\usepackage[nointegrals]{wasysym}
\usepackage{fancyvrb}

\usepackage{hyperref}

\usepackage[accepted]{icml2024}

\usepackage{amsmath}
\usepackage{amssymb}
\usepackage{mathtools}
\usepackage{amsthm}
\usepackage{bm}

\usepackage{thmtools}
\usepackage{thm-restate}
\usepackage{empheq} 

\usepackage[capitalize,noabbrev]{cleveref}

\usepackage{wrapfig}
\usepackage[shortlabels]{enumitem}

\theoremstyle{plain}
\newtheorem{theorem}{Theorem}[section]
\newtheorem{proposition}[theorem]{Proposition}
\newtheorem{lemma}[theorem]{Lemma}

\theoremstyle{definition}

\theoremstyle{remark}

\usepackage[textsize=tiny]{todonotes}

\newcommand{\E}{\mathbb{E}}
\newcommand{\R}{\mathbb{R}}

\newcommand{\Var}{\mathrm{Var}}
\newcommand{\Cov}{\mathrm{Cov}}

\newcommand{\cX}{\mathcal{X}}
\newcommand{\cY}{\mathcal{Y}}

\newcommand{\cZ}{\mathcal{Z}}

\newcommand{\bigmid}{\mathrel{\big\vert}}
\newcommand{\Bigmid}{\mathrel{\Big\vert}}
\newcommand{\middlemid}{\;\middle|\;}

\newcommand{\simiid}{\stackrel{iid}{\sim}}

\newcommand{\indep}{\perp\!\!\!\perp}

\newcommand{\PhiY}{\Phi^\theta_{\scriptscriptstyle{Y\!{|}\!X}}}
\newcommand{\PhiYY}{\Phi^\theta_{\scriptscriptstyle{Y_1\!,\!Y_2\!{|}\!X}}}

\newcommand{\pgt}{p_{\scriptscriptstyle{Y\!{|}\!X}}}
\newcommand{\pgtvec}{\bm{p}_{\scriptscriptstyle{Y\!{|}\!X}}}
\newcommand{\phaty}{\hat{p}^{\theta}_{\scriptscriptstyle{Y\!{|}\!X}}}
\newcommand{\phatprimey}{\hat{p}^{\theta'}_{\scriptscriptstyle{Y\!{|}\!X}}}

\newcommand{\sigmahat}{\hat{\bm{\Sigma}}^{\theta}}
\newcommand{\sigmahatprime}{\hat{\bm{\Sigma}}^{\theta'}}
\newcommand{\sigmacheat}{\hat{\bm{\Sigma}}^{\theta}_{\scriptscriptstyle{Y_1\!,\!Y_2\!{|}\!X}}}

\newcommand{\vhat}{\hat{V}^{\theta}}
\newcommand{\phatyy}{\hat{p}^{\theta}_{\scriptscriptstyle{Y_1\!,\!Y_2\!{|}\!X}}}
\newcommand{\phatyc}{\hat{p}^{\theta}_{\scriptscriptstyle{Y_2\!{|}\!Y_1\!,\!X}}}
\newcommand{\phatym}{\hat{p}^{\theta}_{\scriptscriptstyle{Y_1\!{|}\!X}}}
\newcommand{\phatymtwo}{\hat{p}^{\theta}_{\scriptscriptstyle{Y_2\!{|}\!X}}}

\newcommand{\Ccheat}{C^{\theta}_{\scriptscriptstyle{\!\textsc{Cheat}}}}
\newcommand{\vcheat}{\hat{V}^{\theta}_{\scriptscriptstyle{\!\textsc{Cheat}}}}

\newcommand{\Ydec}{\tilde{Y}}

\definecolor{tab10blue}{HTML}{1F77B4}
\definecolor{tab10orange}{HTML}{FF7F0E}

\definecolor{fig1red}{HTML}{cc4125}
\definecolor{fig1blue}{HTML}{4285f4}
\definecolor{fig1green}{HTML}{489d28}

\DeclareMathOperator*{\argmax}{arg\,max}

\icmltitlerunning{Experts Don't Cheat: Learning What You Don't Know by Predicting Pairs}

\begin{document}

\twocolumn[
\icmltitle{
Experts Don't Cheat: Learning What You Don't Know By Predicting Pairs
}

\icmlsetsymbol{equal}{*}

\begin{icmlauthorlist}
\icmlauthor{Daniel D. Johnson}{gdm,uoft}
\icmlauthor{Daniel Tarlow}{gdm}
\icmlauthor{David Duvenaud}{uoft}
\icmlauthor{Chris J. Maddison}{uoft}
\end{icmlauthorlist}

\icmlaffiliation{gdm}{Google DeepMind}
\icmlaffiliation{uoft}{University of Toronto, Department of Computer Science, Ontario, Canada}

\icmlcorrespondingauthor{Daniel D. Johnson}{ddjohnson@cs.toronto.edu}

\icmlkeywords{Machine Learning, ICML, epistemic uncertainty, aleatoric uncertainty, uncertainty quantification, distribution-free, calibration, confidence intervals, reliability, misspecification, underfitting, hallucination, grouping loss, nonparameteric inference, generative models, language models}

\vskip 0.3in
]

\printAffiliationsAndNotice{}  %

\begin{abstract}
Identifying how much a model $\phaty$ knows about the stochastic real-world process $\pgt$ it was trained on is important to ensure it avoids producing incorrect or ``hallucinated'' answers or taking unsafe actions.
But this is difficult
for generative models
because
probabilistic predictions do not distinguish between per-response noise (aleatoric uncertainty) and lack of knowledge about the process (epistemic uncertainty), and existing epistemic uncertainty quantification techniques tend to
be overconfident
when the model underfits.
We propose a general strategy for teaching a
model to both approximate $\pgt$ and also estimate
the remaining gaps between $\phaty$ and $\pgt$:
train it to predict \emph{pairs} of independent responses drawn from the true conditional distribution, allow it to ``cheat'' by observing one response while predicting the other, then measure how much it cheats.
Remarkably, we prove that being good at cheating (i.e. cheating whenever it improves your prediction) is 
equivalent to being \emph{second-order calibrated}, a principled extension of ordinary calibration that allows us to construct provably-correct frequentist confidence intervals for $\pgt$ and detect incorrect responses with high probability.
We demonstrate empirically that our approach accurately estimates how much models don't know across ambiguous image classification, (synthetic) language modeling, and partially-observable navigation tasks, outperforming existing techniques.
\end{abstract}

\section{Introduction}

\begin{figure}
    \centering
    \includegraphics[width=\linewidth]{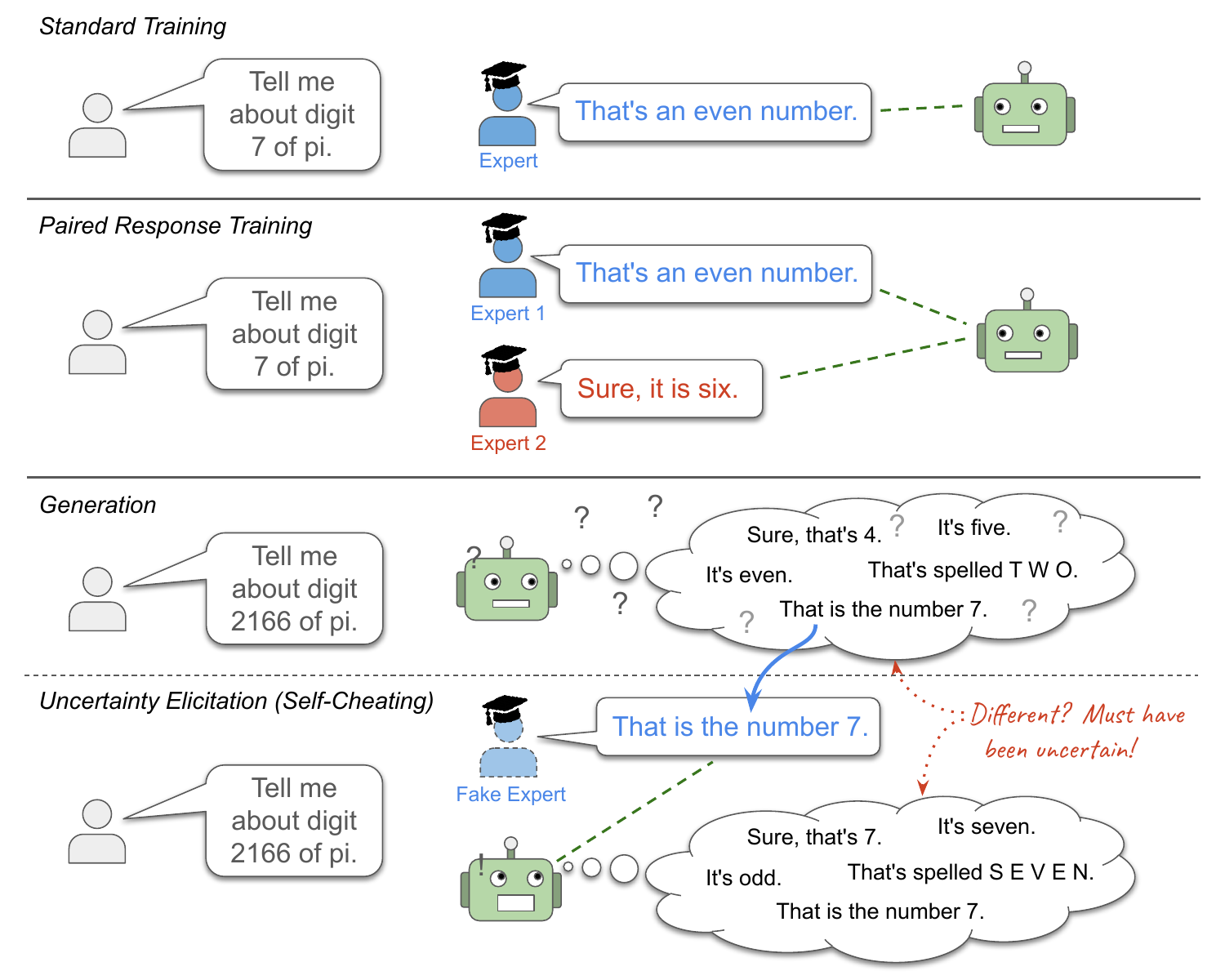}
    \vspace{-2em}
    \caption{
    We train a model (\textcolor{fig1green}{green $\hat{p}_{\theta}$}) to predict pairs of i.i.d. ground-truth answers (\textcolor{fig1blue}{blue \CIRCLE} and \textcolor{fig1red}{red \CIRCLE}), and allow it to ``cheat'' by observing one (\textcolor{fig1blue}{\CIRCLE}) while predicting the other (\textcolor{fig1red}{\CIRCLE}).
    Calibrated models only need to cheat when there is something they don't know, so the amount that the model cheats when its own guesses are presented as expert answers can be used to construct provably-correct ``cheat-corrected'' estimates of how close $\phaty$ is to $\pgt$.}
    \label{fig:teaser}
\end{figure}

When a generative model $\phaty$
(such as a large language model)
is trained to imitate a stochastic real-world process $\pgt$, it's important to identify what the model doesn't know about the process. Missing information can cause even well-trained models to ``hallucinate'' incorrect claims \citep{Ji2022SurveyOH,Kalai2023CalibratedLM}, make unjustified decisions \citep{HbertJohnson2018MulticalibrationCF}, or exhibit ``self-delusions'' that conflate cause and effect \citep{ortega2021shaking}.
Unfortunately, detecting missing information is very difficult when the true responses $Y$ are not deterministic functions of the input $X$, because probabilistic predictions made by $\phaty$ must account for both the model's uncertainty about the process (called ``epistemic uncertainty'') and the variability intrinsic to $\pgt$ (``aleatoric uncertainty'').
For example, if responding to a query $X = \,$``Tell me about digit 5641 of $\pi$'',
the predicted probability of a response (e.g. ``That is 7'') may be small either because the model does not know how $\pgt$ would respond (e.g. whether the answer is actually ``That is 4''), or simply because there are many plausible responses under $\pgt$ (e.g. ``Sure, it's an odd number'').

If we want to determine whether our model knows enough about $\pgt$ for us to trust its responses, we cannot rely on the value of $\phaty(y|x)$ alone, since it may just be small because $\pgt(y|x)$ was small. Instead, what matters is
whether $\phaty(y|x)$ is close to $\pgt(y|x)$.
Unfortunately, although we can use metrics like cross-entropy or marginal likelihood to measure improvements in $\phaty$ toward $\pgt$, we do not generally know the entropy of $\pgt$ itself, so it is difficult to know how much our model can still improve.  In fact, when training on a dataset of $(X,Y)$ pairs it is in general \emph{impossible} to tell how close $\phaty$ is to $\pgt$ without making assumptions about $\pgt$ \citep{barber2020distribution}. And if we make assumptions that turn out to be false, ensembling or Bayesian-inference-based approaches can produce highly-confident low-uncertainty estimates despite converging to a model that fails to fit important patterns in the data.

In this work, we show that these limitations can be overcome without making assumptions about $\pgt$ if we instead make a small modification to the training procedure: collect and train on \emph{pairs} of responses $(Y_1, Y_2)$ for each $X$. Our strategy is based on the following intuition:
if an unscrupulous student doesn't know the answer to a question, they could improve their guess by peeking at someone else's answer. By analogy, if a model's prediction $\phaty(\cdot|x)$ does not match the true distribution $\pgt(\cdot| x)$, the model should be able to improve its prediction if it \emph{cheats} by peeking at a sample $y_1 \sim \pgt(\cdot|x)$ from the distribution first.
And since models only benefit from cheating when they do not already know the distribution, the amount that a calibrated model cheats gives us exactly what we need to robustly estimate the gaps between $\phaty$ and $\pgt$.
Our contributions are:
\begin{itemize}[left=\parindent,topsep=0pt]
    \item 
    We define \emph{second-order calibration}, an extension of ordinary calibration that requires models to additionally report how much the true probabilities $\pgt(\cdot| x)$ (co)vary around $\phaty(\cdot| x)$ when there are inputs the model cannot distinguish (\cref{fig:simplex-calibration}). 
    We also demonstrate that popular epistemic uncertainty quantification approaches are not second-order calibrated under misspecification (\cref{fig:intro_1d_problem}).
    \item We show that second-order calibration is equivalent to ordinary calibration over pairs of responses $(y_1, y_2)$, and propose a simple modification to standard maximum-likelihood training (``\emph{training models to cheat}'' as in \cref{fig:teaser}) which incentivizes models to become second-order calibrated given sufficient capacity and training data.
    \item We prove that, given a calibrated model of pairs, you can construct confidence intervals for the true probabilities $\pgt(y|x)$ and reliable tests for ``statistical hallucinations'' (responses $y$ with $\pgt(y|x) = 0$). Our tests rely on a novel and easily-computable \emph{cheat-corrected epistemic confidence} metric, and can be combined with most off-the-shelf decoding strategies to construct new selective decoders with bounded hallucination rates.
    \item For binary $\cY = \{0,1\}$, we further show that you can construct nontrivial confidence intervals for $\pgt$ even with a miscalibrated model as long as you have a calibration set of paired responses, without making any assumptions about the form of $\pgt$. This means that impossibility results for distribution-free probability regression \citep{barber2020distribution} do not apply when we use paired responses.
    \item We demonstrate that pair-based variance estimates are empirically second-order well-calibrated on the \mbox{CIFAR-10H} perceptual uncertainty dataset \citep{peterson2019human}, outperforming a variety of existing uncertainty quantification baselines while only requiring small modifications to the data format and output layer.
    \item We also train Transformer \citep{vaswani2017attention} sequence models on paired responses in synthetic language modeling and partially-observable gridworld tasks, and show that our statistical-hallucination tests enable reliable detection of false statements and unsafe actions despite never observing any such errors during training.
\end{itemize}

\section{Second-Order Calibrated Models Report Where They Know The True Conditional}
Let $\mathcal{X}$ be a set of inputs (e.g. prompts or images), and $\mathcal{Y}$ be an arbitrary discrete set of possible responses (such as token sequences or class labels).
Suppose we train a model $\phaty$ on a dataset collected from a query distribution $p(X)$ and a ground-truth conditional distribution $\pgt(Y|X)$, with $X \in \mathcal{X}$ and $Y \in \mathcal{Y}$, and we then use this model to predict the distribution of $Y$ for new $X \sim p(X)$ drawn at inference time.
How can we tell if our model $\phaty$ knows enough to match $\pgt$ for these new queries? Specifically, how can we obtain a reliable estimate of the gap between $\phaty(\cdot | x)$ and $\pgt(\cdot | x)$?

\subsection{Calibrated Models Can Be Far From Perfect}\label{sec:predictive_calibration_grouping}

A common way to measure the quality of $\phaty$ is to measure its \emph{calibration}: if we aggregate over inputs $X$ that have the same predicted probability $\phaty(y|X)$, we should hope the true fraction for which $Y=y$ to be about $\phaty(y|X)$.

\begin{restatable}{definition}{RestateDefnCalibration}\label{defn:calibration}
Let $\Delta^{\cY}$ denote the set of probability distributions over the discrete space $\cY$.
A predictor $\phaty \!: \cX \to \Delta^{\cY}$ is \textbf{(first-order) calibrated}
if there exists a \textbf{grouping function} $\Phi : \cX \to \cZ_\Phi$
such that $\phaty$ maps each input $x \in \cX$ to the average ground-truth distribution $\pgt$ across random inputs $X$ in the same \emph{equivalence class} $[x]_\Phi = \{ x' : \Phi(x) =\Phi(x')\} \subset \cX$:
\vspace{-0.25em}
\begin{align}
\phaty(y|x) &= \E\big[ \pgt(y|X) \bigmid X \in [x]_\Phi \big]
\label{eqn:calibdefn}\\
&= p\big( Y=y \bigmid \Phi(X)=\Phi(x) \big).
\nonumber
\end{align}
\end{restatable}
\vspace{-1em}
Calibration is usually defined for the specific grouping function
$\PhiY \!\!:\! \cX \!\to\! \R^\cY$ with $\PhiY\!(x)_y \!=\! \phaty(y|x)$, so that the groups are the subsets of $\cX$ that map to the same predicted distribution \citep{kumar2019verified,vaicenavicius2019evaluating,perez2022beyond}.
We define calibration in terms of an arbitrary grouping function $\Phi$ to emphasize that a model $\phaty$ can ignore parts of $X$ and still be well-calibrated;
in this case the grouping function $\Phi(x)$ identifies the subsets of $\cX$ that the model distinguishes between.
These two definitions are equivalent  \citep{gupta2020distribution}, since $\PhiY$ is the coarsest $\Phi$ satisfying \cref{eqn:calibdefn}:

\begin{restatable}{proposition}{RestatePropCalibCoarse}\label{prop:calib_coarse}
If Eqn. (\ref{eqn:calibdefn}) holds for some fixed $\Phi$, then it must also hold for
$\PhiY \!: \cX \to \R^\cY$, where $\PhiY\!(x)_y \triangleq \phaty\!(y|x)$.
\end{restatable}

\begin{figure}
    \centering
    \includegraphics[width=\linewidth]{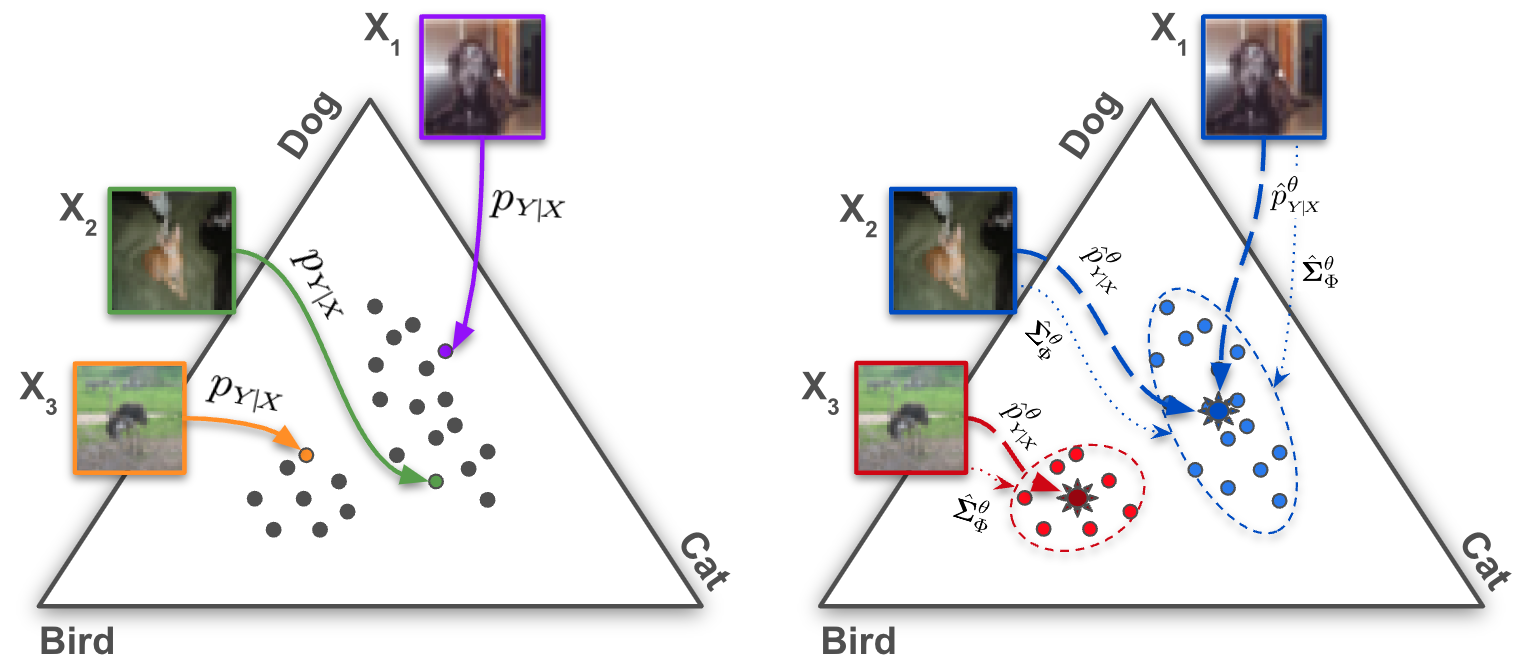}
    \vspace{-2em}
    \caption{
    Each input point $x$ (e.g. an ambiguous image) has its own ground-truth response distribution  $\pgt(\cdot|x)$ (e.g. possible human annotator labels for $x$), but
    first-order calibration only requires the model's prediction $\phaty$ to be an \emph{average} of $\pgt$ across an arbitrary grouping of examples (red and blue), which means $\phaty$ can still be far from $\pgt$ for each individual $x$.
    A second-order-calibrated model additionally measures the suboptimality of this approximation by predicting the per-group \emph{covariance} $\sigmahat$ of $\pgt$, but this is challenging because $\pgt$ itself is never observed.
    }
    \label{fig:simplex-calibration}
\end{figure}

(We defer proofs of all theoretical results to \cref{app:theory}.)

A well-calibrated predictor can still be a bad estimate of $\pgt$ if it fails to distinguish inputs with different true probabilities $\pgt(y|X)$ and thus averages across them.
For example, a calibrated coin-flip predictor might output $\phaty(\textsc{\footnotesize{Heads}}|x) = 50\%$ because it knows coin $x$ is fair, or because it cannot distinguish coins $x_+$ and $x_-$ with opposite biases. In the first case $\phaty\!(\textsc{\footnotesize{Heads}}|x) \!=\! \pgt\!(\textsc{\footnotesize{Heads}}|x)$ and the model is optimal, but in the second the model is suboptimal because it has put inputs with $\pgt(y|x_+) \ne \pgt(y|x_-)$ into the same group. This additional error is called the \emph{grouping loss} \citep{perez2022beyond,kull2015novel}, which can be lower-bounded but is difficult to upper-bound.

\subsection{Second-Order Calibration Measures The Gap}\label{sec:second_order_calibration}
It would be useful if we could get a model to tell us how far $\phaty(y|x)$ might be from $\pgt(y|x)$ for each $x$, conditioned on what the model ``knows''. We make this precise by proposing the following definition.

\begin{restatable}{definition}{RestateDefnSecondOrderCalibration}\label{defn:second_order_calibration}
A predictor $\phaty \!: \cX \to \Delta^{\cY}$ and covariance estimator $\sigmahat \!: \cX \to \R^{\cY\times\cY}$ are \textbf{second-order calibrated} if there exists a grouping function $\Phi$ such that
$\phaty$ and $\sigmahat$ map each input $x \in \cX$ to the average \emph{and covariance matrix} of the ground truth probability vector $\pgtvec\!(\cdot|x) \in \Delta^{\cY}$ across inputs $X$ in the same equivalence class under $\Phi$:
\begin{align*}
    \phaty(y | x) &= \E\big[\pgt(y | X) \bigmid
    X \in [x]_\Phi
    \big],
    \\
    \sigmahat\!(x) &= \Cov\!\Big[ \pgtvec\!(\cdot|X),\, \pgtvec\!(\cdot|X)
    \!\Bigmid\!
    X \in [x]_\Phi
    \!\Big]
\end{align*}
where $\pgtvec(\cdot|x)_y = \pgt(y|x)$.
We call $\sigmahat$ the \textbf{epistemic covariance} of the true conditional $\pgtvec\!(\cdot|x)$ under $\Phi$.
\end{restatable}

If we had a second-order-calibrated predictor, we could use it to identify how tightly concentrated the true probability vector $\pgtvec$ is around the model's best guess $\phaty$ (as shown in \cref{fig:simplex-calibration}), which would tell us whether $\phaty$ is a good approximation of $\pgt$.
In our coin-flip example,
a second-order-calibrated model would report $\sigmahat(x)_{y, y} = 0$
if it knows the coin is fair, and $\sigmahat(x)_{y, y} > 0$ if it can't tell which way $x$ is biased (i.e. if $\Phi(x) = \Phi(x_+) = \Phi(x_-)$).

Unfortunately, it is not straightforward to construct a second-order-calibrated predictor, because we only observe a \emph{sample} $Y \sim \pgt(\cdot|x)$ and not the full $\pgtvec$. Second-order calibration requires the predictor to distinguish between epistemic and aleatoric uncertainty, but the variance $\Var(Y| \Phi(X))$ of $Y$ itself (for a binary $Y$) still only measures the total uncertainty and is thus a first-order quantity.

\begin{figure}
    \centering
    \includegraphics[width=\linewidth]{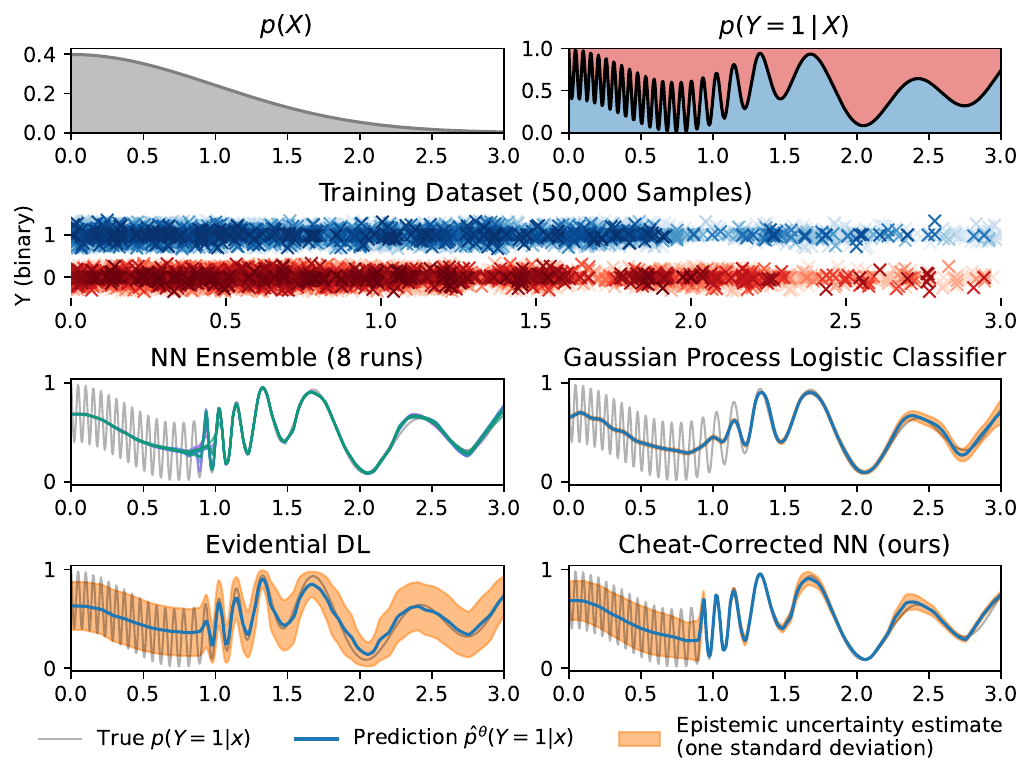}
    \vspace{-2em}
    \caption{
    Popular epistemic uncertainty quantification methods are under- or overconfident when $\pgt$ does not match their assumptions. Given a large number of samples $X \in \R, Y \in \{0,1\}$, ensembles and misspecified Gaussian process classifiers report low uncertainty at convergence despite failing to match $\pgt$ around $x \approx 0$; Evidential DL \citep{Sensoy2018EvidentialDL} reports high uncertainty near $x \approx 2.0$ despite fitting well. In contrast, by using \emph{two} samples $(Y_1, Y_2)$ for each $X$, our method reports uncertainty that matches the true gap $(\phaty - \pgt)^2$ even when it underfits.
    }
    \label{fig:intro_1d_problem}
\end{figure}

\subsection{Existing Epistemic Uncertainty Estimators Under- or Over-estimate The Gap For Underfit Models}
Existing techniques for estimating epistemic uncertainty often attempt to estimate how much $\pgt$ could vary given what the model ``knows''. For instance, Gaussian processes \citep{bernardo1998regression} and Bayesian neural networks \citep{Goan2020BayesianNN} impose a prior distribution over the generative process, then evaluate the variance of the prediction under an approximate posterior \citep{Kendall2017WhatUD}. Other related strategies include ensembling \citep{Lakshminarayanan2016SimpleAS}, injecting noise into the model or training process \citep{Gal2015DropoutAA, Osband2021EpistemicNN, Maddox2019ASB}, or 
predicting a ``distribution over distributions''
\citep{Sensoy2018EvidentialDL,malinin2018predictive}.

We might hope that these estimates would be second-order calibrated, but unfortunately this is not generally the case, especially if the model is misspecified or underfit relative to $\pgt$. We demonstrate this in \cref{fig:intro_1d_problem} by applying a variety of methods to a fixed $\pgt$ with both low- and high-frequency variation (discussed more in \cref{appendix:training_1d_regression}).
With a large training set, an ensemble and a Gaussian Process classifier both converge to highly confident but incorrect solutions, because the prior was misspecified and did not include $\pgt$.
Evidential DL \citep{Sensoy2018EvidentialDL}, on the other hand, is underconfident because its objective biases its uncertainty estimates \citep{bengs2022pitfalls,bengs2023second}.
In practice, even the largest models are likely to underfit in some regions of $\cX$, making this a serious concern if we wish to reliably estimate how far $\phaty$ actually is from $\pgt$.

\section{Second-Order Calibration From Paired $Y$s}

How can we obtain a second-order calibrated model?
We now show that making second-order-calibrated predictions about individual response \emph{probabilities} is equivalent to making first-order-calibrated predictions about \emph{paired responses}.

Suppose we have a model $\phatyy(Y_1, Y_2 | X)$ predicting a distribution over $\cY \times \cY$, and let $\phatym$, $\phatymtwo$, $\phatyc$ be the induced marginal and conditional distributions.
If $\phatyy$ is calibrated at predicting a pair of independent responses $Y_1, Y_2\simiid \pgt(\,\cdot\,|X)$, it must be the case that
\begin{align*}
\phatyy(y_1, y_2|x) &= \E\Big[ \pgt(y_1|X)\!\cdot\!\pgt(y_2|X) \!\Bigmid\! X \!\in\! [x]_\Phi \! \Big]
\end{align*}
for some $\Phi$. 
How much should we expect $y_2$ to depend on $y_1$ according to this model? Although $Y_1$ and $Y_2$ are independent given $X$, they may not be independent conditioned on $\Phi(X)$, i.e. conditioned on what our model ``knows'' about $X$. In this case, we should expect a calibrated model to ``\emph{cheat}'' by using information about $y_1$ to better inform its prediction of $y_2$. We can quantify this by measuring how correlated the possible outcomes are under the model:

\begin{restatable}{definition}{DefnCheatCorrCov}\label{defn:cheat_corr_cov}
The \textbf{pair covariance} of $\phatyy$ is
\begin{align*}
\sigmacheat\!(x)_{y_i, y_j} &\triangleq \phatyy\!(y_i, y_j | x) - \phatym\!(y_i|x)\,\phatymtwo\!(y_j|x)
\end{align*}
\end{restatable}
$\sigmacheat\!(x)_{y_i, y_j}$ is the difference between the predicted joint and what we would expect if $Y_1$ and $Y_2$ were independent given $\Phi(X)$.
It turns out that this is exactly what we need to construct a second-order-calibrated predictor of $\pgt$:

\begin{restatable}{theorem}{RestateThmCheatEquivalence}\label{thm:cheat_equivalence}
If $\phatyy$ is first-order calibrated at predicting pairs $(Y_1, Y_2)$,
then its marginal $\phatym$ and pair covariance $\sigmacheat$ are second-order calibrated at predicting $\pgt$.
Moreover, this is a bijection: for any second-order-calibrated $(\phatprimey, \sigmahatprime)$, there is a unique first-order-calibrated $\phatyy$ with $\phatprimey = \phatym$ and $\sigmahatprime=\sigmacheat$.
\end{restatable}

This equivalence means that techniques for training first-order-calibrated models can also be used to construct second-order calibrated models whenever it is possible to draw multiple samples from $\pgt$ (e.g. by asking two random human experts to label $X$). In particular, we propose to directly train a model $\phatyy(Y_1, Y_2 | X)$ to predict paired responses by minimizing the standard cross-entropy loss
\[
-\E_{\substack{X \sim p(X),\\Y_1, Y_2 \sim \pgt}}\Big[ \log\phatyy(Y_1, Y_2 | X) \Big]
\]
over a dataset of $(X^{(i)}, Y_1^{(i)}, Y_2^{(i)})$ triples.
Since cross-entropy is a proper scoring rule \citep{kull2015novel}, we can expect that our model will become more calibrated over $\cY \times \cY$ as it improves. Indeed, calibration is linked to generalization ability \citep{Carrell2022TheCG} and hallucination behavior \citep{Kalai2023CalibratedLM} and tends to emerge in sufficiently-high-capacity models \citep{Basiok2023WhenDO,openai2023gpt4,Kadavath2022LanguageM}.
We note that if our model is explicitly factorized as
\begin{align*}
\phatyy(y_1, y_2 | x) = \phatym(y_1 | x) \cdot \phatyc(y_2 | y_1, x)
\end{align*}
(e.g. an autoregressive model), we expect it to learn to ``cheat'' by copying information from $Y_1$ to $Y_2$ whenever there are regularities between $Y_1$ and $Y_2$ that aren't explained away by what the model knows. This is exactly what we want, because calibration \emph{requires} $\phatyy$ to cheat whenever $\phatym \ne \pgt$; we can then use \cref{thm:cheat_equivalence} to
determine how close $\phatym$ is to $\pgt$.
Informally, an expert doesn't need to cheat, so if you let your model cheat and it does, it must not know the answer to your question.

\section{Bounding Approximation Error With Pairs}
\subsection{Pair Predictors Can Bound Their Own Individual-Response Errors By Self-Cheating}\label{sec:pair_predictors_bound_error}
We now derive a number of properties which are particularly useful when using $\phatym$ to imitate $\pgt$:
bounded deviation between $\phaty$ and $\pgt$, and bounded probability of producing outputs where $\pgt(y|x) = 0$.
These results rely on the fact that,
conditioned on the matrix $\PhiYY(x) \in \R^{\cY\times\cY}$ of model outputs (with $\PhiYY(x)_{y_1, y_2} = \phatyy(y_1, y_2 |x)$),
we can treat $\pgt(y|X)$ as a random variable whose mean is $\phatym(y|X)$ and variance is $\vcheat(y | X)$, defined below:

\begin{restatable}{definition}{RestateDefnCheatVar}\label{thm:cheat_var}
The \textbf{cheat-corrected epistemic variance} of $\pgt$ for response $y$ to query $x$ (under $\phatyy$) is
\begin{align*}
\vcheat(y | x) &\triangleq \phatym\!(y | x)\,\big( \phatyc\!(y | y, x) - \phatym\!(y | x)  \big).
\end{align*}
\end{restatable}
$\vcheat$ can be computed easily by scoring $y$ twice, once under the marginal distribution of $Y_1$ and once when the model ``self-cheats'' by conditioning on $y$ (as $Y_1$) when predicting $y$ again (as $Y_2$). Furthermore, it agrees with the diagonal entries of $\sigmacheat\!(x)$ as long as $\phatyy$ is symmetric (which is true if $\phatyy$ is calibrated). This means we can use it to bound the distance between $\phatym$ and $\pgt$.

\begin{restatable}{theorem}{RestateThmCheatGrouping}\label{thm:conf_cheat_grouping}
Suppose $\phatyy$ is calibrated.
Let $A$ be any event and $\Ydec \in \cY$ be any (possibly random) value such that $\Ydec, A \indep X \mid \PhiYY(X)$. Then
\vspace{-0.2em}
\begin{align*}
\E\Big[
\!\big(\phatym\!(\Ydec | X) - \pgt\!(\Ydec | X) \big)^2 \!\Bigmid\! A
\Big] = \E\Big[
\vcheat\!(\Ydec | X) \!\Bigmid\! A
\Big].
\end{align*}
Furthermore, for any $\beta \in (0,1)$,
\vspace{-0.2em}
\begin{align*}
\textstyle{}P\left[\Big| \phatym(\Ydec | X) - \pgt(\Ydec | X) \Big| \ge \!\sqrt{\!\frac{\strut\vcheat(\Ydec | X)}{\beta}}\, \!\middlemid\! A \right] \le \beta.
\end{align*}
\end{restatable}

This is a input-dependent (frequentist) confidence interval for $\Ydec$; if our model reports a small value of $\vcheat(\Ydec | X)$, we can guess that $\phatym(\Ydec | X)$ is close to $\pgt(\Ydec | X)$ and be right most of the time.
(For instance, if $A$ is the event where our example coin-flip predictor predicts 50\% \textsc{Heads} with epistemic variance $\le \epsilon$, at least 95\% of the coins with that property must have a bias within $\sqrt{\varepsilon/.05}$ of 50\%.)

When $\cY$ is large, we may be less interested in directly estimating $\pgt$ for a particular $y$, and more interested in making sure we don't
generate any response $y$ for which $\pgt(y|x)$ was actually zero; we call such a response a \emph{statistical hallucination}.\footnote{
The term ``hallucination'' is often used to mean ``output with false factual claims''. These count as statistical hallucinations as long as $\pgt$ never produces them, but statistical hallucinations also include behavior such as taking unsafe actions that $\pgt$ would avoid, making a math error where $\pgt$ would be correct, or failing to satisfy any other property of all samples generated by $\pgt$.
} We can do so using the following metric:

\begin{restatable}{definition}{RestateDefnEpistemicConfidence}\label{defn:epistemic_confidence}
The \textbf{cheat-corrected epistemic confidence} of $\phatyy$ about response $y$ to query $x$ is
\vspace{-0.5em}
\[
\Ccheat(y | x) \triangleq \frac{\phatym(y | x)}{\phatyc(y|y,x)}~~~\text{(or 0 if $\phatym(y|x) = 0$)}.
\]
\vspace{-0.5em}
\end{restatable}
$\Ccheat$ measures the \emph{relative} likelihood with and without self-cheating, with the denominator correcting for the ``aleatoric'' aspects of $y$ that remain unpredictable even when the model cheats. 
Similar to $\vcheat$, it can be computed easily by scoring $y$ twice. 
$\Ccheat$ is also properly normalized:

\begin{restatable}{proposition}{RestatePropConfBounded}\label{thm:conf_bounded}
If $\phatyy$ is calibrated, then for any $x \in \cX, y \in \cY$ we have $0 \le \Ccheat(y | x) \le 1$, with $\Ccheat(y | x) = 1$ if and only if $\phatym(y|x) = \pgt(y|x)$.
\end{restatable}

And we can use it to bound the statistical-hallucination rate of any well-behaved decoding algorithm:

\begin{restatable}{theorem}{RestateThmConfHallucination}\label{thm:conf_hallucination}
Suppose $\phatyy$ is calibrated.
Let $A$ be the event that a decoding algorithm responds to a query $X$, and $\Ydec \in \cY$ be its response.
If $A, \Ydec \indep X \mid \PhiYY(X)$, then the statistical hallucination rate of the generated responses is bounded above as
\begin{align*}
    P\left[ \pgt(\Ydec|X) = 0 \;\middle|\; A \right] \le 1 - \E\left[\Ccheat(\Ydec | X) \;\middle|\; A\right].
\end{align*}
\vspace{0.1em}
\end{restatable}
\vspace{-0.1em}
We can use any decoding strategy that only depends on $X$ through $\phatyy$, including temperature sampling, top-$k$/top-$p$ sampling, or beam search (see \citet{zarriess2021decoding} for an overview).
Moreover, we are free to use $\Ccheat(\Ydec | X)$ in the algorithm to ensure that $1 - \Ccheat$ is low. For example, these decoding strategies will all have a statistical hallucination rate at most $\beta$ when $\phatyy$ is calibrated:
\begin{itemize}[topsep=0em]
    \item \textbf{Cheat-corrected selective generation / filtering}: Generate $\Ydec$ using an arbitrary off-the-shelf sampler, but reject it (and don't respond) if $1 - \Ccheat(\Ydec | X) > \beta$.
    \item \textbf{Cheat-corrected rejection sampling}: Repeatedly sample $\Ydec \sim \phatym$ until $1 - \Ccheat(\Ydec | X) < \beta$.
    \item \textbf{Cheat-corrected top-1 search}: Deterministically output (or approximate) $\argmax_{y \in S} \phatym(y | X)$, where $S = \{ y \mathrel{:} 1 - \Ccheat(\Ydec | X) < \beta \}$, or abstain if $S = \varnothing$.
\end{itemize}
Selectively responding only when we find a $\Ydec$ with $1 - \Ccheat(\Ydec | X) < \beta$ ensures that, conditioned on responding (e.g. on the event $A$), our responses will be non-hallucinated with probability at least $1-\beta$.

\subsection{Paired Data Enables Distribution-Free Frequentist Confidence Intervals for $p(Y|X)$}\label{sec:distnfree}
Finally, we show that we can adjust imperfectly-calibrated estimators $\phaty : \cX \to \Delta^{\cY}$ and $\vhat : \cX \to \R^{\cY}$ to obtain robust statistical guarantees about the unobserved true conditional probabilities $\pgt(Y|X)$ without assumptions about $\pgt$, as long as we have access to a held-out \emph{calibration set} $\{(x^{(i)}, y_1^{(i)}, y_2^{(i)})\}_{i=1}^N$ containing paired response data. 
This demonstrates that the impossiblity result of \citet{barber2020distribution} does not apply when we have access to two $Y$s for each $X$.
For simplicity we assume $\cY = \{0,1\}$.

\begin{algorithm}[t]
   \caption{Conservative adjustment of $\vhat$}
   \label{alg:distfree_bound}
\begin{algorithmic}
   \STATE {\bfseries Input:} Calibration set $\{(x^{(i)}, y_1^{(i)}, y_2^{(i)})\}_{i=1}^N$, variance cutoff $\varepsilon > 0$, tolerance $\alpha$, $\phaty$, $\vhat$
   \FOR{$i=1$ {\bfseries to} $N$}
   \STATE $\hat{p}^{(i)} := \phaty(1 | x^{(i)})$,~~ $\hat{v}^{(i)}_\varepsilon := \max\{\vhat(1 | x^{(i)}), \varepsilon\}$
   \STATE $s_\varepsilon^{(i)} := ( y_1^{(i)} - \hat{p}^{(i)} )( y_2^{(i)} - \hat{p}^{(i)} ) / \hat{v}^{(i)}_\varepsilon$
   \ENDFOR
\STATE $(\gamma_\varepsilon^-, \gamma_\varepsilon^+) := \textsc{MeanConfItvl}\big(\{s_\varepsilon^{(i)}\}_{i=1}^N, -\frac{1}{\varepsilon}, \frac{1}{\varepsilon}, \alpha\big)$
\STATE \textbf{return} $\gamma_\varepsilon^+$
\end{algorithmic}
\end{algorithm}

\begin{restatable}{theorem}{RestateThmDistnFreeBound}\label{thm:distnfreebound}
Let $\phaty$, $\vhat$, and $\pgt$ be arbitrary.
With probability at least $1 - \alpha$ (over draws of the calibration set), \cref{alg:distfree_bound} returns a value $\gamma_\varepsilon^+$ such that, for a randomly sampled input $X \sim p(X)$, and any $\beta \in (0, 1), y \in \{0,1\}$,
\vspace{-0.2em}
\begin{align*}
\textstyle{}P\!\!\left[\Big| \phaty(y | X) - \pgt(y | X) \Big| \ge \!\sqrt{\!\frac{\gamma_\varepsilon^+\max\{\vhat(y | X), \varepsilon\}}{\beta}} \right] \!\le\! \beta.
\end{align*}
\end{restatable}

In \cref{alg:distfree_bound}, $\textsc{MeanConfItvl}$ can be any subroutine that builds a $(1-\alpha)$ confidence interval for the mean of a bounded random variable, e.g. Hoeffding's inequality \citep{Hoeffding1963ProbabilityIF} or betting-based algorithms \citep{waudby2020estimating}.
Smaller $\varepsilon$ allows more precise bounds but requires a well-calibrated $\vhat$ and a large calibration set, and if $\phaty$ and $\vhat$ are in fact second-order calibrated then $\gamma_\varepsilon^+$ will approach 1 as $N \to \infty$ and $\varepsilon \to 0$.
The failure probability $\beta$ should be interpreted as an aggregate over $X \sim p(X)$ rather than pointwise; for a fixed process $\pgt$ and fixed $x$ either $\pgt(y|x)$ lies in the interval or it does not.
We show an example of the resulting confidence intervals in \cref{fig:bound-example}, and discuss them further in \cref{appendix:distnfree_theory,appendix:distnfree_experiment}.

\begin{figure}[t]
    \centering
    \includegraphics[width=\linewidth]{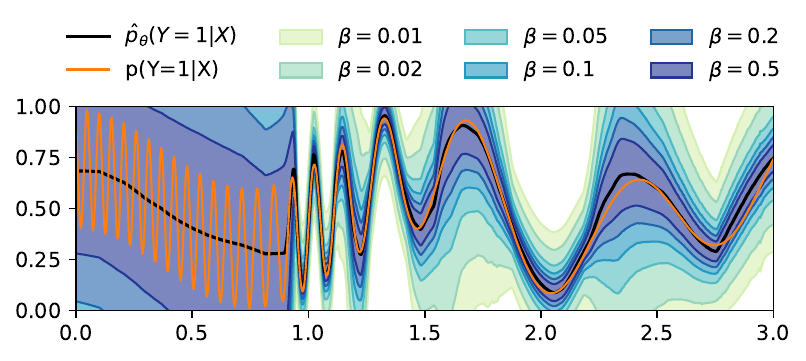} 
    \vspace{-2em}
    \caption{Applying \cref{alg:distfree_bound} to our model from  \cref{fig:intro_1d_problem} produces frequentist confidence intervals for $\pgt(y|X)$ which are provably correct with high probability over random $X$. Here $N=10^6, \varepsilon=0.02^2,$ and $\alpha=0.05$; see \cref{appendix:distnfree_experiment}.}
    \label{fig:bound-example}
\end{figure}

\section{Related Work}

\textbf{Decomposing uncertainty with paired $Y$s.}
Focusing on regression tasks and asymptotic optimality, \citet{lahlou2021deup} estimate aleatoric uncertainty by predicting
$(y_1 - y_2)^2$ for
two real-valued samples from $p(Y|X)$, then use it to quantify epistemic uncertainty. For classification, \citet{narimatsu2023collision} use annotator agreement to quantify aleatoric uncertainty at the population level. Repeated annotations have also been used to improve and evaluate classifiers \citep{peterson2019human,schmarje2022one}.

\textbf{Uncertainty via LLM postprocessing.}
For language models, proposed techniques include verifying, critiquing, or classifying samples \citep{Cobbe2021TrainingVT,Ni2023LEVERLT,Li2022MakingLM,Kadavath2022LanguageM}, or clustering semantically-equivalent samples \citep{kuhn2022semantic,Li2022CompetitionlevelCG,Wang2022SelfConsistencyIC,Chen2023UniversalSF}. This generally requires a task-specific
correctness or similarity metric, and may be less applicable for
generation tasks without well-defined correct answers.
Additionally, most multiple-sample approaches focus on comparing many $Y$s at inference time, whereas our strategy only uses paired $Y$s at training time and then scores each $Y$ individually.

\textbf{Other uses of paired $Y$s.}
In other contexts, paired inputs have been used to learn representations \citep{bromley1993signature, chen2020simple}, and pairwise losses have been used to train energy-based models \citep{gutmann2010noise}. \citet{lin2018pacgan} train a GAN discriminator to distinguish pairs of real v.s. generated images and show that this reduces mode collapse.

\textbf{Uncertainty via dependence on additional information.}
\citet{durasov2022zigzag} train a model to predict the same output both with and without feeding in the correct output as an extra input, and use the change in prediction to measure uncertainty. A key difference between this and our cheat-correction procedure is that  \citeauthor{durasov2022zigzag} treat the output as deterministic (no aleatoric uncertainty) and rely on inductive biases of the predictor rather than calibration. \citet{collier2022transfer} provide additional privileged information about the label process in order to explain away label noise and improve robustness.

\textbf{Uncertainty via extensions of calibration.}
To better measure uncertainty for calibrated models,
\citet{perez2022beyond} propose bounding the population grouping error by partitioning the model's feature space. \citet{HbertJohnson2018MulticalibrationCF} study \emph{multicalibration}, which requires calibration to hold across all computable subsets of a population. 

\textbf{Distribution-free uncertainty quantification.}
A number of approaches have been explored for quantifying uncertainty without making assumptions about the functional form of $p(Y|X)$, generally by using a held-out calibration set. Many build on conformal prediction, and use exchangeability to construct high-probability prediction sets; see \citet{angelopoulos2021gentle} for an introduction. Related approaches can be used to construct calibrated classifiers \citep{kumar2019verified,gupta2020distribution,Park2020PACCP} and randomized predictive distributions \citep{vovk2017nonparametric}.
We discuss these connections in more detail in \cref{app:sec:discussion}.

\begin{table*}[t]
\caption{
\textbf{Cheat-corrected uncertainty estimates are better second-order-calibrated than other techniques, while maintaining similar accuracy.} Our primary metrics: \emph{ECE-2} is second-order calibration error of the variance estimate (best ECE-2 in bold), $\E[\hat{v}^\theta]$ is predicted epistemic variance, and $\E[(\hat{p}^\theta{-}p)^2]$ is actual grouping error (ideally close to $\E[\hat{v}^\theta]$). For comparison, \emph{ECE-1} is first-order calibration error of predicted probabilities, \emph{Acc} is top-1 accuracy on the original labels from CIFAR-10, and \emph{KL} measures the divergence from $\pgt$ (ground-truth annotator labels) to $\phaty$. All metrics are averaged over eight random training seeds, and metrics other than Acc and KL are summed across classes. 
}
\label{tab:cifar10h}
\vskip 0.15in
\begin{center}
\begin{small}
\begin{sc}
\begin{tabular}{l@{\hskip 2em}cc@{\hskip 1.5em}c@{\hskip 1.5em}ccc@{\hskip 2em}cc@{\hskip 1.5em}c@{\hskip 1.5em}cc}
\toprule
& \multicolumn{6}{c}{CIFAR-10H}& \multicolumn{5}{c}{w/ Extra Classes, Scrambled}
\\
\cmidrule(l{-2pt}r{2em}){2-7}\cmidrule(l{-2pt}r){8-12}
Method & \scriptsize \textbf{ECE-2} & \scriptsize $\E[\hat{v}^\theta]$ & \clap{\scriptsize $\E[(\hat{p}^\theta{-}p)^2]$} & \scriptsize ECE-1 & \scriptsize Acc & \scriptsize KL & \scriptsize \textbf{ECE-2}& \scriptsize $\E[\hat{v}^\theta]$ & \clap{\scriptsize $\E[(\hat{p}^\theta{-}p)^2]$} & \scriptsize ECE-1 & \scriptsize KL \\
\midrule
Naive NN & 0.076 & 0.142 & 0.065 & 0.02 & 93.9 & 0.18 & 0.521 & 0.682 & 0.161 & 0.07 & 0.71 \\
NN Ensemble & 0.039 & 0.014 & 0.053 & 0.03 & 94.9 & 0.15 & 0.134 & 0.014 & 0.148 & 0.03 & 0.65 \\
Evidential DL & 0.377 & 0.053 & 0.430 & 1.04 & 88.5 & 1.09 & 0.387 & 0.031 & 0.418 & 0.79 & 2.36 \\
SNGP Cov. & 0.048 & 0.005 & 0.052 & 0.02 & 94.9 & 0.15 & 0.112 & 0.033 & 0.145 & 0.06 & 0.63 \\
Epinet & 0.056 & 0.015 & 0.071 & 0.02 & 93.4 & 0.19 & 0.089 & 0.087 & 0.163 & 0.07 & 0.71 \\
\midrule
Cheat NN & 0.018 & 0.052 & 0.068 & 0.03 & 93.6 & 0.18 & 0.022 & 0.134 & 0.154 & 0.07 & 0.67 \\
Cheat SNGP & \textbf{0.009} & 0.054 & 0.052 & 0.02 & 94.9 & 0.15 & \textbf{0.011} & 0.153 & 0.150 & 0.04 & 0.65 \\
\bottomrule
\end{tabular}
\end{sc}
\end{small}
\end{center}
\vskip -0.1in
\end{table*}

\pagebreak[3] %
\textbf{Predicting distributions-over-distributions.}
Some previous techniques \citep[e.g.][]{Sensoy2018EvidentialDL,malinin2018predictive} have explored measuring uncertainty by predicting a distribution over possible output distributions (sometimes called second-order distributions). However,
\citet{bengs2022pitfalls,bengs2023second} proved that many such approaches do not incentivize faithful reports of uncertainty. \citet{sale2023second} formalize uncertainty measures for second-order distribution predictors in terms of distances to sets of reference distributions. Note that our work uses ``second-order'' in the sense of the second-moment statistics in \cref{thm:cheat_equivalence}, not second-order distributions. Our approach does not predict a full distribution over distributions.

\section{Experiments}

\subsection{Classifying Ambiguous Images}\label{sec:cifar10h}

We demonstrate our technique on CIFAR-10H \citep{peterson2019human}, a relabeling of the CIFAR-10 test set \citep{Krizhevsky2009LearningML} by $>50$ independent annotators per image.
We cast it as a distribution-matching problem rather than an accuracy-maximization problem: the goal is to estimate the fraction of human annotators assigning each label $y$ to each image $x$. In this setting, we expect epistemic uncertainty quantification techniques to distinguish between between images that \emph{human annotators} find ambiguous and images that the \emph{model} has not learned to identify.
Our primary evaluation metric is second-order expected calibration error (ECE-2), the difference between each technique's variance estimate $\vhat$ and the true squared error $\big(\hat{p}_\theta(Y | X) - p(Y|X)\big)^2$, on an in-distribution test set.
Since some uncertainty-quantification methods may affect predictive accuracy, we additionally report the ordinary expected calibration error of $\hat{p}_\theta(Y | X)$ relative to the true annotator labels (ECE-1), the KL divergence between $\hat{p}_\theta(Y | X)$ and the empirical annotator distribution, and the top-1 accuracy with respect to the clean CIFAR-10 labels. We compute ECE-1 and ECE-2 by averaging over 100 quantile bins and summing across classes, as described in \cref{appendix:sec:cifar10h}.

We train pair-prediction models $\phatyy$ to jointly predict two random annotator labels for each minibatch example, with a symmetric $10 \times 10$ softmax output head and either an ordinary wide ResNet backbone (\textbf{Cheat NN}) from \citet{zagoruyko2016wide} or a SNGP backbone (\textbf{Cheat SNGP}) as proposed by \citet{Liu2020SimpleAP}. We then use the marginal $\phatym$ and cheat-corrected variance $\vcheat$ for evaluation. We observed that our models occasionally overfit on the small dataset and produced negative $\vcheat$ estimates due to miscalibration; we regularize them by adding a small penalty for negative eigenvalues, since $\phatyy(\cdot,\cdot|x)$ must be positive semidefinite if $\phatyy$ is calibrated (proven in \cref{appendix:properties_calibrated_of_pairs}).

We compare our approach to a variety of existing uncertainty quantification techniques:
\textbf{SNGP Cov.} \citep{Liu2020SimpleAP}, which uses spectral normalization and a Laplace random-features approximation to a Gaussian process covariance;
\textbf{Evidential DL} \citep{Sensoy2018EvidentialDL}, which uses a regularized Dirichlet output to estimate epistemic uncertainty;
\textbf{Epinet} \citep{Osband2021EpistemicNN}, which models uncertainty by feeding a random ``index'' input through a fixed ``prior'' network and a learned corrector;
\textbf{NN Ensemble} \citep{Lakshminarayanan2016SimpleAS}, which uses the mean and variance across 8 independent ResNet models;
and \textbf{Naive NN}, which uses $\hat{p}_\theta(Y | X)(1-\hat{p}_\theta(Y | X))$ as an estimate of variance (i.e. assuming $Y$ is a deterministic function of $X$). We train these baselines by trating the two randomly-selected annotator labels as separate minibatch examples.

Since CIFAR-10H includes only the original CIFAR-10 test set, we pretrain all models on CIFAR-10N \citep{Wei2021LearningWN}, a relabeling of CIFAR-10's training set by three annotators per image.
We then divide CIFAR-10H's images into two disjoint 5,000-image subsets (with $>50$ annotator labels per image), using the first to train/validate and the second for evaluation metrics.
We tune
hyperparameters
to maximize likelihood on our validation set, but intentionally avoid tuning based on second-order calibration
since this may not be computable under a standard training setup.

As shown in \cref{tab:cifar10h},
our model's cheat-corrected variance estimates are substantially better second-order calibrated than other methods, without sacrificing first-order calibration or predictive accuracy. In particular, most other methods tend to underestimate in-distribution epistemic uncertainty (with $\E[(\phaty - \pgt)^2] > \E[\hat{v}^\theta]$), although Naive NN overestimates it.
We additionally train and evaluate models on a harder task variant, where we both add extra classes to
make $\pgt$ more stochastic
and also scramble the central image pixels to make underfitting more likely, and find that our method remains second-order calibrated,  whereas other techniques become increasingly over- or under-confident.
Of our two models, the SNGP variant performs the best suggesting that well-known techniques for improving first-order calibration also improve second-order calibration when training on pairs.
We also point out that the NN Ensemble baseline gives similar variance estimates and similar ECE-1 values across the two task variants, but has worse ECE-2 and KL divergence scores on the harder variant. This means that ECE-1 and ensemble variance are not sufficient to identify tasks for which the model is a bad fit for $\pgt$, whereas the cheat-corrected variance estimate of our method is more representative of model quality.
Further details for these experiments are provided in \cref{appendix:sec:cifar10h}.

\subsection{English Descriptions of Digits of $\pi$}\label{sec:digitsofpi}

We next demonstrate that our technique
can be directly applied to tasks with large output spaces such as sequence modeling. We construct a synthetic language modeling task that allow us to control the difficulty and amount of stochasticity in the target responses, where the goal is to correctly respond to requests like $x=\,$``Tell me about digit 24 of $\pi$''. Early digits of $\pi$ are sampled more often than later ones, and the target responses are randomly-chosen true statements, such as
 ``Sure, that is the number 6'',
``That's an even number'',
``It is spelled S I X'', or
``Sure, it's spelled with three letters'', which are sampled with different probabilities and exhibit variation in both style and semantic content.

We train a 19M-parameter transformer model \citep{vaswani2017attention} from scratch for 50k iterations, tokenizing and concatenating the query $X$ and two sampled responses $Y_1$ and $Y_2$ for each example. We next sample 120 statements from $\phatym$ for each digit offset from 1 to 3,000, and label each sample
as a statistical hallucination if $\pgt(y|x) = 0$ (e.g. if it is not a true statement about the requested digit).
We then 
evaluate how well the bound in \cref{thm:conf_hallucination} holds in practice by dividing samples into bins based on their predicted confidence $\Ccheat$ and computing the fraction of samples in each bin that were hallucinated. \cref{fig:digitsofpi-precision-recall} (left) shows that the fraction of hallucinated samples is generally slightly lower than $1 - \Ccheat$, as predicted by the bound in \cref{thm:conf_hallucination}. However, somewhat surprisingly, we observe that $\Ccheat(y | x) > 1$ for some samples, which would not occur for a well-calibrated model. Samples with $\Ccheat$ slightly above 1 are usually correct, but in rare cases we also observe very large values of $\Ccheat$ (e.g. about $\approx 10^{4}$ or larger), which tend to happen when the sampled $Y_1$ was malformed and out-of-distribution. 
We believe this stems from the inherent difficulty of making calibrated predictions over the space of all (pairs of) sequences. In practice, we suggest to use $|1 - \Ccheat|$ for thresholding as an alternative to $1 - \Ccheat$, which is equivalent if the model is calibrated. We explore other thresholding options and show specific examples where $\Ccheat(y | x) > 1$ in \cref{appendix:sec:digits-of-pi}.

We next compare different strategies for distinguishing correct and hallucinated samples: ranking by the log-probability of each sample under the model (Total LP), ranking by length-normalized log-probability (Avg. Token LP) \citep{malinin2020uncertainty},
clustering semantically-equivalent answers in groups of $k$ samples and thresholding by cluster size (Clustered) \citep{kuhn2022semantic,Li2022CompetitionlevelCG}, and using our \emph{cheat-based selective filtering} strategy from \cref{sec:pair_predictors_bound_error} (modified to threshold by $|1 - \Ccheat(y|x)| \le \beta$).
We implement correctness and semantic equivalence checks using a lookup table, as described in \cref{appendix:sec:digits-of-pi}.
\cref{fig:digitsofpi-precision-recall} (right) shows that filtering by our confidence measure allows generation of more responses with a lower hallucination rate relative to previously-proposed methods.

\begin{figure}
    \centering
    \includegraphics[width=\linewidth]{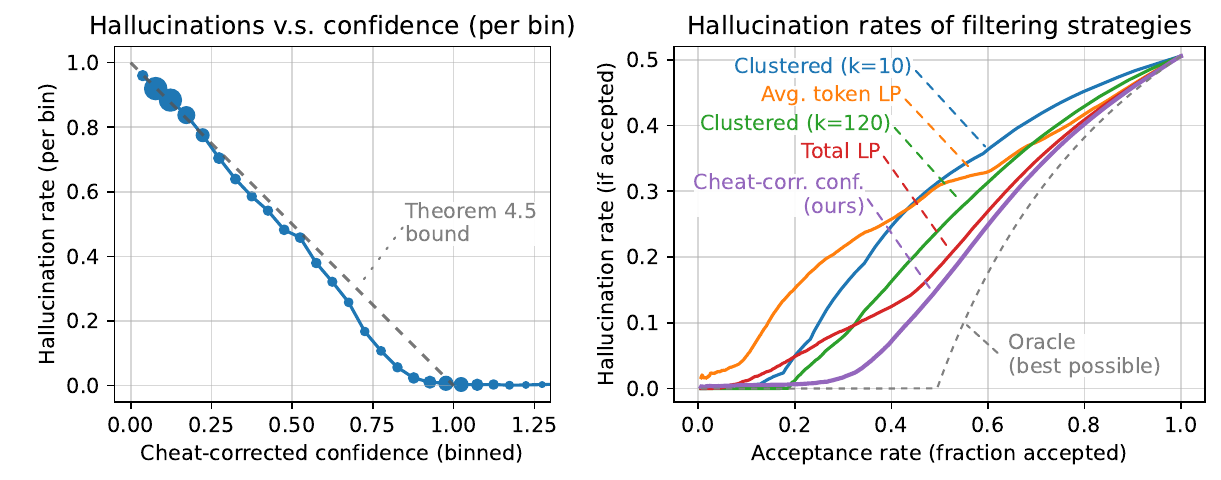}
    \vspace{-2em}
    \caption{
    \textit{Left:} For our digits-of-$\pi$ model, binning samples by $\Ccheat$ shows that hallucination rate is usually $\le 1 - \Ccheat$ as predicted by \cref{thm:conf_hallucination}, although occasionally $\Ccheat > 1$ due to miscalibration.
    \textit{Right:}  Ranking samples by $|1 - \Ccheat(y|x)|$ yields a similar or lower hallucination rate than other common filtering strategies when applied to this model.
    }
    \vspace{-1em}
    \label{fig:digitsofpi-precision-recall}
\end{figure}

\subsection{Safe Offline RL With Unobserved Confounders}\label{sec:frozenlake}
Finally, we show
as a proof of concept
that our approach can detect confounders when doing imitation learning in POMDPs and thus avoid the ``self-delusions'' described by \citet{ortega2021shaking}. We focus on the ``Frozen Lake'' gridworld task \citep{Warrington2020RobustAL}, where agents can take shortcuts across a lake to reach the goal, but a random part of the lake is unsafe to cross in each episode.
We train a model to imitate expert demonstrations, where the experts always know and avoid the location of the unsafe patch, but the model can only see it 50\% of the time.
This partial-observation setting is an extreme example of misspecification, and can be viewed as a restriction on $\Phi(X)$: the model is forbidden from using part of the ``true'' input.
Naive imitation learning in this setting would cause the model to learn to cross the lake randomly, which would be unsafe.

We train an 85M-parameter Transformer to imitate pairs of tokenized trajectories $(Y_1, Y_2)$ drawn randomly from the expert policy, where the two  demonstrations always share the same location of the unsafe patch.
We then apply two of our cheat-corrected decoding strategies (rejection sampling and top-1 search) with the constraint $|1 - \Ccheat| \le 0.05$, and visualize the resulting trajectories in  \cref{fig:frozenlake}. Our strategies behave like ordinary sampling and top-1 search when the unsafe location is visible to the model, but reject paths that cross the lake when the location is hidden, since any such path might have $\pgt(y|x) = 0$. Only the always-safe paths that avoid the lake are kept, since the model is confident that $\pgt(y|x) \approx \phatym(y|x)$ for those trajectories.

\begin{figure}
    \centering
    \includegraphics[width=\linewidth]{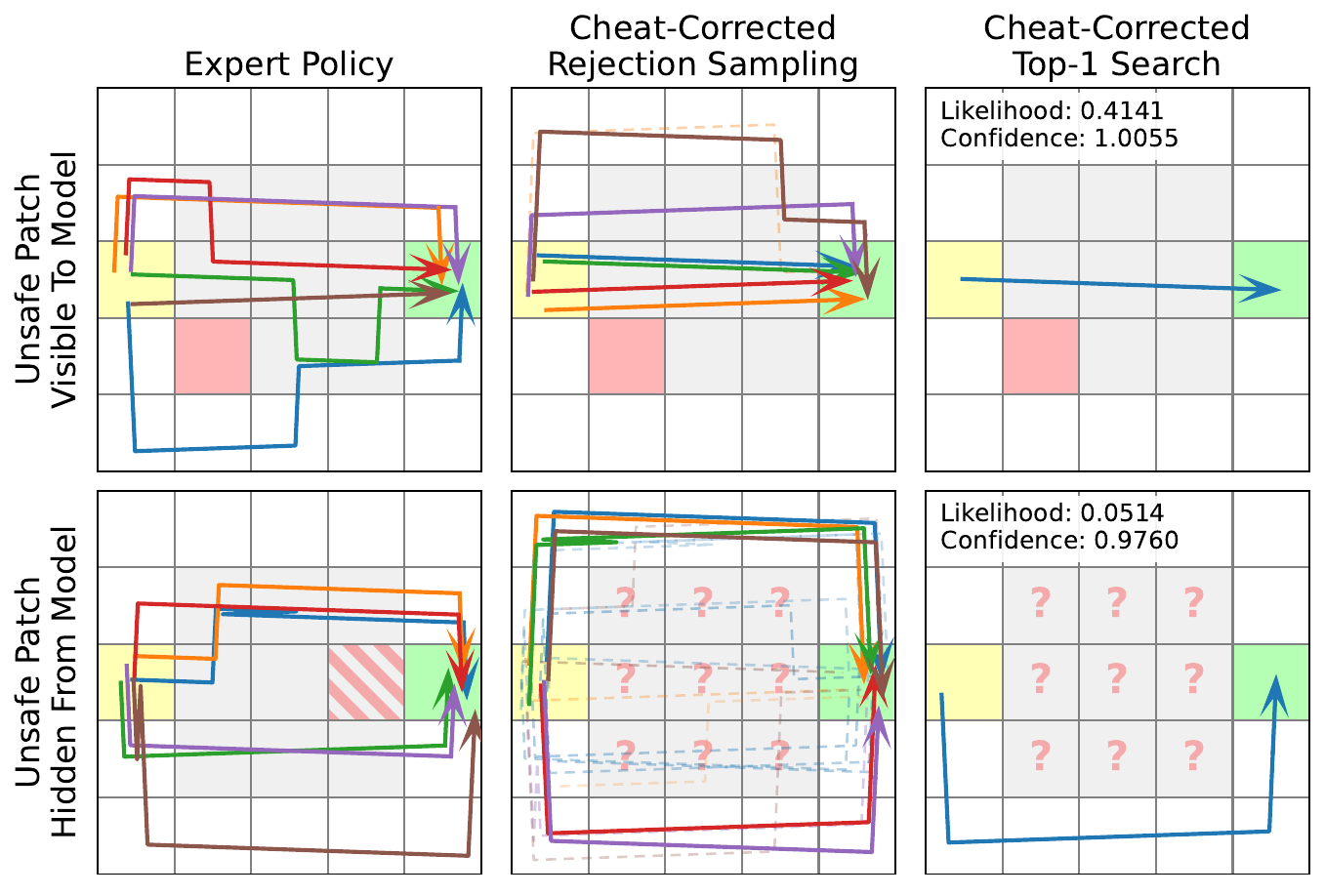}
    \vspace{-2em}
    \caption{Our cheat-corrected decoding strategies (with $\beta = 0.05$) avoid unsafe actions in the ``Frozen Lake'' task. When the unsafe patch (red square) is visible, model samples imitate the expert distribution, and the highest-likelihood path crosses the lake. When it is hidden, $\Ccheat$ is low for possibly-unsafe sampled paths (dashed lines), so our decoding strategies reject them in favor of safe paths.}
    \vspace{-1em}
    \label{fig:frozenlake}
\end{figure}
\section{Discussion}
We have presented a principled new approach for identifying the gaps between a model $\phaty$ and the ground truth $\pgt$, based on a remarkable equivalence between second-order calibration and pair prediction, and proven that calibrated pair predictors can be used to construct provably-correct bounds on $\pgt$. We have further demonstrated that our scheme is practically effective on both classification and sequence-modeling tasks, even without perfect calibration over $\cY\times\cY$.
Although paired responses may not be available for all datasets,
collecting paired fine-tuning data may still be easier than applying architecture-dependent uncertainty quantification strategies, especially for large models.
We are optimistic that our procedure will scale up
to this use case,
and are eager to explore this direction in future work.

\section*{Acknowledgements}
We would like to thank Dami Choi for helping with an early prototype of the idea, and Gustaf Ahdritz, Nikhil Vyas, Zelda Mariet, and Zi Wang for useful discussions. We are also thankful to Ayoub El Hanchi, David Glukhov, Stephan Rabanser, and Jasper Snoek for providing valuable feedback on the paper draft.
Resources used in preparing this research were provided in part by the Province of Ontario, the Government of Canada through CIFAR, and companies sponsoring the Vector Institute. We acknowledge the support of the Natural Sciences and Engineering Research Council of Canada (NSERC), RGPIN-2021-03445.

\section*{Impact Statement}
Our work proposes a general strategy for training a model $\phaty$ to accurately report how well it is able to fit an arbitrary process $\pgt$ on a per-input (or, precisely, per-equivalence-class) level. A large number of machine learning problems can be posed in this form, and the consequences of applying our technique would likely depend on the particular application. Overall, however, we hope our technique will make it easier to build safer and more reliable machine learning systems, by ensuring that they avoid taking unsafe actions or making unfair decisions when they are unable to accurately perform their intended tasks.

\bibliography{references}
\bibliographystyle{icml2024}

\newpage
\appendix
\onecolumn

\clearpage
\section{Sample Visualizations}\label{appendix:sample_visualizations}
In this section we present samples from our models for the Digits of Pi and Frozen Lake tasks (\cref{sec:digitsofpi,sec:frozenlake}).

\begin{figure*}[h!]
    \centering
\fbox{\includegraphics[width=0.675\linewidth,trim={0 5in 0 0},clip]{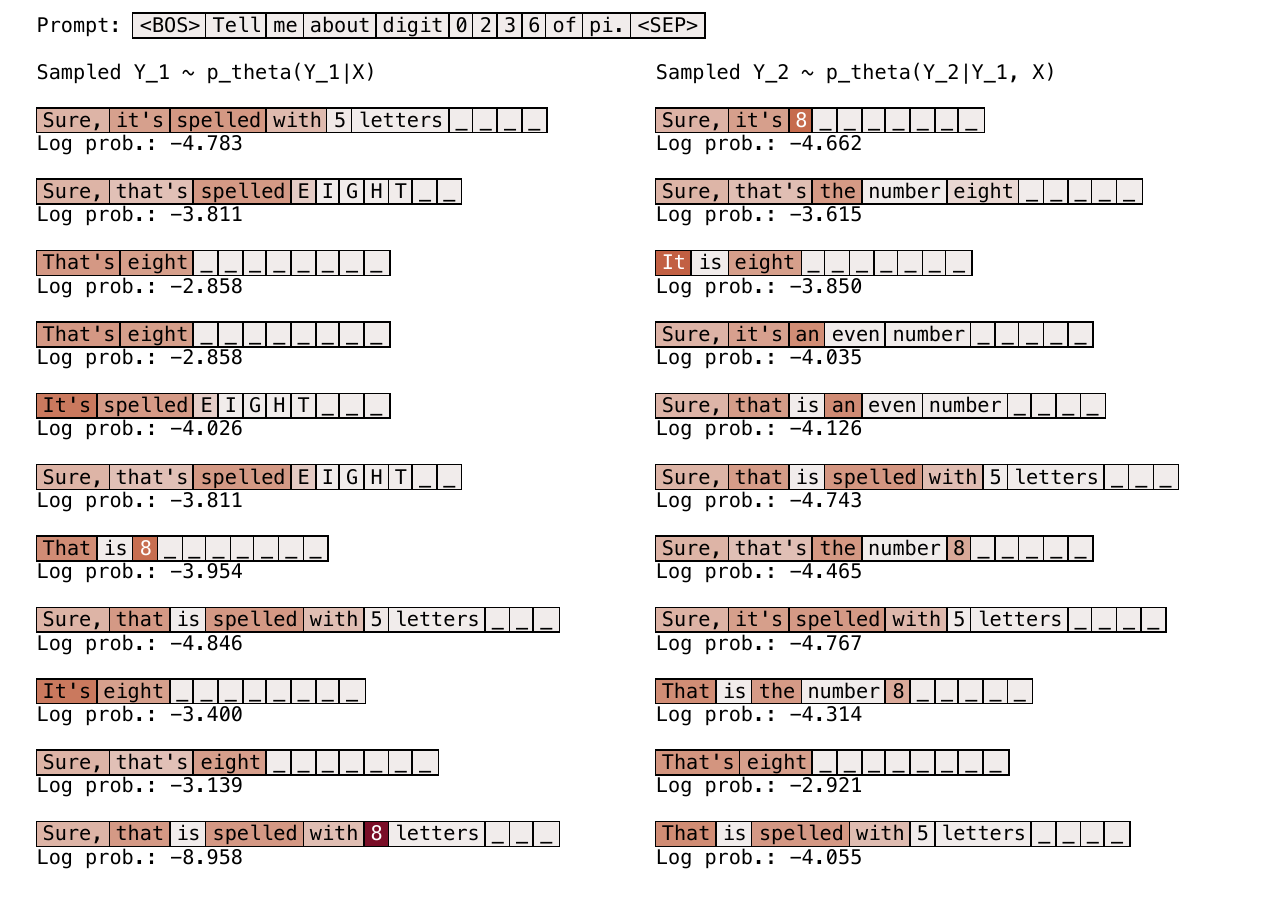}}
    \caption{Results of sampling pairs $(Y_1, Y_2)$ from the model $\phatyy$ when asked about the 236th digit of $\pi$, which it ``knows'' is eight. (Our method does not actually require sampling $Y_2$; we show samples for illustrative purposes only.) Color denotes likelihood, with red denoting less-likely tokens. The left column is $Y_1$ and the right column is $Y_2$ (drawn conditional on $Y_1$); each row is an independent pair of samples for the prompt at the top. Note that the last row's  $Y_1$ is a low-probability mistake which was sampled due to the high temperature. (The model ``knows'' eight is spelled with 5 letters, so $Y_2$ is inconsistent with $Y_1$).}
    \label{appendix:fig:pi_samples_independent_1}
\end{figure*}
\begin{figure*}[h!]
    \centering
\fbox{\includegraphics[width=0.675\linewidth,trim={0 5in 0 0},clip]{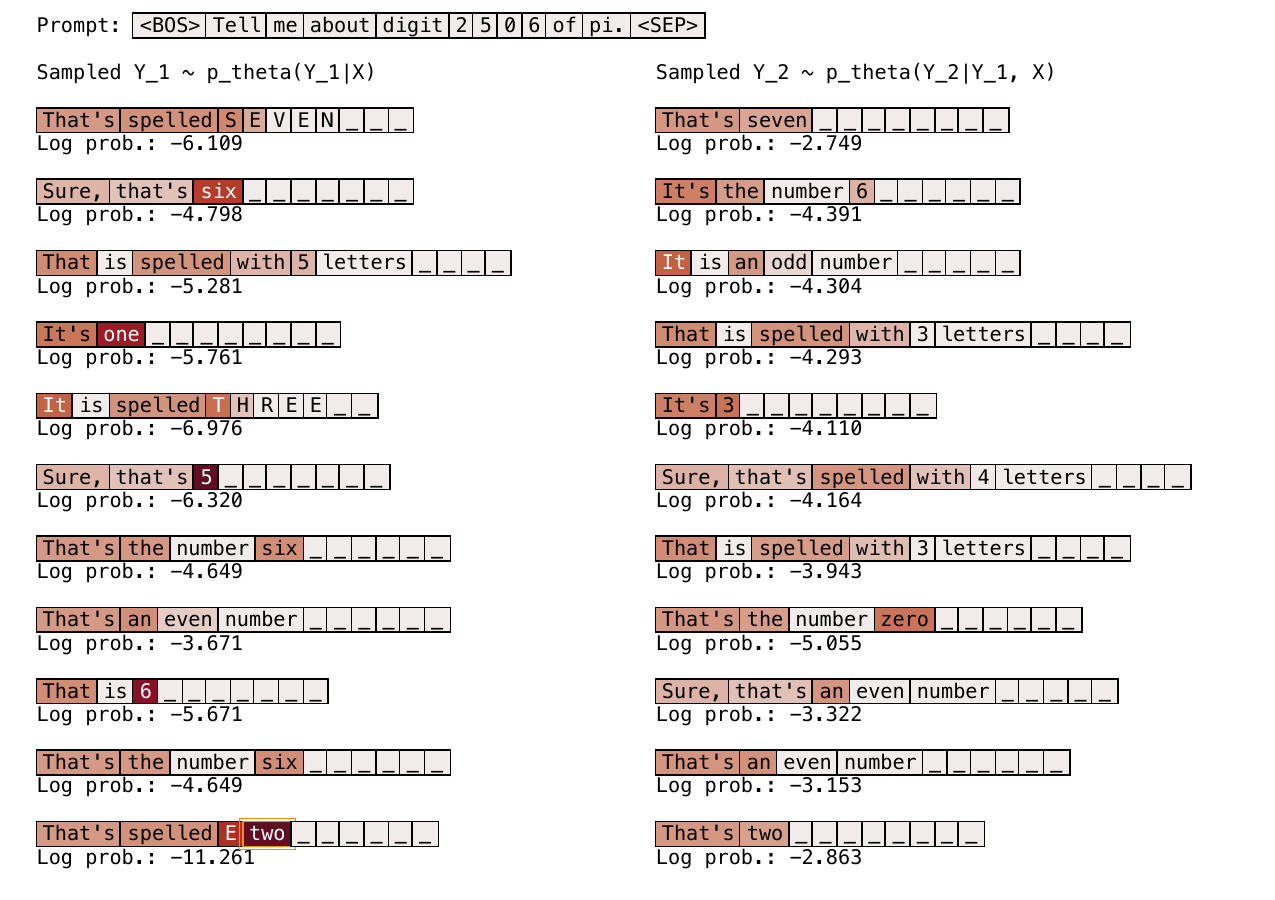}}
    \caption{Results of sampling pairs $(Y_1, Y_2)$ from the model $\phatyy$, when asked about the 2506th digit of $\pi$ (which it has not learned). The sampled $Y_2$ is usually consistent with the $Y_1$ sample, indicating that the model is ``cheating'' well. The last sample of $Y_1$ is malformed due to sampling a low-probability token.}
    \label{appendix:fig:pi_samples_independent_2}
\end{figure*}

\begin{figure*}[p]
    \centering
\fbox{\includegraphics[width=\linewidth,trim={0 3.25in -0.25in -0.25in},clip]{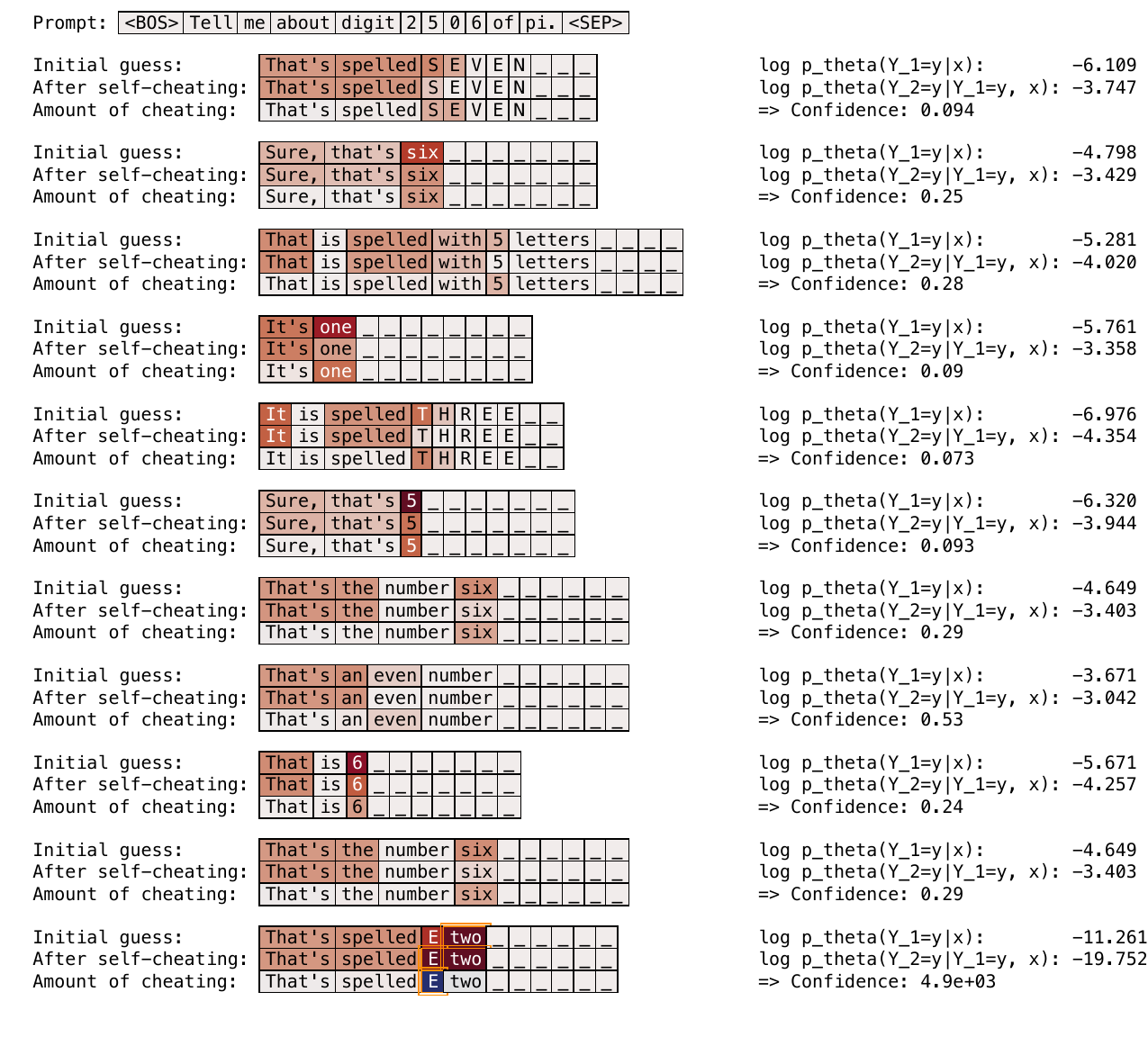}}
    \caption{Scoring samples using our cheat-corrected epistemic confidence, for the 2506th digit of $\pi$ (which it has not learned). Conditioning on $Y_1=y$ reveals information about this digit, so the log probability increases when outputting $Y_2=y$, and we can use the magnitude of the increase as a measurement of confidence. We can also attribute this increase to individual tokens (with red in the ``Amount of cheating:'' rows indicating a token whose likelihood increased after cheating). The last sample is malformed, so has very low probability as either $Y_1$ or $Y_2$, which leads to an outlier confidence greater than 1. We recommend discarding samples with confidences significantly larger than one, e.g. by keeping only those with $|1 - \Ccheat| \le \beta$.}
    \label{appendix:fig:pi_samples_scored_2}
\end{figure*}

\begin{figure*}[p]
    \centering
\fbox{\includegraphics[width=\linewidth,trim={0 3.25in -0.25in -0.25in},clip]{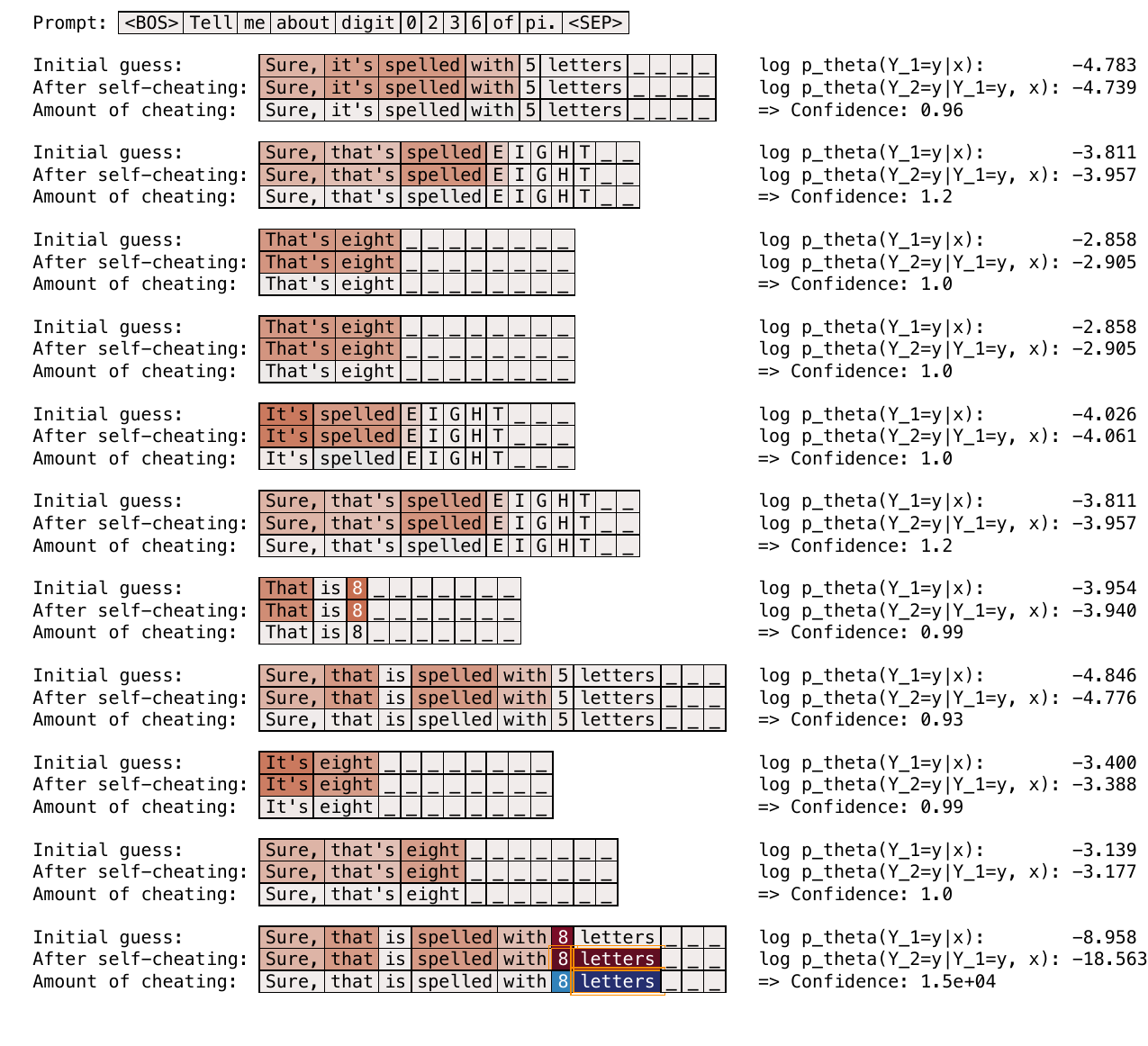}}
    \caption{Scoring samples using our cheat-corrected epistemic confidence, for the 236th digit of $\pi$ (which it ``knows''). We repeat each response twice, comparing the probabilities $\phatym(y|x)$ and $\phatyc(y|y,x)$. Color indicates log probability for ``Initial guess'' and ``After self-cheating'', and differences between log probabilities for ``Amount of cheating''.
    For this prompt, conditioning on $Y_1$ does not significantly change the prediction of the model, because it already knows the value of the 236th digit. However, in the last sample, the probability decreases because the originally-sampled guess was a mistake (see \cref{appendix:fig:pi_samples_independent_1}), leading to an outlier confidence value greater than one.}
    \label{appendix:fig:pi_samples_scored_1}
\end{figure*}

\begin{figure*}[p!]
    \centering
\includegraphics[width=0.9\linewidth]{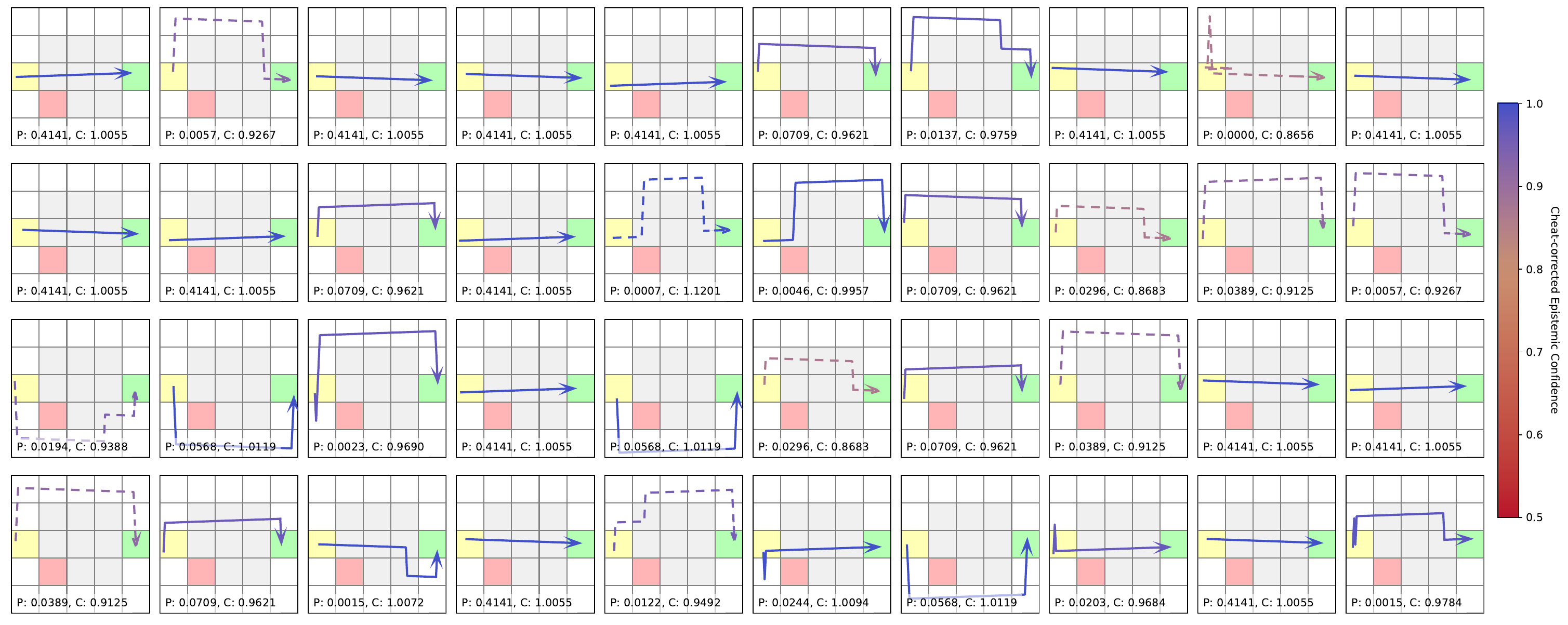}
    \vspace{-1em}
    \caption{Model samples and confidences for the fully-observable version of the ``Frozen Lake'' task, with the unsafe patch in the bottom left. Dashed trajectories indicate samples that we would reject using a $|1 - \Ccheat| \le 0.05$ threshold. We add a small diagonal offset when plotting so that it is easier to follow paths that backtrack; the model itself only predicts discrete actions (left, right, up, down) and moves between grid cells. There is a fair amount of diversity among samples, although our strict decoding strategy does occasionally reject safe paths.}
    \label{appendix:fig:frozen_lake_samples_r1c1}
\end{figure*}
\begin{figure*}[p!]
    \centering
\includegraphics[width=0.9\linewidth]{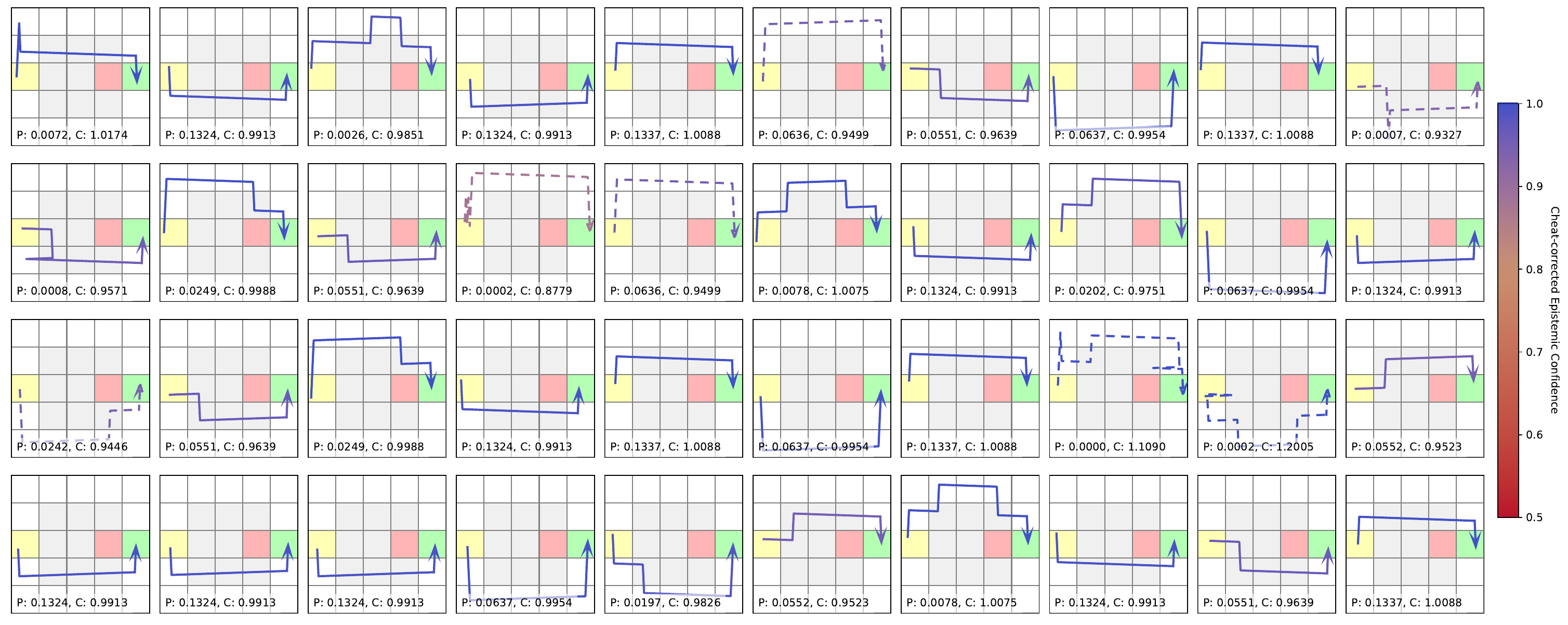}
    \vspace{-1em}
    \caption{Model samples and confidences for the fully-observable version of ``Frozen Lake'' with the unsafe patch in the middle right.}
    \label{appendix:fig:frozen_lake_samples_r3c2}
\end{figure*}
\begin{figure*}[p!]
    \centering
\includegraphics[width=0.9\linewidth]{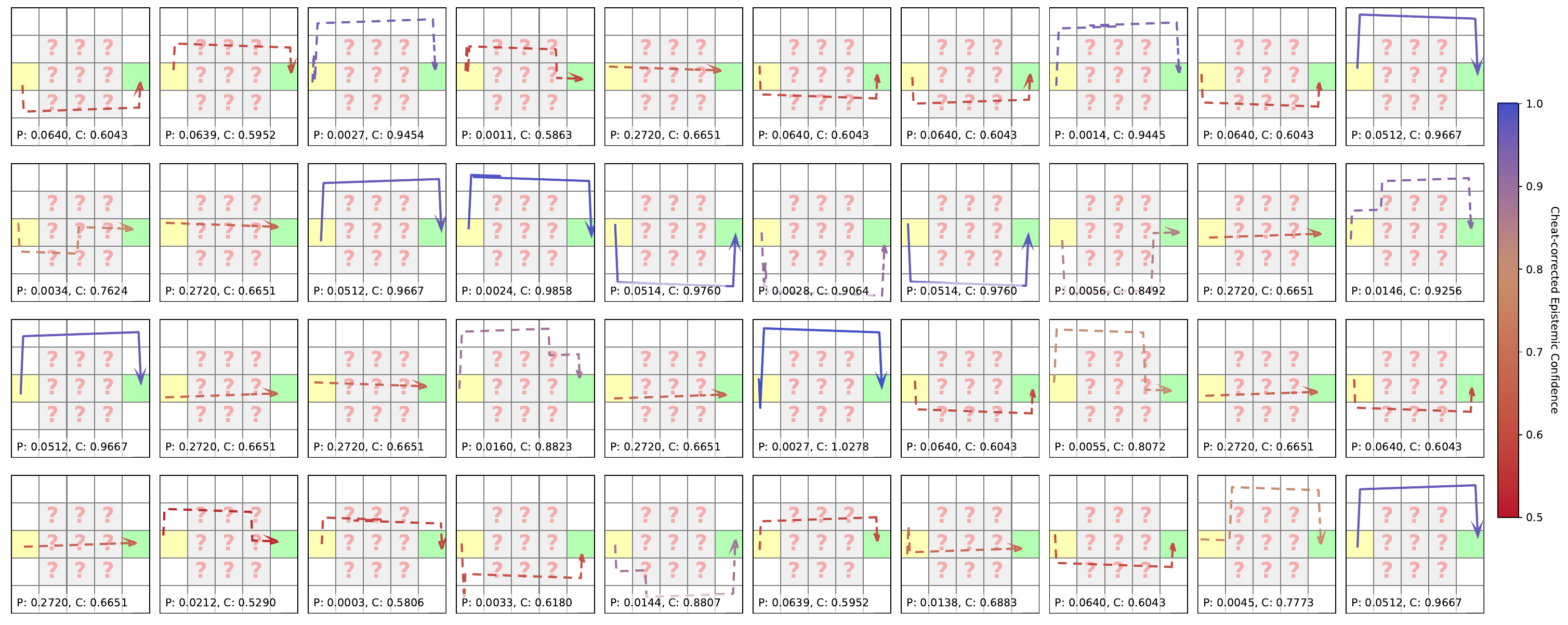}
    \vspace{-1em}
    \caption{Model samples and confidences for ``Frozen Lake'' when the unsafe patch is hidden. Note that samples that cross the lake have much lower confidence when the unsafe patch is hidden, relative to similar trajectories in \cref{appendix:fig:frozen_lake_samples_r1c1,appendix:fig:frozen_lake_samples_r3c2}.}
    \label{appendix:fig:frozen_lake_samples_hidden}
\end{figure*}

\section{Our Distribution-Free Confidence Intervals}\label{appendix:distnfree_experiment}
In this section, we show the results of applying \cref{thm:distnfreebound} to the 1D binary regression problem in \cref{fig:intro_1d_problem}.

\begin{figure*}[h!]
    \centering
\includegraphics[width=0.95\linewidth]{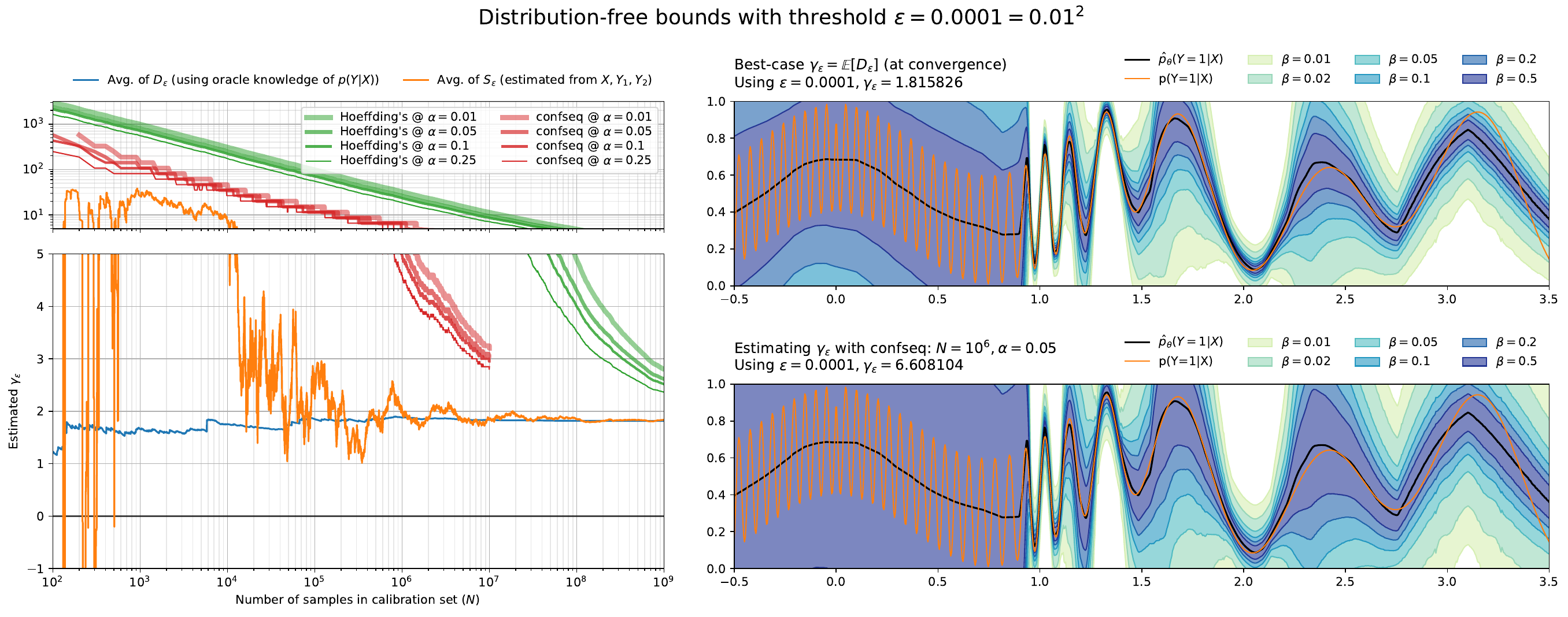}
\includegraphics[width=0.95\linewidth]{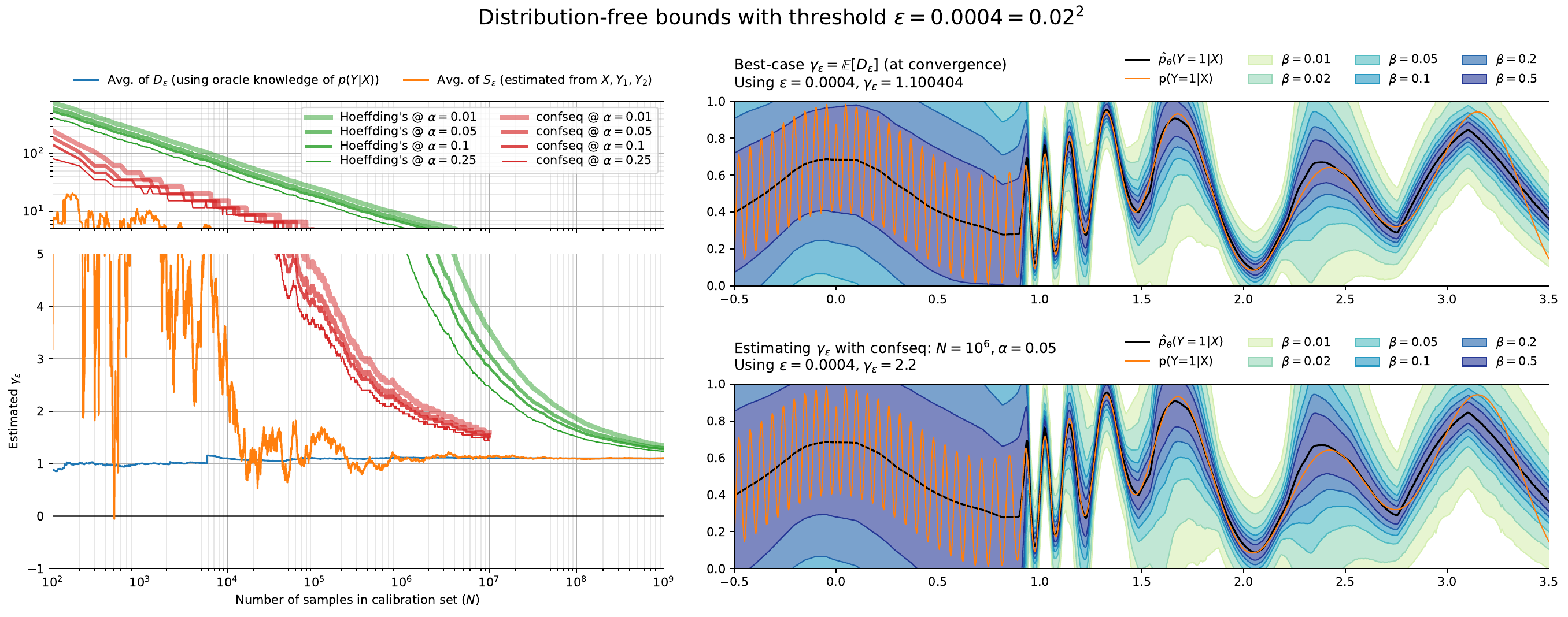}
\includegraphics[width=0.95\linewidth]{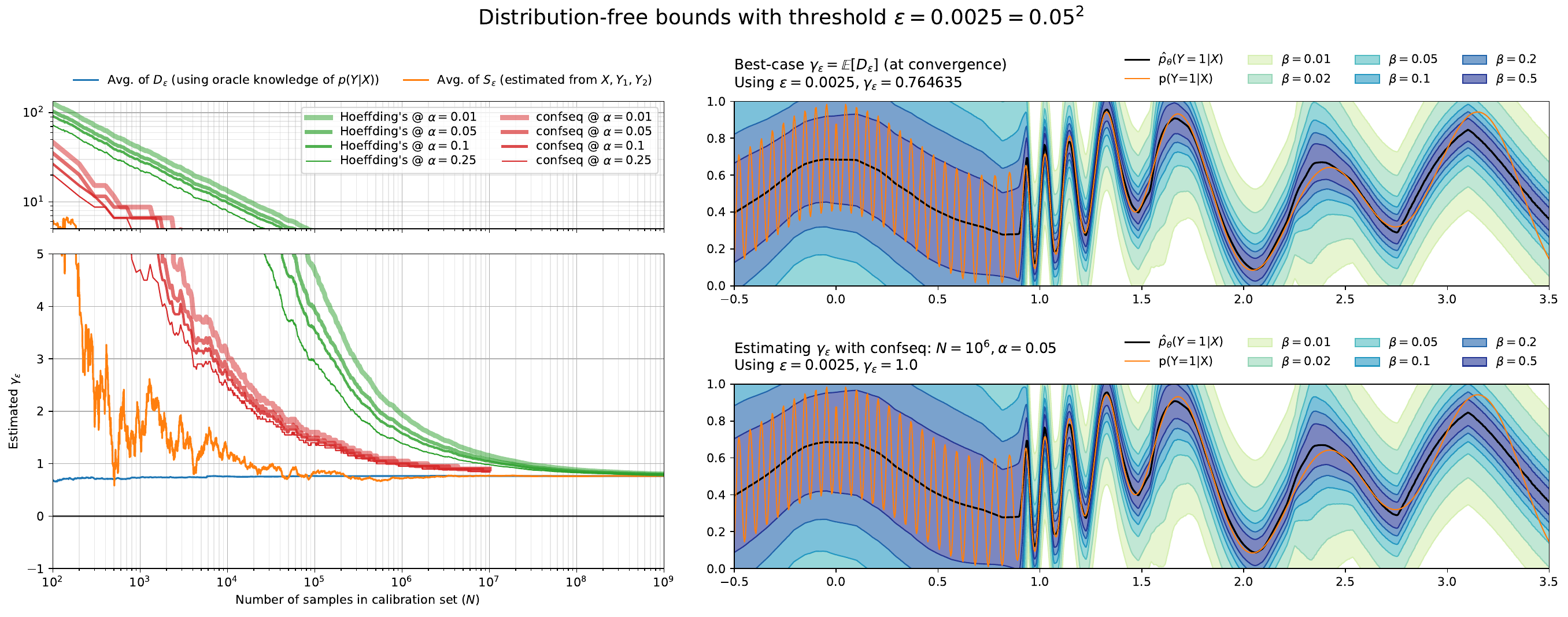}
    \vspace{-1em}
    \caption{Visualization of our distribution-free bound for the toy 1-D binary regression problem in \cref{fig:intro_1d_problem}, with $\varepsilon$ set to $0.01^2, 0.02^2,$ or $0.05^2$. Left: Convergence of $\gamma_\varepsilon$ based on Hoeffding's inequality and \texttt{confseq}, with running averages of $D_\varepsilon$ and $S_\varepsilon$ for reference. Right: Resulting confidence intervals for $p(Y|X)$, using either the best-case $\gamma_\varepsilon = \E[D_\varepsilon]$ or a value of $\gamma_\varepsilon$ returned by \cref{alg:distfree_bound}.}
    \label{appendix:fig:distfree_good}
\end{figure*}

\clearpage

\begin{figure*}[t!]
    \centering
\includegraphics[width=\linewidth]{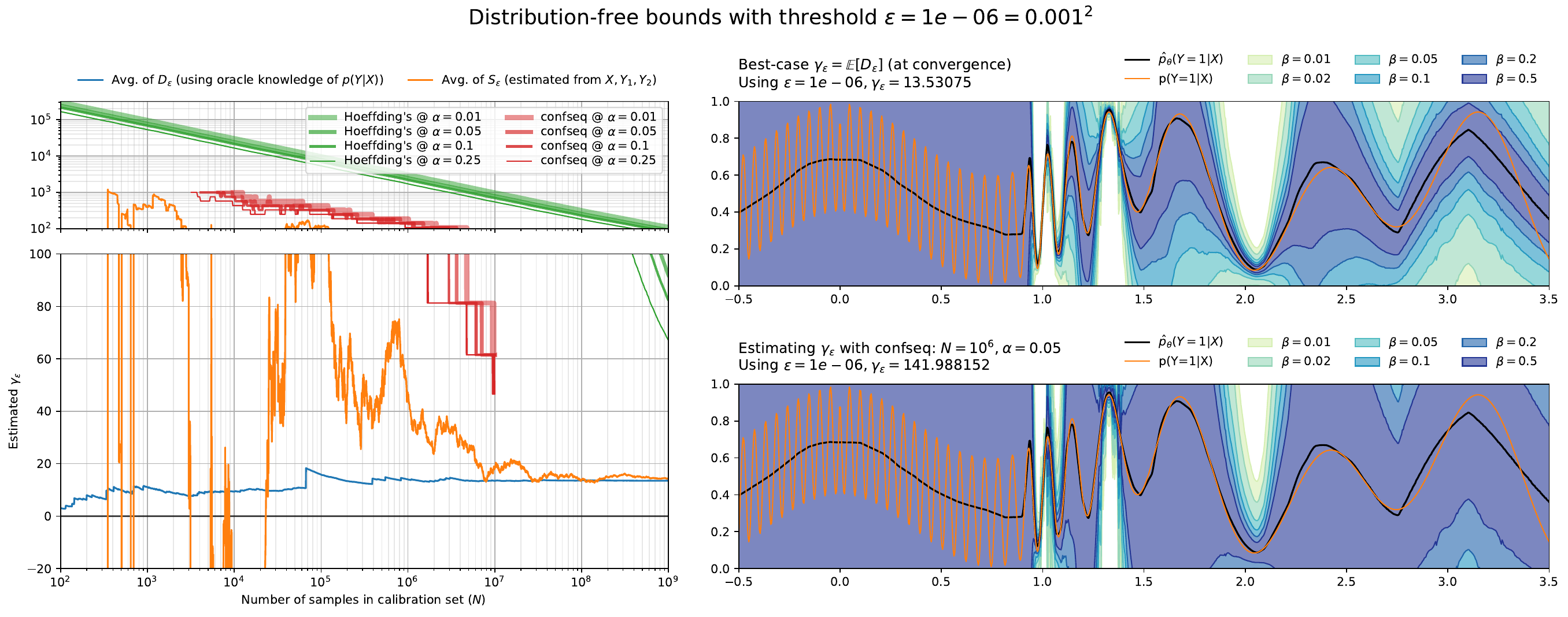}
\\[1em]
\includegraphics[width=\linewidth]{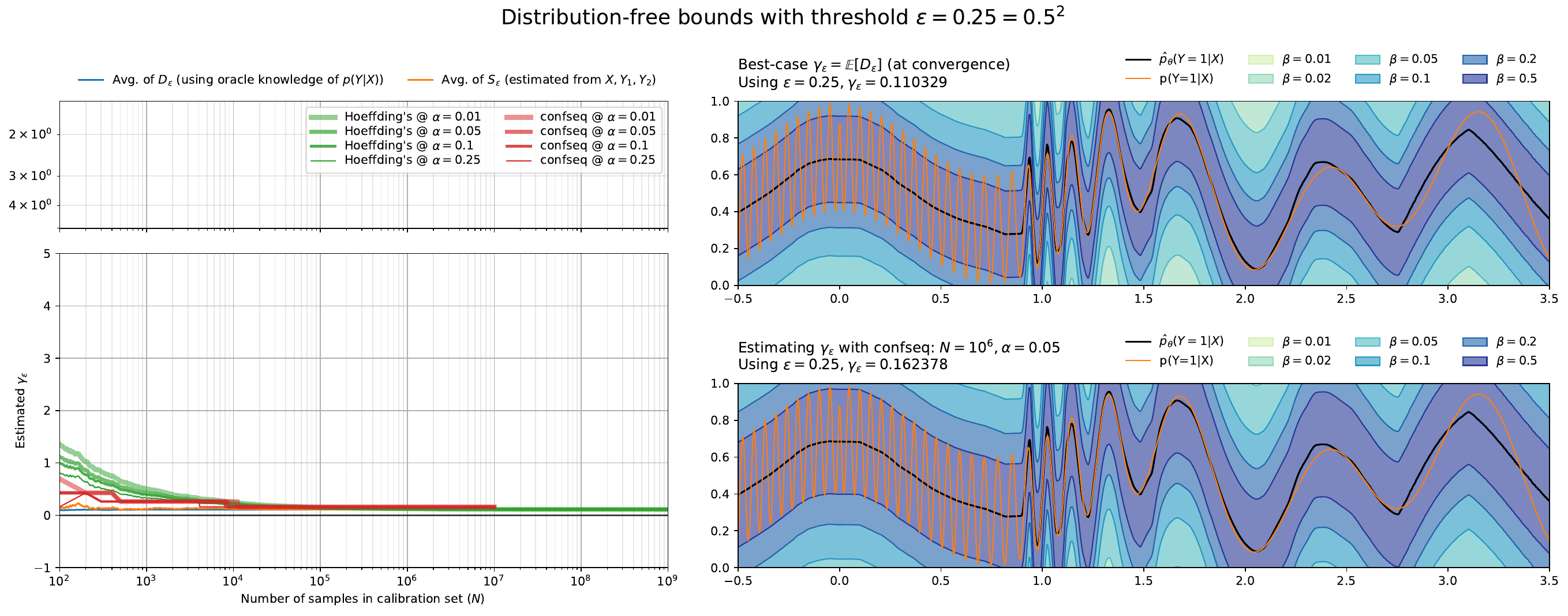}
    \vspace{-2em}
    \caption{Visualization of our distribution-free bound with $\varepsilon$ set to $0.001^2$ (leading to a blowup of $\gamma_\varepsilon$ and a very pessimistic bound) or $0.5^2$ (for which the bound ignores $v^\textsc{Cheat}_\theta$ entirely and has a constant width for all $X$, because $\vcheat(y | x) \le 0.5^2$ everywhere).}
    \label{appendix:fig:distfree_extreme}
\end{figure*}

We hold $(\phatym, \sigmacheat)$ fixed (to the ``Cheat-corrected NN'' described in \cref{appendix:training_1d_regression}), and study the behavior of the bound for different confidence interval algorithms, variance thresholds $\varepsilon$, failure tolerances $\alpha$, and calibration set sizes $N$. We compare two confidence interval algorithms, Hoeffding's inequality \citep{Hoeffding1963ProbabilityIF}, discussed in 
\cref{appendix:distnfree_theory}, and \texttt{confseq} \citep{waudby2020estimating}, desribed below.

In \cref{appendix:fig:distfree_good,appendix:fig:distfree_extreme}, we enumerate $\varepsilon \in \{0.001, 0.01, 0.02, 0.05, 0.5\}$ and plot values of $\gamma_\varepsilon$ as the number of calibration set examples $N$ ranges from 100 to $10^9$. For comparison, we also plot running average estimates of $\E[D_\varepsilon]$ (which require knowledge of $p(Y|X)$) and of $\E[S_\varepsilon]$ (as computed in \cref{alg:distfree_bound}). We also plot the resulting confidence intervals for $p(Y|X)$ at various confidence levels $\beta$, using either an estimate based on $10^6$ calibration set examples or the best-possible value $\E[D_\varepsilon]$ at convergence (computed using oracle knowledge of $p(Y|X)$).

Overall, we observe that \texttt{confseq}'s bounds are considerably tighter than those based on Hoeffding's inequality. Because our model $\phatyy$ is not perfectly calibrated on pairs, setting $\varepsilon$ too small leads to a blowup of $\E[D_\varepsilon]$ and an ineffective bound. Smaller $\varepsilon$ also reduces the rate at which $\gamma$ converges, so setting it to a larger value may be necessary if there is a limit on the size of the calibration set. On the other hand, setting $\varepsilon$ too large produces confidence intervals that are the same width everywhere, ignoring $\vcheat$ and instead using the marginal variance of $\pgt(1|X)$ across all $X$.

Note that, regardless of $\varepsilon$, the resulting bounds are provably correct with high probability (in the sense described by \cref{thm:distnfreebound}). However, when $\varepsilon$ is set too small, it is more likely that the bound is overly conservative, and when $\varepsilon$ is too large, the coverage guarantees are more likely ``trade off'' errors between values of $X$, assigning conservative bounds to some regions and under-covered bounds to others so that the overall coverage target $\beta$ is met. This can be seen in the plot for $\varepsilon = 0.25$. (In a sense \emph{every} such bound must trade off errors between values of $X$, because after fixing $\pgt$, for each $x$ the value of $\pgt(y|x)$ is either in the interval or not. But if $\varepsilon$ is small and the model is well calibrated, this trading-off only occurs between examples in the same equivalence class, i.e. with the same value of $\Phi(x)$.)

\textit{\texttt{confseq} implementation details:} For our \texttt{confseq} bounds, we use the \texttt{betting\_cs} function from the \texttt{confseq} Python package,\footnote{\url{https://github.com/gostevehoward/confseq}} which implements the algorithm described by \citet{waudby2020estimating}. We rescale our $S_\varepsilon$ values so that they are bounded between 0 and 1, as assumed by the algorithm. \texttt{betting\_cs} maintains a finite set of hypotheses about $\E[S_\varepsilon]$ and uses hypothesis testing to reject them; these hypotheses are evenly spaced over the unit interval by default, but we modify it slightly to choose a set of hypotheses that are more-closely concentrated around 0.5, which gives higher precision for $\gamma_\varepsilon \approx 0$ after inverting our rescaling. (The finite hypothesis set is the reason for the discrete jumps in the estimates produced by \texttt{confseq} in \cref{appendix:fig:distfree_good,appendix:fig:distfree_extreme}.) We configure \texttt{betting\_cs} with a prior mean of $1$ and a prior variance of $\frac{0.5^4}{\varepsilon^2}$ for $S_\varepsilon$, which correspond to a prior mean of $\frac{1}{2} + \frac{1}{2\varepsilon}$ and prior variance of $0.5^6$ after rescaling $S_\varepsilon$ to the unit interval. Since betting-based confidence intervals are more computationally expensive than confidence intervals from Hoeffding's inequality, we only run \texttt{confseq} for calibration set sizes smaller than $10^7$.

\clearpage
\section{Additional discussion and related work}\label{app:sec:discussion}

\paragraph{Pair prediction (two $Y$ for the same $X$) v.s. ``joint prediction'' (different $X$s and $Y$s) for epistemic uncertainty:}
As motivation for the Epinet uncertainty-quantification technique, \citet{Osband2021EpistemicNN} have argued that uncertainty-aware agents should be judged not based on their uncertainty about $Y$ for individual inputs $X$, but instead based on their joint distribution over a sequence of $Y^{(i)}$ drawn for a sequence of inputs $X^{(i)}$; they refer to this as ``joint prediction''. While somewhat similar to our proposed pair-prediction formalism in terms of motivation, the focus and applicability of the approaches is quite different.

\emph{Probabistic structure:} The joint prediction criterion assumes the model has some hierarchical structure, such that it is possible to express a joint distribution over the outcomes $Y^{(i)}$ for different inputs $X^{(i)}$. Bayesian neural networks and Epinets satisfy this criterion, but not all neural networks express a joint in this way. In particular, a cheat-corrected pair-prediction neural network (as we propose) does not assume any joint distribution over outcomes for different $X$; its outputs can be converted into well-calibrated pointwise estimates of variance, but each prediction is made pointwise (e.g. for this particular $x$ or the $x$es in a particular equivalence class).

\emph{Evaluation criterion:} \citet{Osband2021EpistemicNN} propose to use joint prediction primarily as an evaluation metric, based on theoretical results showing that joint predictions perform well for decision making \citep{wen2021predictions}. This can be interpreted as measuring how quickly a model can ``learn'' from new data to improve its predictions on future data points. In contrast, our work focuses on pair prediction as a way to train a model to be second-order calibrated, which then lets us estimate how accurately a model predicts its distance from $p(Y|X)$. Our evaluation metric is then the pointwise calibration of the second-order estimates.

\emph{Training objective:} Our proposed training objective in directly trains a model to predict pairs, and our distribution-free adjustment procedure also directly uses paired data. This ensures that our approach can be statistically valid even if the model is misspecified or computationally limited, but requires the data collection process to be modified. On the other hand, \citet{Osband2021EpistemicNN} do not train their models based on a joint prediction objective, but instead show that a per-sample log-likelihood objective leads to good joint predictions under the assumption that the data was generated by a distribution with a specific known form. This implies that the Epinet training objective is not necessarily second-order calibrated or robust to misspecification, and we find empirical evidence of this in \cref{sec:cifar10h}.

\paragraph{Pair prediction v.s. minimum Bayes risk / repeated sampling techniques:}
Many postprocessing-based techniques improving model outputs using multiple samples, including clustering-based approaches, can be interpreted as instances of \emph{minimum Bayes risk} (MBR) decoding \citep{bertsch2023s}. In MBR decoding, after obtaining a model $\hat{p}_\theta$ approximating some generative process, actions are selected not based on their likelihood under the model, but instead based on some error function $L(y, y')$ that compares possible outputs; an output $y'$ is ``good'' if it achieves a low error in expectation across alternative outputs $y$ sampled from the model  $\hat{p}_\theta(Y|X=x)$. For instance, $L$ might return 1 if two outputs are semantically equivalent.

Although our approach and MBR decoding both draw repeated samples from a conditional distribution of $Y$ given $X$, they differ on \emph{which} distribution is sampled. In MBR decoding, the training data usually consists of only one $Y$ drawn from $p(Y|X)$ for each $X$, but multiple samples are drawn from $\hat{p}_\theta(Y|X)$ and compared at inference time. In contrast, in our pair-prediction technique, the training data must consist of two samples $Y_1, Y_2$ drawn from $p(Y|X)$ for each $X$, but at inference time we can sample and score single outputs from $\hat{p}_\theta(Y|X)$.

MBR also requires specification of a task-relevant error function $L(y, y')$, and does not distinguish between aleatoric and epistemic uncertainty in $\hat{p}_\theta(Y|X)$, but instead distinguishes between ``risky''/``unusual'' and ``safe''/``common'' samples using $L(y, y')$. In contrast, our technique is task-agnostic and explicitly distinguishes aleatoric and epistemic uncertainty.
Note that, because it is task-agnostic, our approach may report uncertainty about hard-to-predict parts of $Y$ even if they are not relevant to the downsteram task, whereas MBR decoding can explicitly ignore the irrelevant parts when computing $L$.

\paragraph{Comparison to other distribution-free statistical guarantees:} Our adjustment procedure in \cref{sec:distnfree} has some similarities to previous algorithms for distribution-free prediction.

Conformal prediction is a particularly common and powerful form of distribution-free inference; see \citet{angelopoulos2021gentle} for an introduction. In general, conformal prediction lifts a predictor of \emph{points} (i.e. samples $y_i \in \cY$) into a predictor of \emph{prediction sets} (subsets of $\cY$) such that, for a new input drawn from the same distribution, the result lies in the predicted set with high probability. The basic idea is to associate a ``conformal score'' to each outcome in $\cY$ that measures how badly the predictor was wrong (e.g. the prediction error), estimate an upper quantile of the conformal scores for the actual outcomes in the dataset, then construct a prediction set by removing any observation whose conformal score would be higher than this quantile). This idea is closely linked to that of hypothesis testing.

Directly applying conformal prediction to a binary classification problem produces prediction sets that are subsets of $\cY = \{0,1\}$, but this is not ideal if there is uncertainty about $Y$, because then the best prediction set will often be $\{0,1\}$ itself, which is trivial and uninformative; similar issues may also arise for larger $\cY$ in high-uncertainty settings.
Related to our work, \citet{barber2020distribution} investigated the feasibility of constructing confidence intervals for the probability $p(Y=1|X)$ instead of the samples $Y$ themselves. Working under the assumption that the data consists of $(X, Y)$ pairs and that each $X$ is seen at most once, \citeauthor{barber2020distribution} proved that any confidence intervals for $p(Y=1|X)$ must necessarily also be a prediction set for $Y$ itself. In other words, the confidence interval cannot be a tight bound on the true $p(Y=1|X)$, since it must include at least one of the endpoints $0$ or $1$ with high probability, and may need to cover the whole unit interval for highly-stochastic $Y$. (Our approach avoids this limitation by assuming each input $X$ is seen twice.)

\citet{gupta2020distribution} study the relationship between calibration, grouping functions, confidence intervals, and predictive sets for binary classification problems. They introduce the notion of confidence intervals and predictive sets with respect to a function $f$, where $f$ plays the same role as our grouping function $\Phi$, and study methods for bounding the true expectation $\E[Y | f(X)]$. They prove that, in general, parametric recalibration methods cannot be distribution-free calibrated, but if outputs of $f$ are discretized to a finite set of bins first, then it is possible to construct distribution-free confidence intervals for the conditional probability $\E[Y | f(X)]$. A similar guarantee about calibration error was given by \citet{kumar2019verified}, who also proposed an efficient combined scaling-binning scheme, and a refined analysis that allows re-using samples was also given by \citet{gupta2021distribution}.

We note that confidence intervals for the expectation $\E[Y | f(X)]$ with respect to $f(X)$ are not the same as confidence intervals for the true conditional probability $p(Y=1|X) = \E[Y|X]$. Constructing a confidence interval for $\E[Y | f(X)]$ allows you to guarantee that your model is nearly first-order calibrated; it gives an interval of values that is likely to contain the answer to the question ``across all of the inputs for which my model's output is $\phi$, how many will have $Y = 1$?'' However, it does not tell you whether your model is doing a good job at separating examples with different true conditional probabilities. In contrast, our procedure directly produces a confidence interval for $p(Y=1|X)$; it gives an interval of values in answer to the question ``what is the chance that $Y=1$ for this specific $x$?'' such that the answer is likely\footnote{With probability at least $1 -  \alpha$.} to be correct for most\footnote{At least $1 - \beta$ of them.} randomly-chosen $x$.

For an estimator with finitely-many bins, and an infinite number of recalibration examples, the confidence intervals for $\E[Y | f(X)]$ will eventually converge on the exact value of $\E[Y | f(X)]$. However, our confidence intervals for $p(Y=1|X)$ may \emph{never} converge to the exact value of $p(Y=1|X)$ if the model is unable to distinguish $X$s with different label probabilities. This is unavoidable, since our model may not have capacity to express $p(Y=1|X)$ without additional assumptions (whereas a lookup table always has enough capacity to estimate $\E[Y | f(X)]$ over finitely many bins). Nevertheless, if we happen to be lucky, and our model \emph{is} actually able to predict the exact value for $p(Y=1|X)$ (and is both calibrated and confident about this prediction, i.e. $\vcheat(y|x) = 0$), then our confidence intervals will converge to that value. Our procedure is thus \emph{adaptive} to the complexity of the specific dataset being used while remaining correct without additional assumptions, a desirable property for a distribution-free algorithm.

\paragraph{Previous uses of ``second-order calibration'' terminology:}
The term ``second-order calibration'' has been previously used by \citet{muralidharan2015second} to refer to a particular method for adjusting an arbitrary system to give approximate posteriors over a latent real-valued parameter, under strong distributional assumptions. Here ``calibration'' is used in the sense of ``a calibration procedure'' rather than as ``the property of being calibrated'', and the goal is to distinguish error in estimating the parameter from the intrinsic noise in the response-generating process, using an estimate of variance. The calibration procedure relies on binning examples $x$ based on the output $t$ of a learned model, assuming that the real-valued true parameter $\theta$ of interest follows a simple parameteric family, and then fitting the parameters of the family for each bin separately using maximum (marginal) likelihood; this is an extension of (re)calibration procedures that try to post-process a non-calibrated model to make it more (first-order) calibrated (e.g. \citet{kumar2019verified}). Since true data for the parameter is not available, the focus of the work is primarily on ensuring that the simple parametric model fits the data well rather than on measuring the accuracy of the variance estimates.

We are not aware of any previous work that uses ``second-order calibration'' to refer to a formally-defined \emph{property} of a predictive model rather than to a technique for postprocessing an existing model, nor any that considers it in the sense of predicting the squared error between a predicted discrete distribution and an unknown ground-truth discrete distribution (e.g. for a classifier or generative model) without distributional assumptions.

\paragraph{Joint, marginal, and class-wise calibration:} Our definition of first-order calibration requires that all elements of the output joint distribution match their true expectation conditional on a \emph{single} grouping function (or, equivalently, conditional on the full \emph{vector} of model outputs). This is sometimes referred to as being ``jointly calibrated''. There are also weaker definitions of calibration. Following the terminology of \citet{perez2022beyond}, ``classwise calibrated'' models make individually-calibrated binary predictions about each possible class $y$  \citep{zadrozny2002transforming}, and ``top-label calibrated'' models first identify a most likely label and then make a calibrated binary predictions about the correctness of that guess \citep{guo2017calibration}.

Our technique fundamentally requires making predictions about a \emph{pair} of outcomes $(y_1, y_2)$. In particular, it is not enough to make separately-calibrated predictions $\phatym(y_1|x)$ and $\phatyc(y_2|y_1, x)$, since the procedure works by comparing how much more likely any given outcome $y$ would be to occur a second time.
In principle, however, we could still transform a multi-class classification problem into a set of binary classification problems, then apply our technique to the binary problems. In this setting, instead of predicting a full joint $\phatyy(Y_1, Y_2 | x)$, we could introduce binary outcome variables $O^y_i$ such that $O^y_i = 1$ whenever $Y_i = y$, then make a set of individually-calibrated pair predictions $p(O^y_1,O^y_2 | x)$, one for each $y$. This could then be used to construct second-order marginally-calibrated versions of classwise calibration or top-label calibration. Note that this approach would still allow you to estimate the epistemic variance for any particular class, but would not tell you a full covariance matrix.

\paragraph{How much does the choice of $X$ matter?} Our work has assumed the existence of a joint distribution of variables $X, Y$ given by a conditional $\pgt(Y|X)$ and a distribution of queries $P(X)$. However, a first-order calibrated model is free to condition on an arbitrary \emph{function} of $X$ instead of $X$ itself. This means that, in a standard machine learning setup, there may be some metaphysical ambiguity about what  $X$ ``really'' refers to.

A concrete example of this is our ``Frozen Lake'' experiments, where we randomly sample an environment $X$, then occasionally hide information about the unsafe patch to obtain $X_\textsc{Partial}$, and finally train a model to map $X_\textsc{Partial}$ to a distribution over expert trajectories $Y$. If all we care about is first-order calibration, we could think of the procedure that transforms $X$ into $X_\textsc{Partial}$ as either being part of the model's grouping function $\Phi$ or as being part of the ambient probability space. Similarly, we could either think of this as learning to approximate $p(Y | X)$ with a misspecified model, or as learning to approximate $p(Y|X_\textsc{Partial})$ directly. These two are in a sense equivalent, since they produce the same samples of $X_\textsc{Partial}$ and would use the same cross entropy loss over $Y$, and a perhaps more standard choice would be to think of the hidden information as being some other variable $Z$, and think of the ``true conditional'' the model is learning as being $p(Y|X_\textsc{Partial})$. You could make a similar argument for ordinary classifiers also, e.g. are feature normalization or augmentation strategies part of the data distribution or are they part of the model?

Once we involve second-order calibration, however, this distinction becomes practical rather than metaphysical: the query $X$ is whatever information makes the two responses $Y_1, Y_2$ independent and identically distributed. In other words, if we construct a process that samples two responses $Y_1, Y_2$ that are i.i.d. given $Z$, the true conditional we are estimating will then be the conditional $p(Y | Z)$ regardless of what transformation we apply to $Z$ before giving it to our model.

We believe this is a powerful strength of our approach, because it allows you to specify the ``boundaries'' of your desired conditional distribution by example rather than by assumption. If you wish to imitate a set of experts, you can collect a dataset by asking those experts, and any ``common knowledge'' that those experts have will become ``part of $X$'', regardless of whether or not you can encode it as part of the input to the model itself; a pair-predictor model will thus be incentivised to estimate whether or not it also knows that common knowledge. Similarly, anything that is independent between those experts will be treated as part of the aleatoric uncertainty in $Y$, since it cannot be used to help predict the answer of a different expert.

\paragraph{How tight is \cref{thm:conf_hallucination} (hallucination rate)?} Our results in \cref{thm:conf_hallucination} provide an upper bound on the rate of statistical hallucinations when using a sufficiently well-behaved decoding algorithm. A natural question is whether this bound is tight, and in what circumstances.

The bound in \cref{thm:conf_hallucination} will be tight if there are exactly two values for $\pgt$ conditioned on what the model ``knows'': zero, and some nonzero constant value. In this case, all of the variance in $\pgt$ conditioned on $\Phi(X)$ is caused by these two point masses, and the ratio of probabilities in $\Ccheat$ will tell you the fraction of inputs $X$ for which $\pgt$ takes the nonzero value. This might be the case if the model is very confident about the probability of the particular response $y$ assuming it is correct, but does not know whether or not $y$ is correct. Our experiments in the digits-of-pi task (discussed in \cref{appendix:sec:digits-of-pi}) approximately satisfy this property, since the only thing that changes between digits is the set of statements that are correct; the probability of any given statement is consistent across all queries for which it is correct.

In more realistic scenarios, there may be other aspects that influence the probability of a given response other than its correctness, e.g. the model may not know something about the typical style of answers to a particular type of question. This will lead to increased variance in $\pgt$ and a lower confidence.
The epistemic confidence may still be useful in those settings as a normalized measurement of uncertainty in general, but it will likely produce a conservative overestimate of the chance of hallucination in particular.

We also note that the bound in \cref{thm:conf_hallucination} is particularly simple because it attempts to bound the rate of generating statements whose ground truth probability was exactly zero. However, this bound is a special case of Cantelli's inequality \citep{cantelli1929sui}, a more general upper bound on a random variable given its mean and variance. We demonstrate how to use this to construct other one-sided bounds in \cref{app:sec:theory_error_bounds}.

\paragraph{Partial observability and misspecification for decision making:}
The general problem of decision making under uncertainty is a well-studied problem, with much analysis under the formalism of partially-observable Markov decision processes (POMDPs) \citep{kaelbling1998planning}. Agents acting in POMDPs must perform inference about their unknown state, based on a limited view of the environment.

Of particular relevance to our work is \emph{asymmetric} imitation learning: the problem of learning to correctly imitate expert demonstrations when the experts may have access to additional information not known to the imitation agent. Naive imitation can cause an agent to take unsafe or undesirable actions, while ``deluding itself'' into expecting that every action it takes will be safe; \citet{ortega2021shaking} demonstrate this problem and identify it as an instance of \emph{confounding} in a causal graph. This problem can be avoided if all training data is collected under the imitation-learning policy, where the expert actions are queried but only the imitation-learner's action are used. Relatedly, \citet{Warrington2020RobustAL} describe a procedure for modifying the \emph{expert} policy so that it can be safely imitated. Unfortunately, these procedures require the ability to dynamically query or adjust the expert policy, which is not always possible.

In the ``Frozen Lake'' experiment, we used the same hidden location when drawing the two expert decisions, so that making calibrated predictions about pairs of expert trajectories would requires us to quantify the influence of that extra information. As discussed above, this is essentially folding the partial observability of the ``Frozen Lake'' experiments into the grouping function $\Phi(X)$.

We think this is an interesting perspective which may be useful for thinking about the behavior of misspecified agents more broadly: training a calibrated predictor is roughly the same as having an optimal predictor that only sees some restricted view of its input, so perhaps techniques that work under partial observability could also be extended to work for arbitrary calibrated models.

\paragraph{Conditional independence requirements in Pair prediction v.s. randomized causal effect estimation:}
Our technique fundamentally assumes that $Y_1$ and $Y_2$ are independent and identically distributed according to $p(Y|X)$ for each $X$. A straightforward way to ensure this holds is to sample $Y_1$ and $Y_2$ from an explicit process for generating $Y$ from $X$, e.g. by querying a random human annotator for each. Unfortunately, if direct access to an explicit response process is not available, our technique may not be directly applicable unless conditional independence is satisfied in some other way.

This use of an explicit label process in some ways resembles the use of treatment assignment in randomized controlled trials, where treatments are explicitly chosen by a randomized algorithm to ensure that treatments are conditionally independent of the outcomes \citep{ding2023first}. In the case of causal inference, this randomness allows estimating average treatment effects without making assumptions about the causal mechanism that induces those effects. In the case of our pair-prediction technique, the randomness of the label process allows us to distinguish aleatoric uncertainty from underfitting without making assumptions about the form of the distribution $p(Y|X)$.

We note also that a variety of techniques have been proposed for estimating causal effects \emph{without} control of the treatment-assignment process, usually by making assumptions about the causal structure of the naturally-occuring data; studies that estimate causal effects in this way are referred to as ``observational'' studies \citep{ding2023first}. Such techniques can be effective if correctly designed, but can produce incorrect estimates if the causal model is misspecified (due to not accounting for all confounders). Similarly, methods such as Bayesian inference can produce good uncertainty estimates without using paired $Y$ data if they are well specified, but can fail if misspecified.

\paragraph{Handling uncertainty using privileged information:}
\citet{collier2022transfer} propose a technique (TRAM) for improving robustness to label noise by training on \emph{privileged information}. At training time, they allow the later layers of a network to condition on information such as annotator IDs, which can help explain away label noise. At inference time, this privileged information can then be marginalized out.

Similar to our method, TRAM involves collecting additional data from the response process $p(Y|X)$ at training time, but does not require this additional information when scoring new inputs. However, the additional information in TRAM can be seen as ``explaining away'' the \emph{aleatoric uncertainty} in the process, allowing the model to focus on learning the link between $X$ and $Y$; this aleatoric variation is then added back in through marginalization. In contrast, our technique conditions on a separate sample $y_1$ when predicting $y_2$, which can roughly be seen as ``explaining away'' the \emph{epistemic uncertainty}. The remaining noise in $\hat{p}_\theta(y_2 | y_1, x)$ is likely aleatoric, so we can correct for it by dividing it out. 

\paragraph{Calibration, forecasting, and game-theoretic probability:}
In machine learning, calibration is usually formulated and evaluated with respect to an i.i.d. distribution of inputs $X$ and outcomes $Y$. However, much of the initial work on calibration focused instead on \emph{sequential forecasting} \citep{dawid1984present}, where the inputs $X$ arrive sequentially and may not be identically distributed, and the goal is to produce a sequence of forecasts that are calibrated in the long run (e.g. asymtotically) and also achieve good performance according to a scoring rule. Although our contributions are focused on the i.i.d. setting, and paired responses seem difficult to extend to the sequential-forecasting setting, we briefly review some of the results on calibration for sequential forecasts for the interested reader.

\citet{dawid1982well} proved that a coherent Bayesian reasoner must assign probability 1 to being eventually well-calibrated on any sequence of outcomes. This is roughly because a coherent Bayesian must be certain about their own prior (over the set of possible sequences); observed miscalibration can sometimes provide evidence about the sequence but can never convince the Bayesian to change their inference algorithm.
\citet{dawid1985calibration} expanded the notion of calibration to range over all computable subsequences (akin to the definition of multicalibration in the i.i.d. setting \citep{HbertJohnson2018MulticalibrationCF}), and showed that any computable forecasting strategies that achieve this stronger notion of calibration must eventually agree with each other.

Unfortunately, \citet{Oakes1985SelfCalibratingPD} showed that no deterministic algorithm can be calibrated on every sequence: given any deterministic forecasting strategy, there exists an adversarial distribution of sequences for which it is miscalibrated. Interestingly, \citet{foster1998asymptotic} proved that a (non-Bayesian) forecaster can achieve asymptotic calibration on every sequence if they are allowed to add noise to their forecasts independently of the adversarially-selected outcomes in the sequence, but this comes at the cost of higher prediction error on each sequence due to the added noise. \citet{Sandroni2003CalibrationWM} strengthened this result, showing that there exist (randomized) computable forecasting strategies that are asymptotically calibrated over all computable subsets of any sequence.

The above works show that calibration can be formalized in both probabilistic and game-theoretic terms. Game theory can also be used as a foundation for probability theory and hypothesis testing \citep{shafer2019game,waudby2020estimating}, and approaches based on betting can even be used to define coherent ``probabilities'' over logical implications \citep{garrabrant2016logical}. A promising property of this kind of formalization is that it can naturally account for computational constraints, by restricting the computational capabilities of the reasoner or adversary; this is much more difficult to do from a purely Bayesian perspective.

We note that, although it is difficult to define a coherent probability system over logical statements that converges to the truth, it is fairly easy to produce nearly-calibrated predictions about the truth values of a fixed distribution of logical statements, as long as you are OK with taking a non-Bayesian perspective and having a large grouping loss: you can simply output the fraction of all statements that are true, optionally after partitioning the space of statements into groups. Our experiments with predicting digits of $\pi$ are closer to this simple procedure than they are to the algorithm of \citet{garrabrant2016logical}, and we conjecture that observed logical reasoning errors in language models can be thought of as more complex versions of this simple procedure as well.

\newpage
\section{Details about and proofs of theoretical results}\label{app:theory}
In this section, we prove our theoretical results and discuss their implications.

\subsection{First-Order Calibration}

We first prove that our definition of calibration is equivalent to the more specific definition used in previous work \citep{kumar2019verified,vaicenavicius2019evaluating,perez2022beyond}. This result was previously shown by \citet{gupta2020distribution}.

\RestatePropCalibCoarse*
\begin{proof}
Fix $\Phi$ and suppose $\phaty(Y{=}y | X{=}x) = \E\big[ p(Y{=}y | X) ~\big|~ \Phi(X)=\Phi(x)\big]$. Then
\begin{align*}
\PhiY(x) &= \Big[~ 
\phaty(Y{=}y_1 | X{=}x),
~\dots,~
\phaty(Y{=}y_{|\cY|} | X{=}x),
~\Big]
\\&= \Big[~ 
\E\big[ p(Y{=}y_1 | X) ~\big|~ \Phi(X)=\Phi(x)\big],
~\dots,~
\E\big[ p(Y{=}y_{|\cY|} | X) ~\big|~ \Phi(X)=\Phi(x)\big]
~\Big]
\end{align*}
Let $h(\phi)$ be the vector
\begin{align*}
h(\phi) &= \Big[~ \E\big[ p(Y{=}y | X) ~\big|~ \Phi(X)=\phi\big] ~\Big]_{y \in \cY}
\\&= \Big[~ 
\E\big[ p(Y{=}y_1 | X) ~\big|~ \Phi(X)=\phi\big],
~\dots,~
\E\big[ p(Y{=}y_{|\cY|} | X) ~\big|~ \Phi(X)=\phi\big]
~\Big]
\end{align*}
and observe then that $\PhiY(x) = h(\Phi(x))$. It follows that, for any $x \in \cX$ and $y_i \in \cY$,
\begin{align*}
\E\big[ p(Y{=}y_i | X) ~\big|~ \PhiY(X)=\PhiY(x)\big]
&= \E\big[ p(Y{=}y_i | X) ~\big|~h(\Phi(X))=h(\Phi(x))\big]
\\&=  \E\Big[ \E\big[ p(Y{=}y_i | X) ~\big|~\Phi(X)\big] ~\Big|~h(\Phi(X))=h(\Phi(x))\Big]
\\&= \E\Big[ h(\Phi(X))_i ~\Big|~h(\Phi(X))=h(\Phi(x))\Big]
\\&= h(\Phi(x))_i
\\&= \E\big[ p(Y{=}y_i | X) ~\big|~ \Phi(X)=\Phi(x)\big]
\\&= \phaty(Y{=}y_i | X{=}x).
\end{align*}
In words, conditioning on the output of a calibrated model $\phaty$ instead of on a more refined grouping $\Phi(X)$ only combines equivalence classes $\phi$ that have the same conditional expected value of $p(Y | X)$, so the overall expected value doesn't change in the larger equivalence classes.
\end{proof}

\subsection{Equivalence of Pair Calibration and Second-Order Calibration}

We now prove our main result \cref{thm:cheat_equivalence}, which we restate below:
\RestateThmCheatEquivalence*
\begin{proof}
Consider the mapping $f$ defined by $f(\phatyy) = (\phatym, \sigmacheat)$. We will show that $f$ is a bijection between the set of first-order-calibrated $\phatyy$ and the set of second-order-calibrated $(\phatprimey, \sigmahatprime)$.

Recall that we can decompose the (conditional) covariance into a difference of expectations:
\begin{align*}
\Cov\Big[ \pgt(y|X), \pgt(y'|X) \Big]
&= \E\Big[ \big(\pgt(y|X) - \E[\pgt(y|X)]\big)\big(\pgt(y'|X) - \E[\pgt(y'|X)]\big) \Big]
\\&= \E\Big[ \pgt(y|X) \pgt(y'|X) \Big] - \E\Big[ \pgt(y|X) \Big]\E\Big[ \pgt(y'|X) \Big]
\\&= \E\Big[ P(Y_1 = y|X) P(Y_2 = y'|X) \Big] - \E\Big[ P(Y_1 = y|X) \Big]\E\Big[ P(Y_2 = y'|X) \Big]
\\&= \E\Big[ P(Y_1 = y, Y_2 = y'|X) \Big] - \E\Big[ P(Y_1 = y|X) \Big]\E\Big[ P(Y_2 = y'|X) \Big].
\end{align*}
This also holds when conditioned on a specific equivalence class $X \in [x]_{\Phi}$.

First suppose $\phatyy(y_1, y_2|x)$ is a first-order-calibrated predictor of pairs with grouping function $\Phi$, i.e.
\begin{align*}
\phatyy(y_1, y_2|x) &= \E\big[ P(Y_1=y_1, Y_2=y_2 | X) \bigmid X \in [x]_{\Phi} \big].
\end{align*}
Marginalizing out $y_2$ gives
\begin{align*}
\phatym(y_1|x) &= \sum_{y_2} \E\big[ P(Y_1=y_1, Y_2=y_2 | X) \bigmid X \in [x]_{\Phi} \big]
\\ &= \E\big[ \sum_{y_2} P(Y_1=y_1, Y_2=y_2 | X) \bigmid X \in [x]_{\Phi} \big]
\\ &= \E\big[  P(Y_1=y_1 | X) \bigmid X \in [x]_{\Phi} \big]
\\ &= \E\big[  \pgt(y_1 | X) \bigmid X \in [x]_{\Phi} \big]
\end{align*}
so $\phatym$ is first-order calibrated at predicting $Y$. The same is true for $\phatymtwo$. We then also have
\begin{align*}
\sigmacheat(x)_{y, y'} &= \phatyy(y, y' | x) - \phatym(y|x)\,\phatymtwo(y'|x)
\\&= \E\Big[ P(Y_1=y, Y_2=y' | X) \Bigmid X \in [x]_{\Phi} \Big]
\\&\qquad - \E\Big[ P(Y_1 = y|X) \Bigmid X \in [x]_{\Phi} \Big]\E\Big[ P(Y_2 = y'|X) \Bigmid X \in [x]_{\Phi} \Big]
\\&= \Cov\Big[ \pgt(y|X), \pgt(y'|X)  \Bigmid X \in [x]_{\Phi} \Big].
\end{align*}
Thus $(\phatym, \sigmacheat)$ is second-order calibrated. We conclude that $f(\phatyy)$ is second-order calibrated whenever $\phatyy$ is first-order calibrated (with the same grouping function $\Phi$).

Now consider the mapping $g$ that maps each second-order calibrated $(\phatprimey, \sigmahatprime)$ to the $\phatyy$ given by
\[
\phatyy(y_1, y_2 | x) = \sigmahatprime(x)_{y_1, y_2} + \phatprimey(y_1|x)\,\phatprimey(y_2|x).
\]
If $(\phatprimey, \sigmahatprime)$ are second-order calibrated with respect to grouping function $\Phi$, then we can expand this as
\begin{align*}
\phatyy(y_1, y_2 | x) &= \Cov\Big[ \pgt(y_1|X), \pgt(y_2|X)  \Bigmid X \in [x]_{\Phi} \Big] + \E\big[  \pgt(y_1 | X) \bigmid X \in [x]_{\Phi} \big]\E\big[  \pgt(y_2 | X) \bigmid X \in [x]_{\Phi} \big]
\\&= \E\Big[ P(Y_1=y_1, Y_2=y_2 | X) \Bigmid X \in [x]_{\Phi} \Big].
\end{align*}
This implies that $\phatyy$ is a first-order-calibrated predictor of pairs, and thus that $g(\phatprimey, \sigmahatprime)$ is first-order calibrated whenever $(\phatprimey, \sigmahatprime)$ are second-order calibrated (with the same grouping function $\Phi$).

Finally, to show that $g$ is the inverse of $f$, let $\phatyy$ be an arbitrary first-order calibrated predictor with grouping function $\Phi$, and observe that
\begin{align*}
\Big[g(f(\phatyy))\Big](y, y' | x)
&= \sigmacheat(x)_{y, y'} + \phatym(y|x)\phatym(y'|x)
\\&= \Cov\Big[ \pgt(y|X), \pgt(y'|X)  \Bigmid X \in [x]_{\Phi} \Big]
\\&\qquad\qquad + \E\big[  \pgt(y | X) \bigmid X \in [x]_{\Phi} \big]\E\big[  \pgt(y' | X) \bigmid X \in [x]_{\Phi} \big]
\\&= \E\Big[ P(Y_1=y, Y_2=y' | X) \Bigmid X \in [x]_{\Phi} \Big]
\\&= \phatyy(y, y' | x).
\end{align*}
Thus $g(f(\phatyy)) = \phatyy$, so $f$ is a bijection with inverse $f^{-1} = g$.
\end{proof}

Since predictors can always reduce their expected loss (under a proper scoring rule) by becoming better calibrated on their training task, \cref{thm:cheat_equivalence} implies that our pair-prediction procedure incentivizes second-order calibration.

\subsection{Proofs of Error Bounds for Calibrated Pair Predictors}\label{app:sec:theory_error_bounds}

\RestateThmCheatGrouping*
\begin{proof}
We will show that these properties hold individually for every value of $\PhiYY(X)$ and $\Ydec$ conditioned on $A$, and so they must also hold overall.

Let $\phi$ and $y$ be arbitrary, and $x$ be such that $\Phi(x) = \phi$.
Since $\phatyy$ is calibrated, it must be symmetric, so $\vcheat(y|X) = \sigmacheat(X)_{y,y}$. Furthermore $(\phatym, \sigmacheat)$ must be second-order calibrated, and $A,\Ydec$ are independent of $X$ given $\Phi(X) = \phi$, so
\begin{align*}
\phatym(y|x) &= \E\!\Big[ \pgt(y | X) \Bigmid \PhiYY(X) = \phi \Big]
= \E\!\Big[ \pgt(y | X) \Bigmid \PhiYY(X) = \phi, \Ydec=y, A \Big]
\\
\vcheat(y|x) &= \Var\!\Big[ \pgt(y | X) \Bigmid \PhiYY(X) = \phi \Big]
= \Var\!\Big[ \pgt(y | X) \Bigmid \PhiYY(X) = \phi, \Ydec=y, A \Big]
\\
\end{align*}
Let $B$ be the event where $(\PhiYY(X) = \phi, \Ydec=y, A)$ all occur.

For the first part, we have
\begin{align*}
&\E\Big[
\!\big(\phatym\!(y | X) - \pgt\!(y | X) \big)^2 \!\Bigmid\! B
\Big]
\\&\qquad= \E\left[
\!\left( \E\!\Big[ \pgt(y | X) \Bigmid \PhiYY(X) = \phi, \Ydec=y, A \Big] - \pgt\!(y | X) \right)^2 \middle| B
\right]
\\&\qquad=\Var\!\Big[ \pgt(y | X) \Bigmid \PhiYY(X) = \phi, \Ydec=y, A \Big]
\\&\qquad= \vcheat(y|x) = \E\big[ \vcheat(y|X) \big| B \big]
\end{align*}
where the last step follows because
$x$ is an arbitrary input with $\PhiYY(x) = \phi$ and every such $x$ has the same value for $\vcheat(y | x)$.
Taking expectations over all values of $X$ and $\Ydec$ given $A$ yields the desired result.

For the second part, Chebyshev's inequality ensures that
\begin{align*}
P\left[ |\E[Z] - Z| \ge \sqrt{\frac{\Var(Z)}{\beta}} \right] \le \beta
\end{align*}
for any random variable $Z$ and any $\beta$. Applying this with $Z = \pgt(y | X)$ conditioned on $B$ gives
\begin{align*}
P\left[\Big| \E[\pgt(\Ydec | X) | B] - \pgt(\Ydec | X) \Big| \ge \!\sqrt{\!\frac{\Var\!\Big[ \pgt(y | X) \Bigmid B\Big]}{\beta}}\, \!\middlemid\! B \right] \le \beta.
\end{align*}
so
\begin{align*}
P\left[\Big| \phatym(y|x) - \pgt(\Ydec | X) \Big| \ge \!\sqrt{\!\frac{\vcheat(y|x)}{\beta}}\, \!\middlemid\! B \right] \le \beta.
\end{align*}
We can now marginalize over $\PhiYY(X)$ and $\Ydec$ conditioned on $A$ to obtain the desired result.
\end{proof}

\RestatePropConfBounded*
\begin{proof}
$\Ccheat(y | x) \ge 0$ because both its numerator and denominator are nonnegative. Furthermore, if $\Ccheat(y | x) = 1$, the numerator and denominator must be equal.

To show that $\Ccheat(y | x) \le 1$, algebraic manipulation allows us to write it in the form
\begin{align*}
\Ccheat(y | x) &=
\frac{\phatym(y | x)}{\phatyc(y|y,x)}
= \frac{\phatym(y | x)^2}{\phatyy(y,y|x)}
= \left( \frac{\phatyy(y,y|x)}{\phatym(y | x)^2} \right)^{-1}
\\&= \left(1 + \frac{\phatyy(y,y|x) - \phatym(y | x)^2}{\phatym(y | x)^2} \right)^{-1}
=\left( 1 + \frac{\vcheat(y|x)}{\phatym(y|x)^2} \right)^{-1}.
\end{align*}
If $\phatyy$ is calibrated, it must be symmetric, so $\vcheat(y|X) = \sigmacheat(X)_{y,y}$, which is a conditional variance and thus cannot be negative. It follows that $\frac{\vcheat(y|x)}{\phatym(y|x)^2} \ge 0$, so $\Ccheat(y | x) \le 1$.
\end{proof}

\textbf{Note of caution:}
When $\phatyy$ is \emph{not} calibrated, it is no longer true that $\vcheat(y|X)$ is necessarily equal to $\sigmacheat(X)_{y,y}$, because $\phatyy$ is not necessarily symmetric. It is also not necessarily true that $\sigmacheat(X)_{y,y}$ is an epistemic variance, since \cref{thm:cheat_equivalence} does not hold. This can mean that, for miscalibrated $\phatyy$, it is possible to observe $\vcheat(y|x) < 0$ and $\Ccheat > 1$, and we do observe this in some of our experiments. We discuss this further in \cref{app:experiments}.
\vspace{1em}

\RestateThmConfHallucination*
\begin{proof}
Similar to our proof of \cref{thm:conf_cheat_grouping}, we can prove that this holds individually for every value of $\PhiYY(X)$ and $\Ydec$ conditioned on $A$ and then take an expectation.
As before, let $\phi$ and $y$ be arbitrary, and $x$ be such that $\Phi(x) = \phi$.
By Cantelli's inequality \citep{cantelli1929sui} (also known as the one-sided Chebyshev's inequality),
\begin{align*}
P\left[ Z  \le \E[Z] - \lambda \right] \le \frac{\Var(Z)}{\Var(Z) + \lambda^2}
\end{align*}
for any random variable $Z$ and any $\beta$. Substituting $\lambda = \E[Z]$,
\begin{align*}
P\left[ Z  \le 0 \right] \le \frac{\Var(Z)}{\Var(Z) + \E[Z]^2} = \frac{\E[Z^2] - \E[Z]^2}{\E[Z^2]} = 1 - \frac{\E[Z]^2}{\E[Z^2]}
\end{align*}
Now letting $Z = \pgt(y | X)$ and conditioning on $B = (\PhiYY(X) = \phi, \Ydec=y, A)$ as before, and using the fact that $\phatyy$ is calibrated,
\begin{align*}
P\big[ \pgt(y | X)  \le 0 \bigmid B \big]
&\le 1 - \frac{\E[\pgt(y | X) | B]^2}{\E[\pgt(y | X)^2 | B]}
= 1 - \frac{\E[\pgt(y | X) | \Phi(X) = \phi]^2}{\E[\pgt(y | X)^2 | \Phi(X) = \phi]}
= 1 - \frac{\phatym(y | x)^2}{\phatyy(y, y | x)}
\\&= 1 - \frac{\phatym(y | x)}{\phatyc(y| y, x)}
= 1 - \Ccheat(y | x)
= 1 - \E[\Ccheat(y | X) | B]
\end{align*}
where here $x$ is an arbitrary input with $\PhiYY(x) = \phi$, since every such $x$ has the same value for $\Ccheat(y | x)$. Taking expectations of both sides over all values for $\phi$ and $y$ completes the proof.
\end{proof}

As an aside, we note that it's possible to prove a one-sided bound on $\pgt(y | X)$ by combining the proof ideas from \cref{thm:conf_cheat_grouping} and \cref{thm:conf_hallucination}:
\begin{proposition}\label{app:thm:cantelli_bound}
Suppose $\phatyy$ is calibrated.
Let $A$ be any event and $\Ydec \in \cY$ be any (possibly random) output such that $\Ydec, A \indep X \mid \PhiYY(X)$. Then for any $\beta \in (0,1)$,
\begin{align*}
P\left[
\pgt(y | X)  \le \phatym(y|X) - \sqrt{\vcheat(y|x) \left(\frac{1}{\beta} - 1\right)}
\middlemid A \right] \le \beta.
\end{align*}
\end{proposition}
\begin{proof}
As above, but let $\lambda = \sqrt{\vcheat(y|x) \left(\frac{1}{\beta} - 1\right)}$ in Cantelli's inequality. Then substituting $\vcheat$ as before we obtain
\begin{align*}
P\big[ \pgt(y | X)  \le \phatym(y|X) - \lambda \bigmid B \big]
\le \frac{\vcheat(y|x)}{\vcheat(y|x) + \vcheat(y|x) \left(\frac{1}{\beta} - 1\right) }
= \beta,
\end{align*}
and taking expectations completes the proof.
\end{proof}
This is a better bound than the one in \cref{thm:conf_cheat_grouping} in the case where we want a conservative estimate of how small $\pgt(y | X)$ could be with confidence $1 - \beta$, instead of wanting to bound the distance to $\phatym(y|X)$.

\subsection{Distribution-Free Bounds on $p(Y|X)$}\label{appendix:distnfree_theory}

Suppose $\cY = \{0, 1\}$.
Our distribution-free high-probability bound on $p(Y|X)$ is based on the random variable
\[
D_{\varepsilon} = \frac{\big( \pgt(1 | X) - \phaty(1 | X) \big)^2}{\max \{ \vhat(1 | X), \varepsilon\}}.
\]
$D_{\varepsilon}$ is always nonnegative, and if $\phaty$ and $\vhat$ are second-order well-calibrated (or, more precisely, if $\vhat$ is the diagonal of the epistemic covariance matrix) the expected value $\E[D_\varepsilon]$ should be at most 1.

The following lemma shows that we can use $\E[D_\varepsilon]$ to bound $p(Y|X)$, even if $\phaty$ is not well calibrated:

\begin{lemma}\label{lemma:threshold_bound}
Suppose $\E[D_\varepsilon] \le \gamma_\varepsilon$. Then for a randomly sampled input $X \sim p(X)$ and any $\beta \in [0, 1)$, the true conditional $p(Y=1 | X)$ lies within
\begin{align}
&\phaty(1 | X) \pm \sqrt{\max\{\vhat(Y=1 | X), \varepsilon\}\,\gamma_\varepsilon  / \beta}\label{eqn:distfree_bound}
\end{align}
with probability at least $1 - \beta$.
\end{lemma}
\begin{proof}
By Markov's inequality, for any $\beta \in [0, 1)$,
$p\Big( D_\varepsilon \ge \E[D_\varepsilon]/\beta \Big) \le \beta.$
In other words, with probability at least $1 - \beta$, $D_\varepsilon < \E[D_\varepsilon]/\beta$. Since $\E[D_\varepsilon] \le \gamma_\varepsilon$, this means that
\begin{align*}
\frac{\big( \pgt(1 | X) - \phaty(1 | X) \big)^2}{\max \{ \vhat(1 | X), \varepsilon\}} <  \gamma_\varepsilon / \beta
\end{align*}
with probability at least $1 - \beta$, in which case
\[
\big| \pgt(1 | X) - \phaty(1 | X) \big| < \sqrt{\max \{ \vhat(1 | X), \varepsilon\} \gamma_\varepsilon / \beta}.\qedhere
\]
\end{proof}

We can use this to prove \cref{thm:distnfreebound}, which we restate below:
\RestateThmDistnFreeBound*
\begin{proof}
It remains to show that the $\gamma_\varepsilon$ produced by \cref{alg:distfree_bound} is an upper bound on $\E[D_\varepsilon]$, at which point we can apply \cref{lemma:threshold_bound}.

Define the random variable
\[
S_\varepsilon = \frac{\big( Y_1 - \phaty(1 | X) \big)\big( Y_2 - \phaty(1 | X) \big)}{\max \{ \vhat(1 | X), \varepsilon\}}
\]
and observe that
\begin{align*}
\E[S_\varepsilon]
&= \E\left[
\frac{\big( Y_1 - \phaty(1 | X) \big)\big( Y_2 - \phaty(1 | X) \big)}{\max \{ \vhat(1 | X), \varepsilon\}}
\right]
\\&= \E\left[
\frac{\E\big[\big( Y_1 - \phaty(1 | X) \big)\big( Y_2 - \phaty(1 | X) \big) ~\big|~ X \big]}{\max \{ \vhat(1 | X), \varepsilon\}}
\right]
\\&= \E\left[
\frac{\big( \E[Y_1|X] - \phaty(1 | X) \big)\big( \E[Y_2|X] - \phaty(1 | X) \big)}{\max \{ \vhat(1 | X), \varepsilon\}}
\right]
\\&= \E\left[
\frac{\big( p(Y=1|X) - \phaty(1 | X) \big)^2}{\max \{ \vhat(1 | X), \varepsilon\}}
\right]
= \E[D_\varepsilon].
\end{align*}
This means that any confidence interval for $\E[S_\varepsilon]$ is also a confidence interval for $\E[D_\varepsilon]$.
Furthermore, we know that $-\frac{1}{\varepsilon} \le S_\varepsilon \le \frac{1}{\varepsilon}$, and we can construct samples of $S_\varepsilon$ by using $\phaty$ and samples $(X, Y_1, Y_2)$, as described in \cref{alg:distfree_bound}.

By assumption, the subroutine $\textsc{MeanConfItvl}$ constructs a confidence interval for the mean of a bounded random variable. In other words, it satisfies the property that, for any bounded i.i.d. random variables $V^{(i)} \in (a, b)$, if we let $(L, U) = \textsc{MeanConfItvl}\big( \{ V^{(1)}, \dots, V^{(N)}\}, a, b, \alpha \big)$, then $L \le \E[V] \le U$ with probability at least $\alpha$. 

\cref{alg:distfree_bound} then applies $\textsc{MeanConfItvl}$ to the samples of $S_\varepsilon^{(i)}$, which are each bounded between $-1/\varepsilon$ and $1/\varepsilon$. As such, we know that the returned value $\gamma_\varepsilon^+$ will satisfy  $\E[S_\varepsilon] \le \gamma_\varepsilon^+$ with probability at least $(1-\alpha)$. This confidence interval must also be a bound on $\E[D_\varepsilon]$, so we can apply \cref{lemma:threshold_bound}, which completes the proof.
\end{proof}

There are multiple possible implementations of the subroutine $\textsc{MeanConfItvl}$, based on different confidence intervals for the mean of a bounded random variable. A particularly simple implementation is based on Hoeffding's inequality \citep{Hoeffding1963ProbabilityIF}:
\begin{proposition}[Confidence interval via Hoeffding's inequality]
In \cref{alg:distfree_bound}, an implementation of subroutine
$\textsc{MeanConfItvl}\big(\{s_\varepsilon^{(i)}\}_{i=1}^N, -\frac{1}{\varepsilon}, \frac{1}{\varepsilon}, \alpha\big)$
that returns the values
\[
\gamma_\varepsilon^- = -\frac{1}{\varepsilon},
\qquad
\gamma_\varepsilon^+ = \frac{1}{N} \sum_{i=1}^N s_\varepsilon^{(i)} + \sqrt{2 \frac{-\log \alpha}{n \varepsilon^2} }
\]
guarantees that $\gamma_\varepsilon^- \le \E[S_\varepsilon] \le \gamma_\varepsilon^+$ with probability at least $1 - \alpha$ (and thus that \cref{thm:distnfreebound} holds).
\end{proposition}
\begin{proof}
For the lower bound, we know that $\gamma_\varepsilon^- = -\frac{1}{\varepsilon} \le \E[S_\varepsilon]$ due to the boundedness of $S_\varepsilon$. For the upper bound,
Hoeffding's inequality gives us
\[
p\left( \E[S_\varepsilon] - \frac{1}{N} \sum_{i=1}^N s_\varepsilon^{(i)} \ge t \right) \le \exp\left(-\frac{nt^2\varepsilon^2}{2}\right).
\]
Choosing $t = \sqrt{2 \frac{-\log \alpha}{n \varepsilon^2} }$, this becomes
\[
p\left( \E[S_\varepsilon] - \frac{1}{N} \sum_{i=1}^N s_\varepsilon^{(i)} \ge t \right) \le \alpha,
\]
so we must have
\[
\E[S_\varepsilon] \le  \frac{1}{N} \sum_{i=1}^N s_\varepsilon^{(i)} + t
= \frac{1}{N} \sum_{i=1}^N s_\varepsilon^{(i)} + \sqrt{2 \frac{-\log \alpha}{n \varepsilon^2} } = \gamma_\varepsilon^+
\]
with probability at least $1 - \alpha$.
\end{proof}

There are also more complex algorithms which may require fewer samples to give an accurate upper bound. For instance,  \citet{waudby2020estimating} describe an algorithm for constructing tighter confidence intervals based on ``betting strategies''. This algorithm is available in the \texttt{confseq} Python package\footnote{\url{https://github.com/gostevehoward/confseq}}.
See \cref{appendix:distnfree_experiment} for additional experiments studying the convergence of our bound in practice, and comparing the bounds constructed using Hoeffding's inequality to bounds using \texttt{confseq}.

\subsubsection{Is \cref{thm:distnfreebound} the best we can do without distributional assumptions?}\label{appendix:distnfree_best_we_can_do}

\cref{thm:distnfreebound} does not require that the model is perfectly calibrated.
If $\phaty$ and $\vhat$ are actually epistemically perfectly calibrated, and we take $\varepsilon \to 0$, we will have $\E[S_\varepsilon] = \E[D_\varepsilon] \to 1$, so in principle we can make \cref{eqn:distfree_bound} arbitrarily close to \cref{thm:conf_cheat_grouping} by choosing a small enough $\varepsilon$ and a large enough calibration set. (This assumes that the confidence interval for $\E[S_\varepsilon]$ will converge asymptotically to the true value of $\E[S_\varepsilon]$, which is true for both Hoeffding's inequality and the betting-based algorithms in \texttt{confseq}). 

Even so, the guarantee provided by \cref{thm:distnfreebound} is somewhat weaker than that of \cref{thm:conf_bounded}, because the $1-\beta$ chance only holds for random $X \sim p(X)$ and may not hold after conditioning on additional information (e.g. the event $A$, which can be any function of the output of the model). To give some intuition of why this occurs, suppose we perturb a calibrated model with a tiny amount of per-input noise, e.g. $\phaty(y|X) = p(y|\Phi(X)) + \eta(X)$. Even if $\eta(X)$ is very small, conditioning on $\phaty(y|X)$ may then be enough to identify $X$ itself, and if there is a single such $X$ that is outside of the range given by \cref{thm:conf_bounded}, the stronger statement will no longer hold. One way to circumvent this in principle would be to explicitly bin the outputs of a model to a finite number of outputs, similar to the method proposed by \citet{kumar2019verified}; it would then be possible to construct a separate bound for each bin. Effectively, this would mean that we enumerate all of the events $A$ that we care about in advance, and then apply \cref{thm:distnfreebound} separately to each subset of the dataset. (Note that neither bound holds conditioned on $X$ itself, because once $X$ is observed then either $p(Y|X)$ is in the interval or it is not, so the conditional probability is either 0 or 1, not $1 - \beta$. This is why the event in \cref{thm:conf_bounded} must be conditionally independent of $X$ given the output of the model.)

An interesting point of comparison is Theorem 1 of \citet{barber2020distribution}, which states that any distribution-free $(1-\alpha)$-confidence interval for the probability $p(Y=1|X)$ must also be a $(1-\alpha)$-confidence interval for any random variable $Z \in [0, 1]$ for which $\E[Z | X] = p(Y|X)$, as long as the interval was constructed using only one sample $Y \sim p(Y|X)$ for each $X$ (and as long as $Z$ is conditionally independent of the interval-construction procedure given $X$). In particular, if we choose $Z = Y$, this means that any distribution-free $(1-\alpha)$-confidence interval must contain $0$ or $1$ with probability at least $(1-\alpha)$, and so the interval cannot precisely identify $p(Y|X)$ if $p(Y|X)$ is bounded away from 0 and 1.

We can roughly interpret \citeauthor{barber2020distribution}'s theorem as stating that  distribution-free confidence intervals constructed using \emph{one} $Y$ for each $X$ can only effectively estimate the \emph{first} moment of $p(Y|X)$, and must be wide enough to contain the worst-case variable $Z$ with the correct expected value. Our \cref{thm:distnfreebound}, on the other hand, uses \emph{two} $Y$s for each $X$, and converges to a bound based on the first \emph{two} moments (mean and variance). Moreover, \cref{thm:cheat_equivalence} suggests that two samples may in fact be necessary to estimate the second moment in this manner.
We conjecture that this is a general constraint for distribution-free confidence intervals based on samples: if we are allowed to use $k$ samples $Y_1, \dots, Y_k \sim p(Y|X)$ for each $X$, it seems likely that only the first $k$ moments can be identified in a distribution-free way, and thus that our confidence intervals must be wide enough to contain any random variable $Z$ with the same first $k$ moments as the true probability $p(Y|X)$ conditioned on the output of our model or algorithm. If true, this would suggest that we can't do much better than \cref{thm:distnfreebound} with only two samples of $Y$ for each $X$, unless we are willing to make distributional assumptions about the form of $p(Y|X)$, but we might be able to do better with more than two samples.

We also note that if you want a one-sided bound instead of a two-sided bound, it is possible to do better than Chebyshev's inequality by instead using Cantelli's inequality \citep{cantelli1929sui}. This inequality was used to prove \cref{thm:conf_hallucination,app:thm:cantelli_bound}, and it could likely be generalized to apply without assuming calibration using a similar technique to the proof of \cref{thm:distnfreebound}.

\clearpage
\section{Properties of Calibrated Models of Pairs}\label{appendix:properties_calibrated_of_pairs}
In this section we derive some properties that any calibrated model $\hat{p}_\theta(Y_1, Y_2 | X)$ must satisfy, which can be useful when designing neural network architectures for pair prediction.

\begin{proposition}\label{prop:properties_of_calibrated_pair_models}
Suppose $|\cY| = K$, and order it as $\cY = \{v_1, \dots, v_K\}$. If $\hat{p}_\theta(Y_1, Y_2 | X)$ is a perfectly-calibrated predictor of outcomes $(Y_1, Y_2) \in \cY \times \cY$, then
\begin{enumerate}[(i)]
    \item $\hat{p}_\theta(y_1, y_2 | x)$ is a proper probability distribution, i.e. $\hat{p}_\theta(y_1, y_2 | x) \ge 0$ for all $x \in \cX, y_1, y_2 \in \cY$ and $\sum_{y_1, y_2} \hat{p}_\theta(y_1, y_2 | x) = 1$ for all $x \in \cX$,
    \item $\hat{p}_\theta(y_1, y_2 | x)$ is symmetric, i.e. $\hat{p}_\theta(Y_1 = y_1, Y_2 = y_2 | x) = \hat{p}_\theta(Y_1=y_2, Y_2=y_1 | x)$,
    \item The joint probability matrix $\hat{P}^{[x]} \in \R^{K \times K}$ given by
    $\hat{P}_{ij}^{[x]} = \hat{p}_\theta(Y_1 = v_i, Y_2 = v_j | x)$ is positive semidefinite for each $x \in \cX$.
\end{enumerate}
\end{proposition}
\begin{proof}
Since $\hat{p}_\theta$ is perfectly calibrated, there exists a grouping function $\Phi$ such that
\begin{align*}
\hat{p}_\theta(Y_1=y_1, Y_2=y_2 | X=x)
&= \E\big[ p(Y_1=y_1, Y_2=y_2 | X) ~\big|~ \Phi(X)=\Phi(x)\big]
\\&= \E\big[ p(Y=y_1 | X)p(Y=y_2 | X) ~\big|~ \Phi(X)=\Phi(x)\big],
\end{align*}
which implies properties (i) and (ii).

If we let $\bm{p}^{[x]} \in \R^K$ be the vector such that $\bm{p}^{[x]}_k = p(Y=v_k|X=x)$, we can write this in matrix form as
\[
\hat{P}^{[x]} = \E\Big[ \bm{p}^{[x]} \big(\bm{p}^{[x]}\big)^T ~\Big|~ \Phi(X)=\Phi(x)\Big].
\]
For any $\bm{v} \in \R^K$, we must then have
\[
\bm{v}^T\hat{P}^{[x]}\bm{v}
= \E\Big[ \bm{v}^T\bm{p}^{[x]} \big(\bm{p}^{[x]}\big)^T\bm{v} ~\Big|~ \Phi(X)=\Phi(x)\Big]
= \E\Big[ \big(\bm{v}^T\bm{p}^{[x]}\big)^2 ~\Big|~ \Phi(X)=\Phi(x)\Big] \ge 0,
\]
so $\hat{P}^{[x]}$ is positive semidefinite.
\end{proof}

We note that a nonnegative matrix satisfying (ii) and (iii) is known as a ``doubly nonnegative'' matrix, and if $p(Y|X)$ takes only finitely many values the matrix $P^{[x]}$ will also be ``completely positive'' (i.e. factorizable as $P^{[x]} = B^T B$ where $B$ is entrywise nonnegative) \citep{berman2003completely}.

If $Y$ is a binary outcome, we can characterize the space of calibrated predictors even more precisely:
\begin{proposition}\label{prop:properties_binary}
If $\hat{p}_\theta(Y_1, Y_2 | X)$ is a perfectly-calibrated predictor of paired binary outcomes $(Y_1, Y_2) \in \{0,1\} \times \{0, 1\}$, then the matrix $\hat{P}^{[x]}$ can be written in the form

\begin{align}
\hat{P}^{[x]} = 
\rho(x) \begin{bmatrix}
1 - \mu(x) & 0 \\
0 & \mu(x)
\end{bmatrix} + \left(1 - \rho(x)\right) \begin{bmatrix}
(1 - \mu(x))^2 & \mu(x)(1 - \mu(x)) \\
\mu(x)(1 - \mu(x)) & \mu(x)^2
\end{bmatrix}
\label{eqn:binary_kappa_form}
\end{align}

for some $\mu : \cX \to [0, 1]$ and  $\rho : \cX \to [0, 1]$.
\end{proposition}
\begin{proof}
Fix a particular $x$. By \cref{prop:properties_of_calibrated_pair_models} we know we can write
\[
\hat{P}^{[x]} = 
\begin{bmatrix}
a & b \\
b & c
\end{bmatrix}
\]
for some nonnegative $a, b, c \in \R$ such that $a + 2b + c = 1$. Choose
\begin{align}
    \mu(x) &= b + c,
    &
    \rho(x) &= \frac{a c - b^2}{(a+b)(b+c)}
    = 1 - \frac{b}{\mu(1-\mu)}.
\end{align}
Substituting shows that $\hat{P}^{[x]}$ can then be expressed as \cref{eqn:binary_kappa_form}. $0 \le \mu(x) \le 1$ because $b$ and $c$ are nonnegative and sum to at most 1. $\rho(x) \le 1$ because $ac - b^2 \le ac \le (a+b)(c+d)$. Finally, since $\hat{P}^{[x]}$ must be positive semidefinite, the determinant $|\hat{P}^{[x]}| = ac - b^2$ must be nonnegative, so $\rho(x) \ge 0$.
\end{proof}

Note that $\mu$ is the predicted probability of $Y=1$, and $\rho$ is the predicted correlation between $Y_1$ and $Y_2$. (In fact, $\rho$ is exactly the Pearson correlation coefficient of $Y_1$ and $Y_2$ given $\Phi(X)$, also referred to as the ``Phi coefficient'' \citep{kotz_phicoefficient_2005}.) Given this parameterization, we can efficiently compute
\[
\hat{p}_\theta(Y=1|X=x) = \mu(x),
\qquad
v^\textsc{Cheat}_{\theta}(Y=1 | X=x) = \kappa(x) \mu(x)(1 - \mu(x)).
\]

\textbf{Neural network architectures for $\cY = \{0,1\}$:} When we know $Y$ is a binary outcome, we suggest parameterizing the output head of a pair predictor $\hat{p}_\theta(y_1, y_2 | x)$ using \cref{eqn:binary_kappa_form}. Specifically, we can parameterize our model to produce two-dimensional vectors $h_\theta: \cX \to \R^2$, then set $\phi(x) = \sigma(h_\theta(x)[0])$, $\rho(x) = \sigma(h_\theta(x)[1])$, where $\sigma$ is the logistic sigmoid function.

\textbf{Neural network architectures for enumerable $\cY = \{0,1, \dots, K\}$:}
For classification tasks, where $\cY$ is a finite (and ``reasonably-sized'') set of classes, we suggest using the properties in \cref{prop:properties_of_calibrated_pair_models} to design the architecture. In particular, we can enforce property (i) by applying the \texttt{softmax} operation across the set of $\cY \times \cY$ possible outputs, and enforce property (ii) by constraining the output layer to output a symmetric matrix of logits $\R^{\cY \times \cY}$ before applying the softmax operation.

We are not aware of a simple method for strictly enforcing property (iii) as part of the architecture while simultaneously ensuring that property (i) holds. However, empirically we observe that violations of property (iii) can lead to unreasonable negative variance estimates. We thus suggest computing the eigenvalues of the post-softmax matrix $\hat{P}^{[x]}$ and adding a regularization penalty to negative eigenvalues, e.g.
\[
\mathcal{L}_\text{regularized}(x, y_1, y_2) = - \log \hat{p}_\theta(y_1, y_2 | x) + \alpha \sum_{i=1}^K \max\{0, \lambda_i(x)\}^2
\]
where $\lambda_i(x)$ is the $i$th eigenvalue of $\hat{P}^{[x]}$. (Since this regularization  penalty only applies to negative eigenvalues, and a calibrated model should never produce negative eigenvalues, this regularization penalty should not change the optimal calibrated solution if one exists.)

\textbf{Neural network architectures for sequential or exponentially-large $\cY$:}
When $\cY$ is an exponentially large set, such as the set of all sequences, it may be intractable to enforce either condition (ii) or condition (iii) of \cref{prop:properties_of_calibrated_pair_models}. For our experiments, we settled on only enforcing property (i) by concatenating the two outputs $Y_1$ and $Y_2$ together. We found that padding them to a constant length improved performance by ensuring that $Y_1$ and $Y_2$ each have consistent positional embeddings, because othewise the positional shift in $Y_2$ can make it harder to predict $Y_2$ than $Y_1$ and thus introduce additional noise into the confidence metric.
We believe adjusting the architecture for sequence models to enforce (or encourage) it to satisfy properties (ii) and (iii) is an exciting area for future work.
\newpage
\section{Details of Experimental Results}\label{app:experiments}

\subsection{One-Dimensional Binary Regression (\cref{fig:intro_1d_problem})}\label{appendix:training_1d_regression} 

\subsubsection{Data Distribution}
We choose $p(X)$ as a standard normal random variable $\mathcal{N}(0, 1)$, and define $p(Y|X)$ as a Bernoulli distribution with
\begin{align*}
p(Y=1 | X=x) &= \frac{0.98\, u(x) + 1}{2},\\
u(x) &= 0.6 \cos(v(x)) + 0.4 \cos(4.2 x),\\
v(x) &= \text{sign}(x) \cdot (120 |x| - 112 w(|x|) - 0.0635).\\
w(z) &= 0.2 \log\Big(1 + \exp\big((z - 1.0)/.2\big)\Big)\\
\end{align*}
This function was chosen to have higher-frequency variation near $x=0$ with a lower-frequency component throughout.

We construct a dataset of 25,000 samples of $X$, each of which have two corresponding samples $Y_1, Y_2 \sim_{i.i.d.} p(Y|X)$, for a total of 50,000 $Y$s.

\subsubsection{Architectures and Training Details}

For the NN Ensemble, Evidential NN, and Cheat-corrected NN models, we use a small MLP/LayerNorm/Residual architecture inspired by the MLP blocks in a Transformer \citep{vaswani2017attention}, with the following form:

\begin{algorithm}[h]
   \caption{NN architecture for 1D Binary Regression}
   \label{appendix:alg:1d_regression_nn_arch}
\begin{algorithmic}
   \STATE {\bfseries Input:} value $x \in \R$, output dimension $d$
   \STATE \textcolor{gray}{\emph{Input layer:}}
   \STATE $\bm{v}^{(0)} := \bm{w}^{(0,a)} \odot (x \cdot \bm{1} + \bm{b}^{(0,a)})$
        \qquad\textcolor{gray}{where $\bm{w}^{(0,a)} \in \R^{512}, \bm{b}^{(0,a)}\in \R^{512}$}
   \STATE $\bm{r}^{(0)} := \bm{W}^{(0,b)} \text{relu}(\bm{v}^{(0)}) + \bm{b}^{(1,b)}$
        \qquad\textcolor{gray}{where $\bm{W}^{(0,b)} \in \R^{128 \times 512}, \bm{b}^{(0,b)}\in \R^{128}$}
   \FOR{$i=1$ {\bfseries to} $3$}
       \STATE \textcolor{gray}{\emph{Residual block:}}
       \STATE $\bm{u}^{(i)} := \text{LayerNorm}^{(i)}(\bm{v}^{(i-1)})$
            \qquad\textcolor{gray}{with learnable scale and shift \citep{ba2016layer}}
       \STATE $\bm{v}^{(i)} := \bm{W}^{(i,a)} \bm{u}^{(i)} + \bm{b}^{(i,a)}$
            \qquad\textcolor{gray}{where $\bm{W}^{(i,a)} \in \R^{512 \times 128}, \bm{b}^{(i,a)}\in \R^{512}$}
       \STATE $\bm{r}^{(i)} := \bm{r}^{(i-1)} + \bm{W}^{(i,b)} \text{relu}(\bm{v}^{(i)}) + \bm{b}^{(i,b)}$
            \qquad\textcolor{gray}{where $\bm{W}^{(i,b)} \in \R^{128 \times 512}, \bm{b}^{(i,b)}\in \R^{128}$}
   \ENDFOR
   \STATE \textcolor{gray}{\emph{Output head:}}
   \STATE $\bm{u}^{(4)} := \text{LayerNorm}^{(4)}(\bm{v}^{(3)})$
            \qquad\textcolor{gray}{with learnable scale and shift}
   \STATE $\bm{v}^{(4)} := \bm{W}^{(4,a)} \bm{u}^{(4)} + \bm{b}^{(4,a)}$
        \qquad\textcolor{gray}{where $\bm{W}^{(4,a)} \in \R^{512 \times 128}, \bm{b}^{(4,a)}\in \R^{512}$}
   \STATE $\bm{o}^{(4)} := \bm{W}^{(4,b)} \text{relu}(\bm{v}^{(4)}) + \bm{b}^{(4,b)}$
        \qquad\textcolor{gray}{where $\bm{W}^{(4,b)} \in \R^{d \times 512}, \bm{b}^{(4,b)}\in \R^{d}$}
    \STATE \textbf{Return} $\bm{o}^{(4)}$
\end{algorithmic}
\end{algorithm}

\textbf{NN Ensemble}: We randomly initialize 8 copies of the architecture with output dimension $d=1$, then train each for 10,000 training iterations with a batch size of 512, randomly selecting $(X, Y_1, Y_2)$ triples from the 25,000 training examples. For each example $(x, y_1, y_2)$ we use the loss
\[
\mathcal{L}_\textsc{NN}(x, y_1, y_2, \theta_i) = \frac{1}{2}\sum_{i=1}^2 - \log \hat{p}_{\theta_i}(Y=y_i | X = x)
\]
where $\hat{p}_{\theta_i}(Y=1|X=x) = \sigma(h_{\theta_i}(x))$, $\sigma$ is the logistic sigmoid function, and $h_{\theta_i}$ is the network defined in \cref{appendix:alg:1d_regression_nn_arch}. We use the AdamW optimizer \citep{loshchilov2017decoupled} with a 100-step warmup to a 0.002 learning rate, followed by cosine decay.

\textbf{Evidential NN}: Following \citep{Sensoy2018EvidentialDL}, we set the output dimension to $d=2$ and interpret $\bm{\alpha}(x) = 1 + \texttt{softplus}(h_\theta(x))$ as the parameters of a 2-class Dirichlet distribution. (We use \texttt{softplus} rather than \texttt{relu} to stabilize learning, since otherwise we observed that output units would ``die'' and produce bad estimates.)

We then apply the regularized cross-entropy loss described by \citet{Sensoy2018EvidentialDL}:
\begin{align*}
\mathcal{L}_\textsc{EDL}(x, y_1, y_2, \theta)
&= \frac{1}{2}\sum_{i=1}^2 \left[
\E_{q \sim \text{Dirichlet}(\bm{\alpha}(x))}[ -\log q(y_i) ]
+ \lambda D_{KL}\big(\,\text{Dirichlet}(\tilde{\bm{\alpha}}(x, y_i)) ~\|~ \text{Dirichlet}([1, 1])\,\big)
\right]
\\&= \frac{1}{2}\sum_{i=1}^2 \left[
\psi( \bm{1}^T \bm{\alpha}(x) ) - \psi( \bm{e}_{y_i}^T \bm{\alpha}(x) )
+ \lambda D_{KL}\big(\,\text{Dirichlet}(\tilde{\bm{\alpha}}(x, y_i)) ~\|~ \text{Dirichlet}([1, 1])\,\big)
\right],
\end{align*}
where $\bm{\alpha}(x)$ is the two-dimensional vector of model outputs, $\bm{e}_{y_i}$ is a one-hot indicator vector (either [1,0] or [0,1] depending on $y_i$), $\psi$ is the digamma function, and $\tilde{\bm{\alpha}}(x, y_i) = \bm{e}_{y_i} + (1 - \bm{e}_{y_i}) \odot \bm{\alpha}(x)$ is a vector where the Dirichlet parameter for the correct label has been replaced with 1.

Similar to the NN ensemble, we train the model for 10,000 training iterations with a batch size of 512, randomly selecting $(X, Y_1, Y_2)$ triples from the 25,000 training examples, and use the AdamW optimizer \citep{loshchilov2017decoupled} with a 100-step warmup to a 0.002 learning rate, followed by cosine decay. We interpolate $\lambda$ from 0 to 1 over 5,000 training steps, based on the recommended values for $\lambda$ in \citep{Sensoy2018EvidentialDL}.

\citet{Sensoy2018EvidentialDL} suggest using the magnitude of $\bm{\alpha}(x)$ as a measurement of evidence, with the uncertainty corresponding to the value $2 / \bm{1}^T \bm{\alpha}(x)$ (for two classes). However, in order to treat Evidential Deep Learning in the same way as other uncertainty estimates, we instead use the variance of the predicted probability under $\text{Dirichlet}(\bm{\alpha}(x))$ as our measurement of uncertainty. 

In \cref{fig:intro_1d_problem}, we plot the mean probability $\hat{p}(x)$ under the distribution $\text{Dirichlet}(\bm{\alpha}(x))$, as well as the variance $\hat{v}(x) = \frac{\hat{p}(x) (1 - \hat{p}(x))}{S(x) + 1}$. We additionally trained variants of EDL with different maximum values for $\lambda$, and found that the magnitude of the variance estimate is highly sensitive to this, as we demonstrate in \cref{appendix:fig:edl_kl_regularization}. The model in \cref{fig:intro_1d_problem} uses $\lambda = 1.0$. (We also tried applying EDL with the MSE loss, as suggested by \citet{Sensoy2018EvidentialDL}, but saw roughly identical behavior.)

\begin{figure*}[t!]
    \centering
\includegraphics[width=\linewidth]{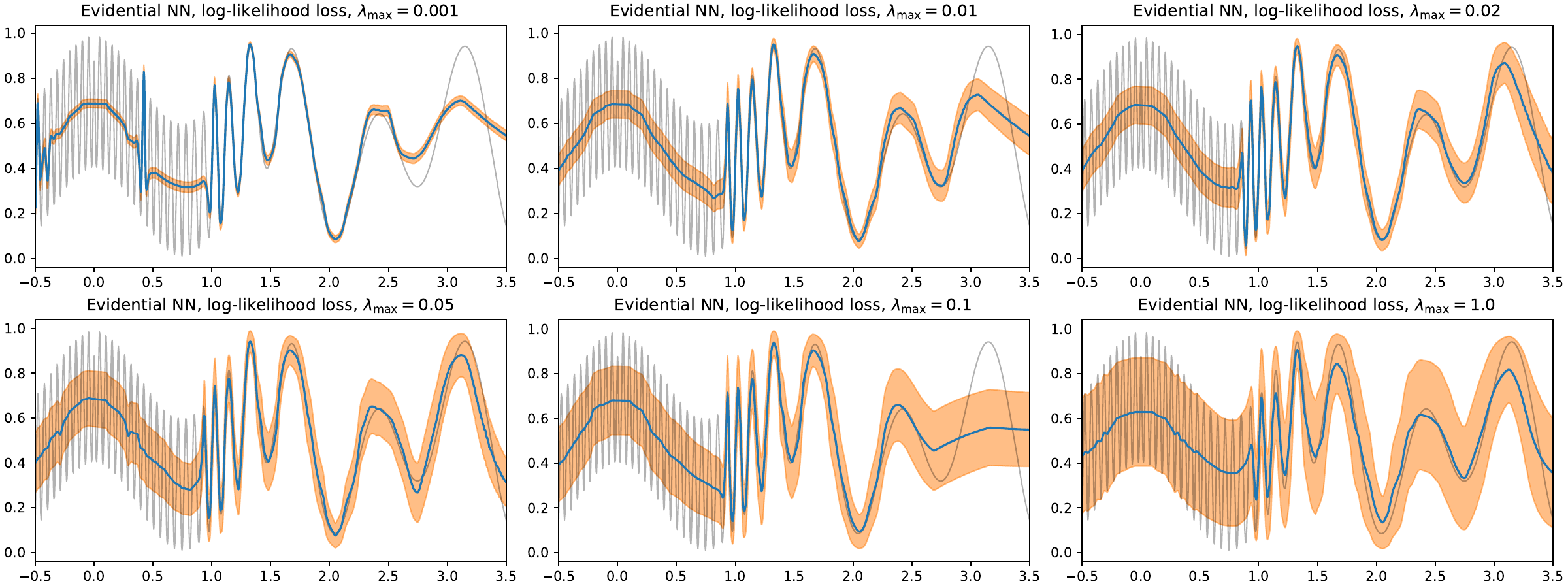}
    \vspace{-2em}
    \caption{Visualization of the dependence of the evidential deep learning technique \citep{Sensoy2018EvidentialDL} on the final regularization strength $\lambda_{\text{max}}$.  \citet{Sensoy2018EvidentialDL} recommend setting $\lambda_{\text{max}}=1$, which leads to high-uncertainty predictions even in regions that the model can fit well. We are unable to find any $\lambda_{\text{max}}$ value that allows the model to identify underfitting.}
    \label{appendix:fig:edl_kl_regularization}
\end{figure*}

\textbf{Cheat-corrected NN}: We parameterize our version of the architecture by setting $d=2$ and applying the decomposition in \cref{eqn:binary_kappa_form} of \cref{appendix:properties_calibrated_of_pairs}. We then train it to predict pairs using the loss
\[
\mathcal{L}_\textsc{Cheat}(x, y_1, y_2, \theta) = - \log \hat{p}_{\theta}(Y_1=y_1, Y_2=y_2 | X = x).
\]
We again train for 10,000 training iterations with a batch size of 512, randomly selecting $(X, Y_1, Y_2)$ triples from the 25,000 training examples, and use the AdamW optimizer \citep{loshchilov2017decoupled} with a 100-step warmup to a 0.002 learning rate, followed by cosine decay.
We then compute $\hat{p}_\theta(Y=1|X=x)$ and $v^\textsc{Cheat}_{\theta}(Y=1 | X=x)$ as described in \cref{appendix:properties_calibrated_of_pairs}.

\begin{figure*}[t!]
    \centering
\includegraphics[width=\linewidth]{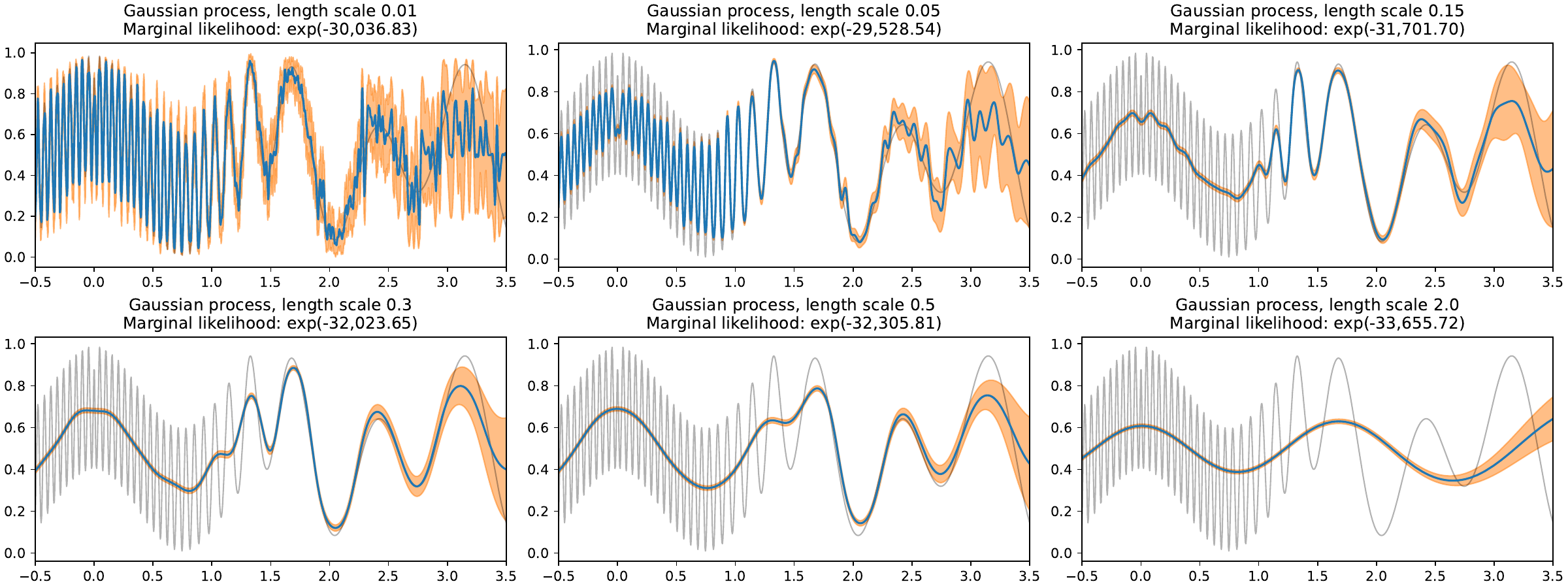}
    \caption{Visualization of the dependence of the Gaussian process inducing-points classifier on the length scale, and the estimated marginal likelihood of the dataset points under each prior.}
    \label{appendix:fig:gp_length_scale}
\end{figure*}

\textbf{Gaussian process classifier}: For our Gaussian process experiment, we follow a standard discriminative Gaussian process classifier setup \citep{rasmussen2005gpclassification}: we impose a Gaussian process prior over a latent ``logit'' function $f : \R \to \R$, then feed it through the logistic sigmoid transformation to obtain a conditional likelihood
\[
p(Y=1|f,x) = \sigma(f(x)) = \frac{1}{1 + \exp(-f(x))}.
\]
Given an observed dataset $\mathcal{D} = \left\{\big(x^{(i)}, y^{(i)}\big)\right\}_{i=1}^N$ of $(X, Y)$ pairs, we can then approximately compute the posterior distribution
\[
p\Big(f \Big| \left\{\big(x^{(i)}, y^{(i)}\big)\right\}_{i=1}^N\Big) \propto p(f) \prod_{i=1}^N p(y^{(i)} |f, x^{(i)})
\]
and use it to compute the posterior mean and variance for a new data point $x$:
\begin{align*}
p(Y=1|x, \mathcal{D}) &= \int_f \sigma(f(x)) p(f|\mathcal{D})\,df,
\\
\Var\big[ p(Y=1|f, x) \big| \mathcal{D}\big] &= \int_f \sigma(f(x))^2 p(f|\mathcal{D})\,df - p(Y=1|x, \mathcal{D})^2.
\end{align*}
For this task, we select a rational-quadratic kernel with standard deviation 2.0, mixture parameter 1.0, and length scale $0.15$:
\[
\text{Cov}\big(f(a), f(b)\big)
= 2.0^2 \left( 1 + \frac{(b - a)^2}{2 \times 0.15^2} \right)^{-1}
\]

Our training set includes 50,000 $Y$ samples, so computing an analytic posterior over $f$ is computationally difficult. We instead use a variational approximation using inducing points, following \citep{hensman2015scalable}: we choose $K$ inducing points $z^{(1)}, \dots, z^{(K)}$, let $u=\{ f(z^{(k)}) \}_{k=1}^{K}$ be the latent function values for those points, then impose an approximate posterior
\[
q(f) = \int_u p(f | u) q(u) \,du,
\]
which can be used to construct an evidence lower bound on the likelihood
\[
\log p(\mathcal{D}) \ge \E_{q(f)}\left[\sum_{i} \log p(y^{(i)}|f,x^{(i)})\right] + D_{KL}(q(u) \| p(u)).
\]
We select $K=512$ inducing points evenly spaced between -4 and 4, parameterize $q(u)$ as a Cholesky-factorized multivariate normal distribution $q(u) = \mathcal{N}(u; \mu, LL^T)$ where $\mu \in \R^{512}, L \in \R^{512\times 512}$, and use the Cholesky factorization to analytically compute the KL divergence. 
(For numerical stability purposes, we add 0.001 to the diagonal of the prior covariance matrix.) We then maximize the evidence lower bound above, approximating the expectation by subsampling 512 $(x, y_1, y_2)$ triples per iteration and using Gauss-Hermite quadrature over the distribution $q(f(x^{(i)})|u)$; we treat the two samples $y_1$ and $y_2$ for each $x$ as independent observations $(x, y_1)$, $(x, y_2)$.
We optimize $\mu$ and $L$ for 20,000 training iterations using stochastic gradient descent, with a maximum learning rate of 0.05, 100 steps of warmup, and a cosine decay schedule, although we observe that the approximate posterior converges within about half of that time.

We note that the degree of misspecification varies based on the length scale, as shown in \cref{appendix:fig:gp_length_scale}, because the true function does not have a consistent length scale and was not chosen from the prior. If we know in advance which length scales to try, the marginal likelihood estimates (our bound on $p(\mathcal{D})$) may allow us to identify the best-fitting model (in this case, the version with length scale 0.05, although it is imperfect). However, the estimates of the variance of $f$ does not provide a good estimate of pointwise misspecification; we chose to use length scale 0.15 in \cref{fig:intro_1d_problem} to emphasize this. (In real world settings, misspecification would likely be much harder to detect or correct, especially without a thorough hyperparameter sweep.) We also note that the marginal likelihoods of each approach are roughly of the same order; similar-looking data could plausibly have been generated by even the misspecified models because the data itself consists of binary outputs.

An expanded version of each of the parts of \cref{fig:intro_1d_problem} is shown in \cref{appendix:fig:big_1d_problem}.

\begin{figure*}[p]
    \centering
\includegraphics[width=0.95\linewidth]{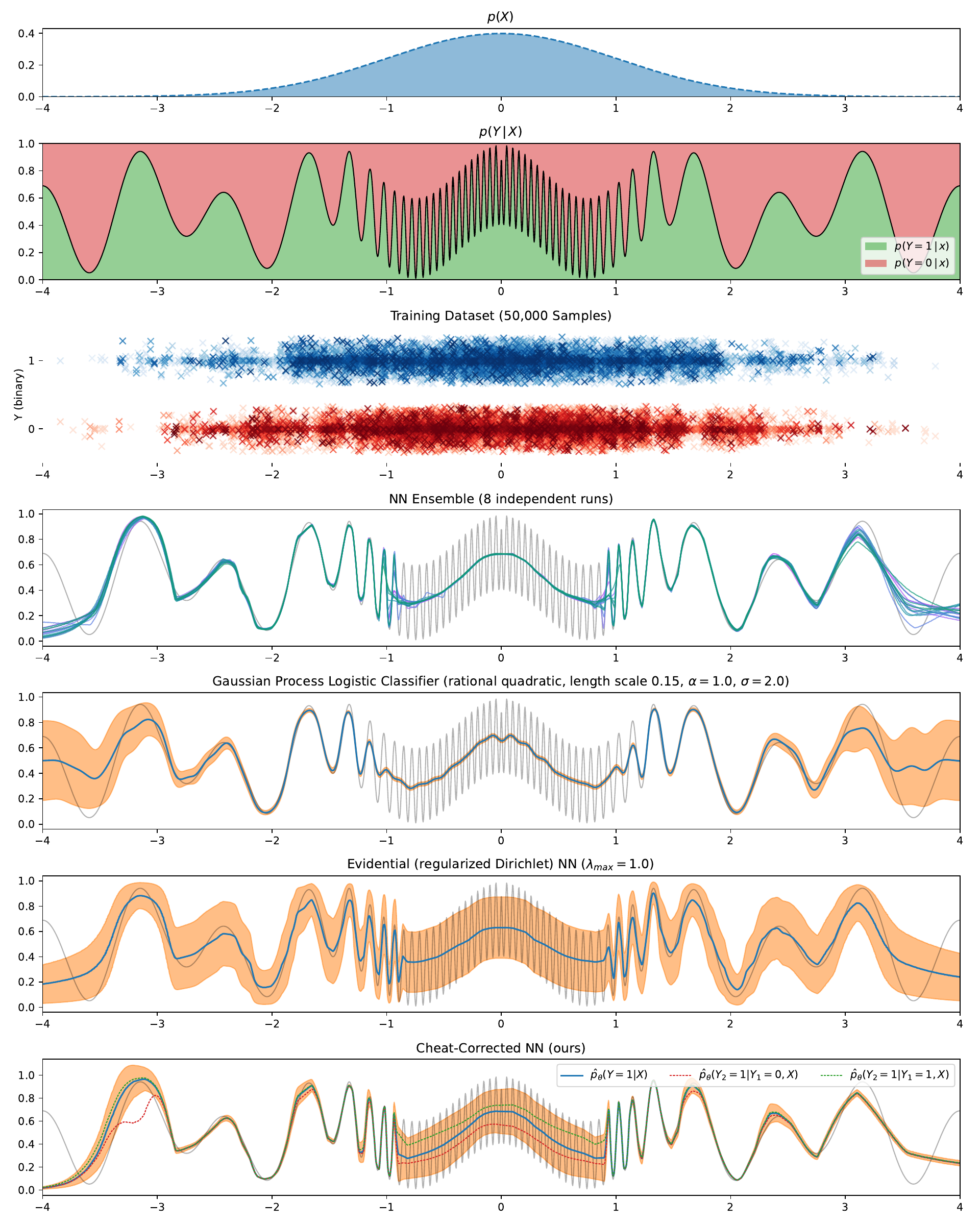}
    \vspace{-2em}
    \caption{Expanded comparison of methods for the 1D binary regression task shown in \cref{fig:intro_1d_problem}, showing both sides of the normal distribution $p(X)$. Note that our method is the best at distinguishing underfitting (near 0) from accurate estimation (1 to 3). Our technique does not directly do out-of-distribution detection, so this model's uncertainty estimates are overconfident outside the range of the dataset ($x > 3$); this could be fixed by combining our method (which detects underfitting) with another method that improves out-of-distribution calibration (as we did with the Cheat-SNGP model in the CIFAR-10H task).}
    \label{appendix:fig:big_1d_problem}
\end{figure*}

\subsubsection{Additional results}

\cref{appendix:fig:intro_problem_ece1_reliability} shows a reliability diagram for first-order calibration for the task in \cref{fig:intro_1d_problem}, demonstrating that all methods are close to first-order calibrated on this task. 
\cref{appendix:fig:intro_problem_ece2_reliability} shows a similar plot for second-order calibration, which indicates that our method is indeed better second-order calibrated. Each point in these figures was computed by aggregating over 20 equal-probability-mass bins; perfect calibration would correspond to a diagonal line with slope 1.

To show that these results are not specific to the particular sinusoidal function we chose, we also present results for a randomly-selected piecewise linear function, shown in \cref{appendix:fig:piecewise_linear_problem}. Similar to the sinusoidal function in \cref{fig:intro_1d_problem}, the ensemble and Gaussian process methods tend to overestimate their confidence for high-probability inputs around $x = 0$. Our technique gives a better estimate for common $x$, although it is overconfident for $x < -3$. Overall, it is better second-order calibrated and similarly first-order calibrated, as shown in \cref{appendix:fig:piecewise_linear_ece1_reliability,appendix:fig:piecewise_linear_ece2_reliability}.

\begin{figure*}[t]
    \centering
\includegraphics[width=0.95\linewidth]{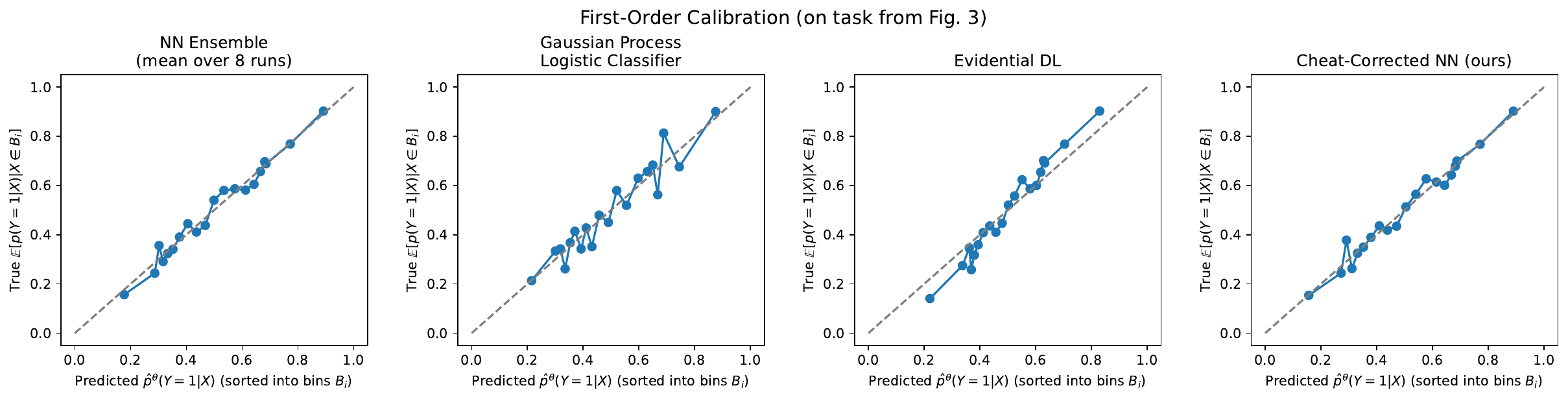}
    \caption{First-order calibration reliability diagram for the task in \cref{fig:intro_1d_problem}. All methods are close to first-order calibrated on this task.}
    \label{appendix:fig:intro_problem_ece1_reliability}
\end{figure*}
\begin{figure*}[t]
    \centering
\includegraphics[width=0.95\linewidth]{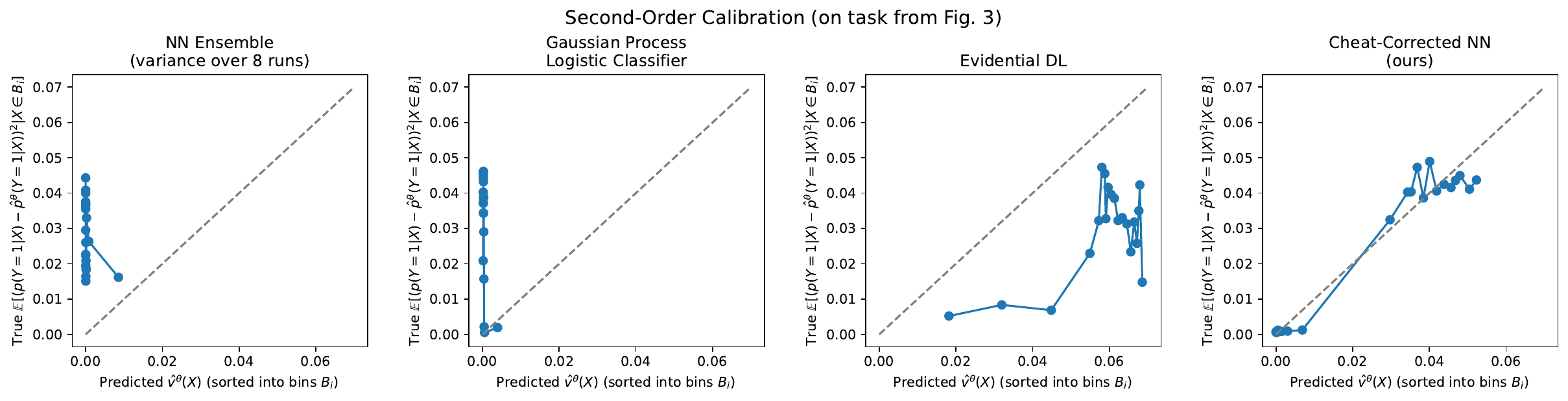}
    \caption{Second-order calibration reliability diagram for the task in \cref{fig:intro_1d_problem}. Our method is better second-order-calibrated than the baseline approaches.}
    \label{appendix:fig:intro_problem_ece2_reliability}
\end{figure*}

\begin{figure*}[p]
    \centering
\includegraphics[width=0.95\linewidth]{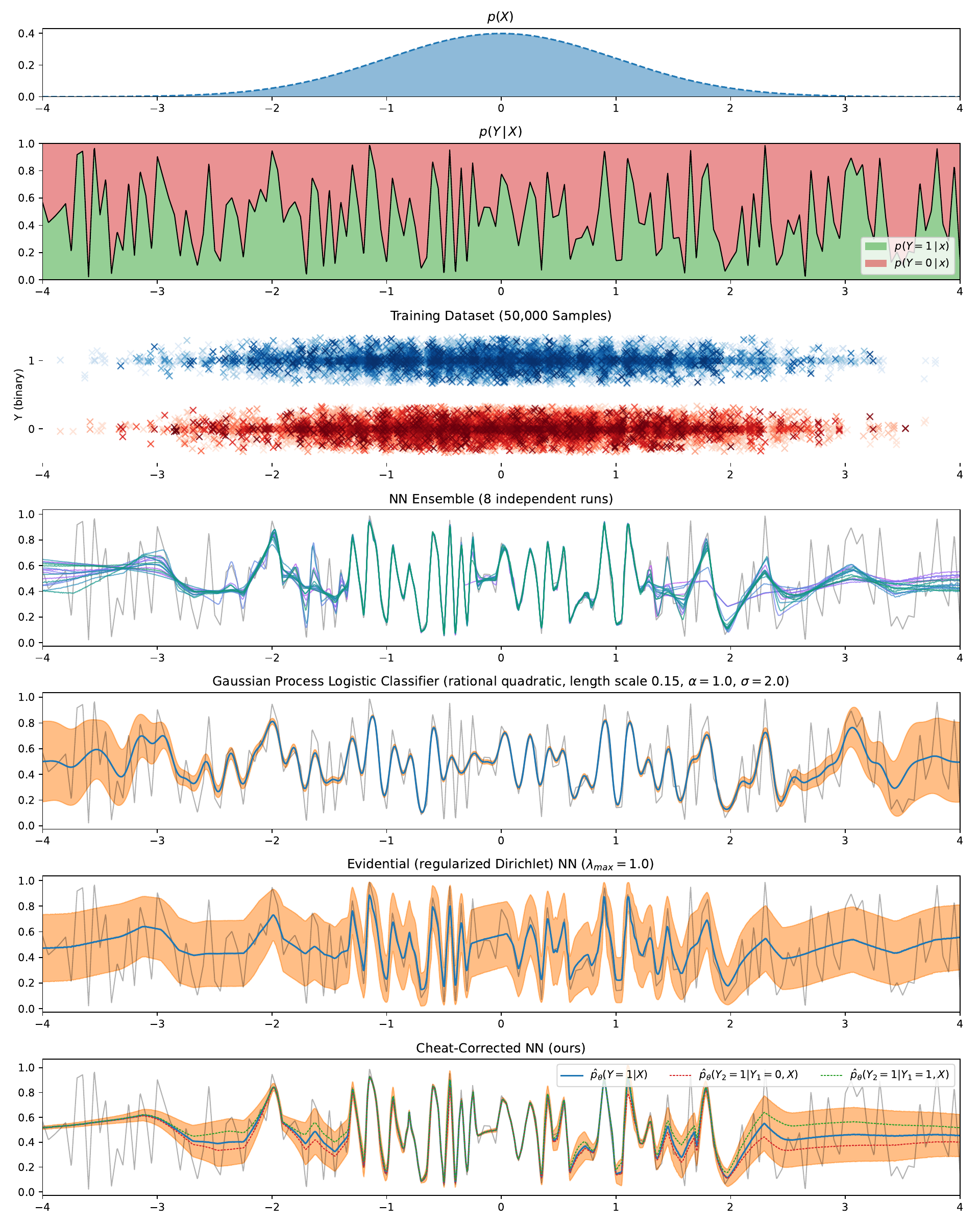}
    \vspace{-2em}
    \caption{Variant of \cref{fig:intro_1d_problem} for a toy task with a random piecewise-linear function.}
    \label{appendix:fig:piecewise_linear_problem}
\end{figure*}

\begin{figure*}[t]
    \centering
\includegraphics[width=0.95\linewidth]{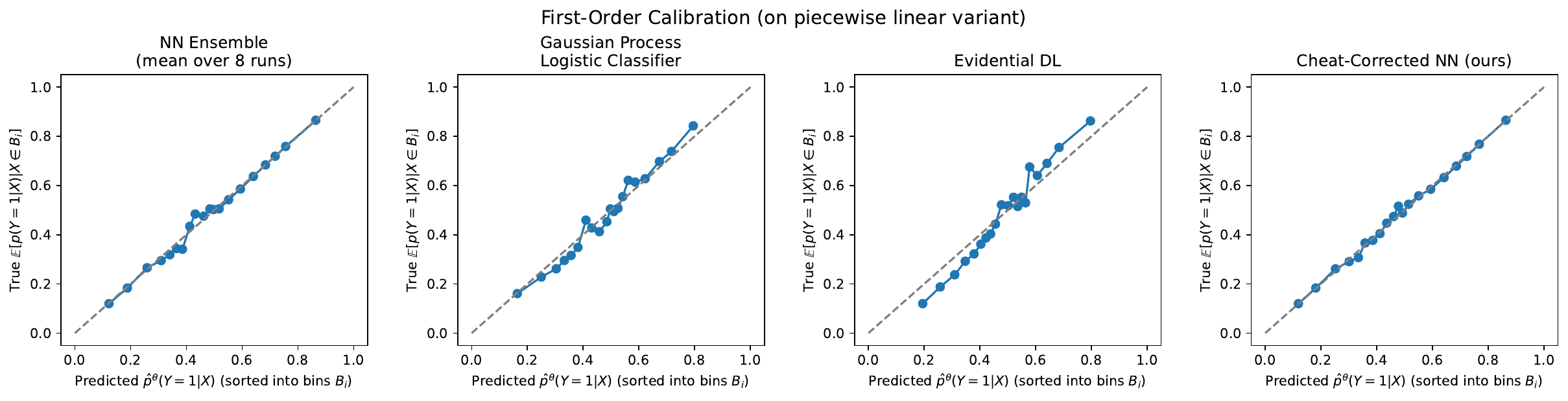}
    \caption{First-order calibration reliability diagram for the task in \cref{appendix:fig:piecewise_linear_problem}. All methods are close to first-order calibrated on this task.}
    \label{appendix:fig:piecewise_linear_ece1_reliability}
\end{figure*}
\begin{figure*}[t]
    \centering
\includegraphics[width=0.95\linewidth]{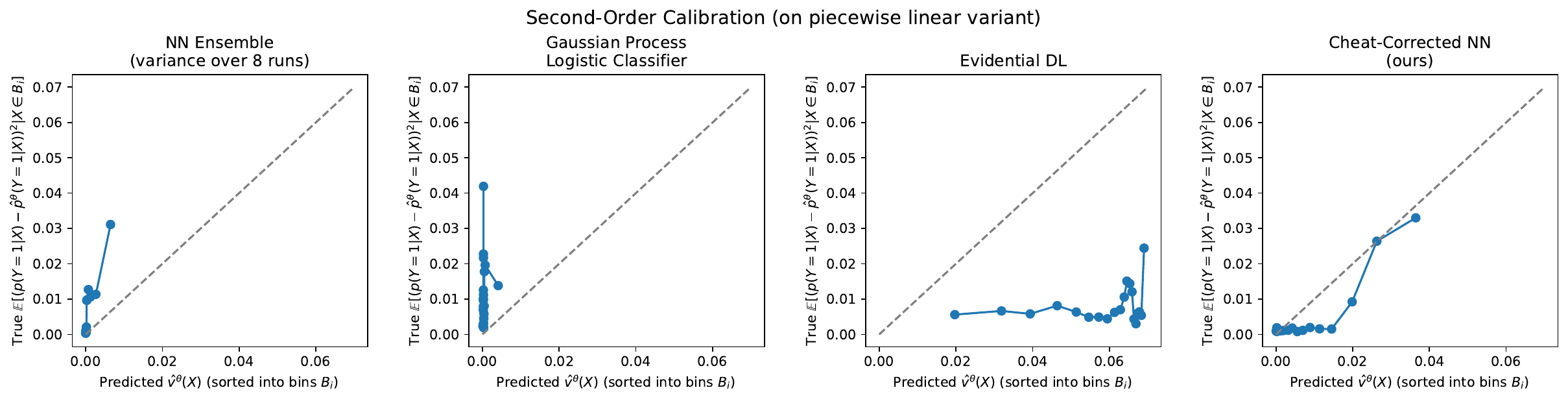}
    \caption{Second-order calibration reliability diagram for the task in \cref{appendix:fig:piecewise_linear_problem}. Our method is better second-order-calibrated than the baseline approaches.}
    \label{appendix:fig:piecewise_linear_ece2_reliability}
\end{figure*}

\clearpage
\subsection{CIFAR-10H}\label{appendix:sec:cifar10h}

\subsubsection{Training and hyperparameter tuning strategy}

Due to the small size of CIFAR-10H \citep{peterson2019human}, we split the dataset into multiple parts and combine it with CIFAR-10N \citep{Wei2021LearningWN}:

\begin{itemize}
    \item We use the 50,000 images in CIFAR-10N as a \emph{pretraining set}; this corresponds to the training set of the original CIFAR-10 dataset. Each image has three  labels from random annotators; we select two from these randomly at each training step. (Note that the distribution of aleatoric noise in CIFAR-10N is not exactly the same as the aleatoric noise in CIFAR-10H, likely due to being labeled at different times by different annotators, under the direction of different authors.)
    \item We use the first 3,000 images in CIFAR-10H as \emph{finetuning set 1}. Each image has at least 50 annotations; we select two from these randomly at each training step.
    \item We use the next 2,000 images in CIFAR-10H as our \emph{validation set}. We use the 50 labels per image as a proxy for the true distribution $p(Y|X=x)$, and take the KL divergence between this empirical distribution and the model's predictions as our hyperparameter tuning objective. After hyperparameter tuning, we reuse this set of images as \emph{finetuning set 2}, combining it with the first 3,000 images in CIFAR-10H to form a 5,000-image set.
    \item Finally, we use the last 5,000 images in CIFAR-10H as our \emph{test set}, and use it to evaluate our final metrics.
\end{itemize}

We apply the AugMix augmentation strategy \citep{hendrycks2019augmix} when sampling examples from the pretraining and finetuning sets, to improve robustness of all methods. We also normalize pixel values based on the mean and standard deviation of pixels across the full CIFAR-10 dataset, following the implementation in \texttt{uncertainty\_baselines}.

We train each method using the AdamW optimizer \citep{loshchilov2017decoupled} with batch size 512. We divide our training and hyperparameter tuning into the following phases:

\begin{itemize}
    \item \emph{Phase 1: Pretraining hyperparameter sweep.} We train each method on the CIFAR-10N pretraining set for 50 epochs. We perform a random search over learning rate and weight decay strength with 250 trials: we choose learning rate logarithmically spaced between $10^{-5}$ and $5 \times 10^{-3}$, and we either sample weight decay uniformly between 0.05 and 0.5, or logarithmically between $10^{-6}$ and $0.05$, since different tasks and methods may benefit from either very strong or very weak weight decay. We use a linear warmup for the learning rate during the first epoch, then use cosine weight decay.
    \item \emph{Phase 2: Pretraining extended.} We take the best-performing hyperparameters from phase 1, as judged by validation KL divergence, and retrain that configuration from scratch for both 500 and 200 epochs.
    \item \emph{Phase 3: Fine-tuning hyperparameter sweep.} We train each method on our CIFAR-10H finetuning set 1. We perform a random search over learning rate and weight decay, as in phase 1, and also randomly search over the number of epochs to use, either 30, 50, 100, or 200. We initialize the parameters from either of the checkpoints from Phase 2, using 250 trials in the random sweep for each checkpoint. (Effectively, we do 500 trials, where we are also tuning the number of epochs of pretraining.)
    \item \emph{Phase 4: Expanded fine-tuning.} We take the best-performing configuration from Phase 3, again judged by validation KL divergence. We then reset to the checkpoint from Phase 2 (depending on the Phase 3 configuration), and train it on the combination of of finetuning set 1 and finetuning set 2 (the validation set), so that we maximize the amount of finetuning data.
    \item \emph{Phase 5: Evaluation.} We take the resulting model from Phase 4 and evaluate our metrics on our test set (the second half of CIFAR-10H).
\end{itemize}

The best-performing hyperparameters for each method are shown in \cref{appendix:tab:cifar10h-hyperparameters}. Interestingly, we found that very strong weight decay was the most effective during pretraining for all models, perhaps because of the relatively small dataset and large number of epochs.
 
\begin{table*}[t]
\caption{
Best-performing hyperparameters for each model architecture, based on our hyperparameter sweeps.
}
\label{appendix:tab:cifar10h-hyperparameters}
\begin{center}
\begin{small}
\begin{sc}
\begin{tabular}{lcccccc}
\toprule
& \multicolumn{3}{c}{Pretrain (CIFAR-10N)}& \multicolumn{3}{c}{Finetune (CIFAR-10H)}
\\
\cmidrule(lr){2-4}\cmidrule(lr){5-7}
Method & \scriptsize Learning rate & \scriptsize Weight decay & \scriptsize Epochs & \scriptsize Learning rate & \scriptsize Weight decay & \scriptsize Epochs \\
\midrule
Naive NN / NN Ensemble & 3.799e-03 & 3.656e-01 & 50 & 9.071e-05 & 2.363e-01 & 50 \\
Evidential DL & 7.465e-04 & 2.950e-01 & 200 & 1.455e-04 & 2.146e-03 & 200 \\
SNGP Cov. & 1.221e-03 & 4.742e-01 & 200 & 6.032e-05 & 1.457e-01 & 100 \\
Epinet & 1.327e-03 & 4.513e-01 & 50 & 7.345e-05 & 2.954e-01 & 50 \\
Cheat NN & 1.113e-03 & 4.835e-01 & 50 & 4.265e-05 & 1.865e-01 & 50 \\
Cheat SNGP & 1.687e-03 & 4.411e-01 & 200 & 4.972e-05 & 7.621e-05 & 30 \\
\bottomrule
\end{tabular}
\end{sc}
\end{small}
\end{center}
\end{table*}

\subsubsection{Model architectures}

We implement all of our models and baselines using the \texttt{uncertainty\_baselines} Python library \citep{nado2021uncertainty}, building on TensorFlow \citep{tensorflow2015whitepaper} and Keras \citep{chollet2015keras}. All methods are based on the wide ResNet architecture \citep{zagoruyko2016wide} as implemented in \texttt{uncertainty\_baselines}, with depth 28 and a width multiplier of 10.

\textbf{Naive NN:} We configure the wide ResNet with 10 output classes and a softmax output layer, and train it using the ordinary cross entropy (negative log likelihood) loss.

At test time, we use $\vhat(y|x) = \phaty(y | x)(1-\phaty(y | x))$ as an estimate of variance. This corresponds to the assumption that all noise is epistemic, and would make sense if we knew $Y$ was a deterministic function of $X$. (However, for this task we know this is not the case, so this will overestimate uncertainty.)

\textbf{NN Ensemble:} We use the same configuration as for Naive NN, but train eight copies of the model. (We do not perform a separate hyperparameter tuning sweep for NN ensemble, since it would be the same as the sweep for Naive NN.)

At test time, we take the empirical mean and variance across the ensemble as our prediction and epistemic uncertainty estimates:
\begin{align*}
    \hat{p}_{\theta_1, \dots, \theta_8}(y|x) &= \frac{1}{8} \sum_{i=1}^8 \hat{p}_{\theta_i}(y | x)
    \\
    \hat{v}_{\theta_1, \dots, \theta_8}(y|x) &= \frac{1}{7} \sum_{i=1}^8 \left( \phaty(y | x) - \hat{p}_{\theta_1, \dots, \theta_8}(y|x) \right)^2
\end{align*}
We divide by 7 so that our sample variance estimator is an unbiased estimate of the variance under a hypothetical infinite ensemble.

\textbf{SNGP Cov.}: We use the SNGP variant of a wide ResNet from \texttt{uncertainty\_baselines}, which applies spectral normalization to the intermediate layers of the ResNet and replaces the normal linear output layer with a random-feature Gaussian process approximation \citep{Liu2020SimpleAP}. We configure it using the default configuration for CIFAR-10: 1024 orthogonal random features, a 1.0 ridge penalty, a 20 mean-field factor, and a spectral-norm bound of 6.0.

Following the training script for SNGP in \texttt{uncertainty\_baselines}, we use the ordinary logits of the model during training; this roughly corresponds to using the posterior mean of the learned Gaussian process posterior. After training the model, we perform another pass over the training set to compute a Laplace approximation of the covariance matrix. To transform the mean and covariance over the logits into a mean and variance over output probabilities, we use a Monte Carlo approximation by applying softmax to each of 1000 samples, then taking the mean and variance of the resulting probability vectors.

\textbf{Epinet:} We augment our ResNet base network architecture with a MLP epinet head, using the implementation in the official \texttt{enn} library\footnote{https://github.com/google-deepmind/enn} \citep{Osband2021EpistemicNN}, which we wrap using the \texttt{jax2tf} library in JAX \citep{jax2018github}. We copy the hyperparameters from the CIFAR-10 checkpoints in that repository: a 20-dimensional index vector, 50-dimensional hiddens, an epinet prior scale of 4.0, and no additional convolutional prior network. The epinet head takes the penultimate layer features from the ResNet and makes an additive contribution to the ResNet's outputs, indexed by a random input. (Note that, although the \texttt{enn} library checkpoints are for a non-wide ResNet, we instead use our wide ResNet backbone for consistency with the other baselines, and train the epinet from scratch.)

We train the base network and epinet head jointly from scratch, taking an average of the ordinary cross-entropy loss across five randomly-sampled ``index'' vectors, as recommended by \citep{Osband2021EpistemicNN}.
We then evaluate the epinet predictions on our test set by taking the mean and variance of the post-softmax probabilities across 1000 sampled index vectors for each input.

\textbf{Evidential NN:} We set up the wide ResNet architecture with 10 outputs, and convert them into parameters for a Dirichlet distribution according to $\bm{\alpha}(x) = 1 + \texttt{softplus}(h_\theta(x))$. 
We then apply the regularized cross-entropy loss described by \citet{Sensoy2018EvidentialDL}:
\begin{align*}
\mathcal{L}_\textsc{EDL}(x, y_1, y_2, \theta)
&= \frac{1}{2}\sum_{i=1}^2 \left[
\E_{q \sim \text{Dirichlet}(\bm{\alpha}(x))}[ -\log q(y_i) ]
+ \lambda D_{KL}\big(\,\text{Dirichlet}(\tilde{\bm{\alpha}}(x, y_i)) ~\|~ \text{Dirichlet}([1, 1])\,\big)
\right]
\\&= \frac{1}{2}\sum_{i=1}^2 \left[
\psi( \bm{1}^T \bm{\alpha}(x) ) - \psi( \bm{e}_{y_i}^T \bm{\alpha}(x) )
+ \lambda D_{KL}\big(\,\text{Dirichlet}(\tilde{\bm{\alpha}}(x, y_i)) ~\|~ \text{Dirichlet}([1, 1])\,\big)
\right],
\end{align*}
where $\bm{e}_{y_i} \in \R^{10}$ is a one-hot indicator vector for the correct class, $\psi$ is the digamma function, and $\tilde{\bm{\alpha}}(x, y_i) = \bm{e}_{y_i} + (1 - \bm{e}_{y_i}) \odot \bm{\alpha}(x)$ is a vector where the Dirichlet parameter for the correct label has been replaced with 1. We interpolate $\lambda$ from 0 to 1 over 10 epochs as recommended by \citet{Sensoy2018EvidentialDL}.

At test time, we compute the average probabilities and variances as
\begin{align*}
    \phaty(Y=y | X) &= \frac{\bm{e}_{y}^T \bm{\alpha}(x)}{\bm{1}^T \bm{\alpha}(x)},
    \\
    \vhat(Y=y | X) &= \frac{\phaty(Y=y | X) \big( 1- \phaty(Y=y | X) \big)}{\bm{1}^T \bm{\alpha}(x) + 1}.
\end{align*}
This is the mean and variance of the probability $q(y)$ when $q \sim \text{Dirichlet}(\bm{\alpha}(x))$, as described by \citet{Sensoy2018EvidentialDL}.

We were initially surprised to find that the Evidential NN has significantly worse KL divergence and calibration scores in \cref{tab:cifar10h} compared to other methods, but only a moderate reduction in classification error. After further investigation, we determined that this is because, in the presence of label noise between a few classes, the Evidential NN's regularization causes it to predict an almost-uniform distribution across all classes, even classes that never appear (as shown in \cref{appendix:fig:edl_weird_behavior}), due to predicting a Dirichlet distribution that is nearly uniform over the simplex. Additionally, because the denominator $\bm{1}^T \bm{\alpha}(x) + 1$ is always at least \texttt{num\_classes} + 1, the variance estimate will never be larger than $\frac{0.5^2}{11} \approx 0.022$, and if the predicted probability is close to uniform, the variance estimate will be around $\frac{0.1 \cdot 0.9}{11} \approx 0.0082$. This may be smaller than the actual squared error of the predictor. (This occurs because the uniform distribution over a high-dimensional simplex actually has very low variance along each dimension.)

\begin{figure*}[t]
    \centering
\includegraphics[width=\linewidth]{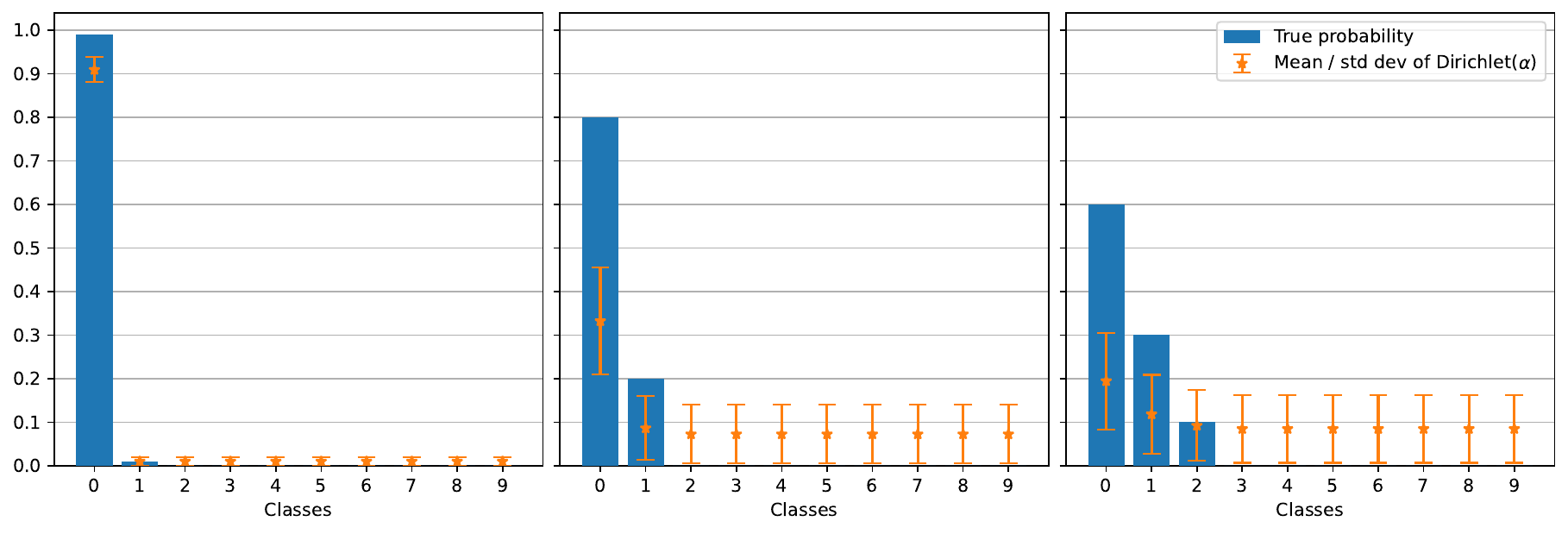}
    \vspace{-2em}
    \caption{Optimizing the Evidential Deep Learning objective \citet{Sensoy2018EvidentialDL} produces biased probability estimates in the presence of label noise. We let $\bm{\alpha} = 1 + \text{softplus}(\bm{v})$ and directly minimize the EDL objective with respect to $\bm{v}$, taking an expectation over a synthetic ground-truth label distribution (blue bars). We then visualize the mean and standard deviation of the learned distribution (orange). When the ground-truth distribution has aleatoric uncertainty, EDL both over-estimates the probability for never-observed classes, and under-estimates the distance from the true label distribution.}
    \label{appendix:fig:edl_weird_behavior}
\end{figure*}

\textbf{Cheat NN:} We parameterize the output layer of the wide ResNet with $10 \times 10 = 100$ output classes. We then reshape them into a $10 \times 10$ matrix and add it to its transpose to enforce that it is symmetric, and then take a softmax over all rows and columns. To regularize this output and prevent negative variance estimates due to overfitting, we compute the eigenvalues of this probability matrix at each training step, then regularize any negative eigenvalues by  multiplying their squared norm by 10.0. See \cref{appendix:properties_calibrated_of_pairs} for additional discussion.

In contrast to the previous approaches, which average the loss over the two samples $y_1, y_2$ for each example seen, for our method we directly compute the log-likelihood of the \emph{pair} of outputs, by indexing into the appropriate row and column of our joint probability matrix.

At test time, we compute $\phatym(y | x)$ by marginalizing out one of the axes of our symmetric $10 \times 10$ matrix. We compute $\vcheat(y | x)$ using
\[
\vcheat(y | x) = \phatyy(y, y | x) - \phatym(y | x)^2.
\]

Note that, for hyperparameter tuning, we compute the KL divergence according to the predicted marginal distribution of $Y_1$ only, and compare it to the empirical distribution of all 50 annotator labels; we do not use the conditional $\phatyc$ (or the joint $\phatyy$) for hyperparameter tuning.

\textbf{Cheat SNGP:} We follow the same procedure as for Cheat NN, but use the SNGP variant of the wide ResNet architecture. This means we use the spectral normalization layers and random-feature Gaussian process output head. However, we do not compute any posterior covariance using the Gaussian process output head, and instead merely use the random features as a convenient parameterization for a deterministic output layer. This allows us to take advantage of the distance-awareness inductive biases in the SNGP architecture \citep{Liu2020SimpleAP} without needing to approximate an actual posterior distribution.

\subsubsection{Dataset variants}

In addition to the original dataset, we consider three dataset variants: extra classes, scrambled, and extra classes + scrambled. Each of these variants is implemented by transforming either the images or the annotator labels for the tasks, and we apply the transformations to all of the dataset splits (pretraining, fine-tuning, validation, and test).

For the original and the extra classes + scrambled variants, we aggregate over eight independent training and evaluation runs. The average performance is shown in the main paper in \cref{tab:cifar10h}, and the standard deviations are given in \cref{tab:cifar10h-standard-deviations}. For the other two variants, we only conduct a single training run; the results for these variants are given in \cref{tab:cifar10h-extra-variants}.

We do not perform separate hyperparameter sweeps for each variant; instead we re-use the optimal hyperparameters for the unmodified dataset, and just run training phases 2, 4, and 5 (as described above). Example images from each of these dataset variants are shown in \cref{appendix:fig:cifar_variants}.

\begin{figure*}[p]
    \centering
\includegraphics[width=\linewidth]{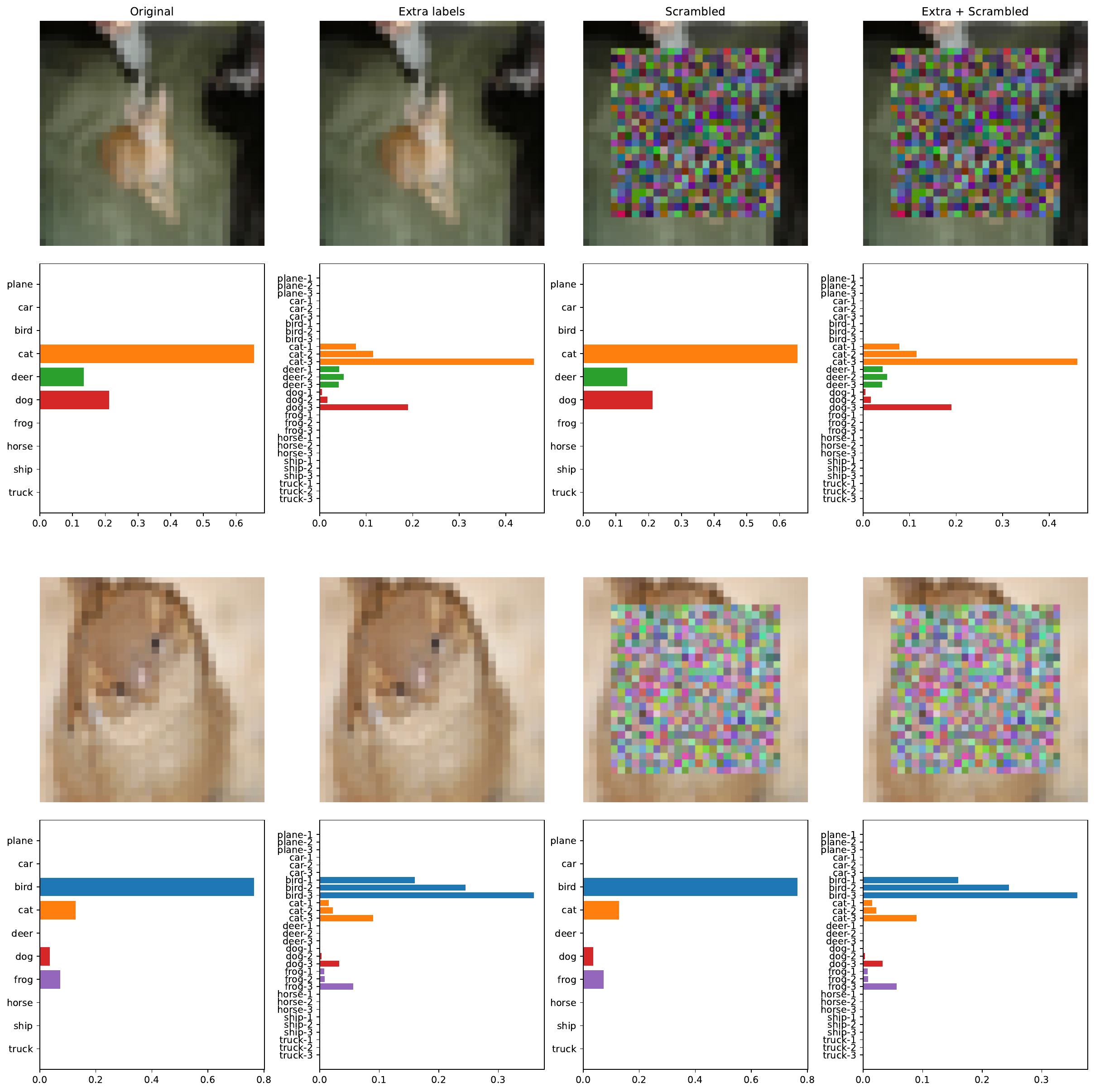}
    \vspace{-2em}
    \caption{Input images $X$ and ground-truth annotator label distributions $p(Y|X)$ for two images in the CIFAR-10H dataset, selected due to having natural aleatoric uncertainty in the label distribution. The ``Extra classes'' variant modifies the label distribution, whereas the ``Scrambled'' variant scrambles the center patch of the image.}
    \label{appendix:fig:cifar_variants}
\end{figure*}

\textbf{Extra Classes:} In this variant, we add label noise by artificially increasing the number of classes. For each of the classes in the original CIFAR-10 dataset, we create three new classes (e.g. dog $\to$ \{ dog-1, dog-2, dog-3\} ), and arbitrarily construct a distribution over them (e.g. $p(Y' | Y = \text{dog})$ for $Y' \in \{ \text{dog-1, dog-2, dog-3} \}$) by sampling from Dirichlet([1, 1, 1]). This produces a classification problem with 30 classes instead of 10, where there is aleatoric variation between the sub-classes even for unambiguous images. The conditional distribution for each class is held fixed for all models and all dataset splits (i.e. it is treated as part of the data distribution itself).

When training with this variant, we increase the number of outputs of each method accordingly; methods that had output dimension 10 instead use output dimension 30, and our cheat-corrected models are configured to produce $30 \times 30$ joint matrices. We then sample a sub-class for each class in each training iteration, potentially using different sub-classes for the same image in different epoch.
When evaluating metrics, we multiply the empirical distribution over the original labels and the closed-form conditional distribution for the sub-labels to construct a partially-empirical distribution over the sub-labels.

We note that this dataset has more aleatoric variation, but it shouldn't require more capacity, since the noise is added in an image-agnostic way.

\textbf{Scrambled:} In this variant, we increase the task difficulty by applying a fixed permutation to the pixels in the center of the image. This permutation applies across all channels and all pixels except for a 4-pixel border around the sides of the image. Since the content is usually in the center of the image, and since a ResNet is a convolutional architecture, this permutation is likely to interfere with the inductive biases of all of the methods, and may cause them to focus on less-informative features in the image border.

Since our permutation is invertible, it is always possible in principle to reconstruct the original image from the scrambled form. This means that the true conditional $p(Y|X)$ is not affected by using a scrambled view of $X$, so all of the additional uncertainty from this transformation is epistemic in nature.

We apply this permutation after the AugMix augmentations during training.

\textbf{Extra + Scrambled:} We apply both of the transformations above together.

\begin{table*}[t]
\caption{
Full results from \cref{tab:cifar10h} with mean and standard deviation across eight training runs. All metrics summed across classes except Acc and KL.
}
\label{tab:cifar10h-standard-deviations}
\vskip 0.15in
\begin{center}
\begin{small}
\begin{sc}
\begin{tabular}{l@{\hskip 2em}cc@{\hskip 1.5em}c@{\hskip 1.5em}ccc}
\toprule
& \multicolumn{6}{c}{CIFAR-10H}
\\
\cmidrule(l{-2pt}r{2em}){2-7}
Method & \scriptsize \textbf{ECE-2} & \scriptsize $\E[\hat{v}^\theta]$ & \clap{\scriptsize $\E[(\hat{p}^\theta{-}p)^2]$} & \scriptsize ECE-1 & \scriptsize Acc & \scriptsize KL
\\
\midrule
Naive NN & 0.076 $\pm$ 0.003 & 0.142 $\pm$ 0.001 & 0.065 $\pm$ 0.002 & 0.017 $\pm$ 0.003 & 93.94 $\pm$ 0.17 & 0.179 $\pm$ 0.003 \\
NN Ensemble & 0.039 $\pm$ 0.000 & 0.014 $\pm$ 0.000 & 0.053 $\pm$ 0.000 & 0.029 $\pm$ 0.002 & 94.90 $\pm$ 0.09 & 0.152 $\pm$ 0.001 \\
Evidential DL & 0.377 $\pm$ 0.001 & 0.053 $\pm$ 0.000 & 0.430 $\pm$ 0.002 & 1.038 $\pm$ 0.004 & 88.45 $\pm$ 0.23 & 1.087 $\pm$ 0.003 \\
SNGP Cov. & 0.048 $\pm$ 0.001 & 0.005 $\pm$ 0.000 & 0.052 $\pm$ 0.001 & 0.020 $\pm$ 0.001 & 94.89 $\pm$ 0.15 & 0.153 $\pm$ 0.001 \\
Epinet & 0.056 $\pm$ 0.002 & 0.015 $\pm$ 0.002 & 0.071 $\pm$ 0.001 & 0.020 $\pm$ 0.002 & 93.40 $\pm$ 0.26 & 0.189 $\pm$ 0.002 \\
\midrule
Cheat NN & 0.018 $\pm$ 0.001 & 0.052 $\pm$ 0.000 & 0.068 $\pm$ 0.001 & 0.029 $\pm$ 0.003 & 93.60 $\pm$ 0.18 & 0.182 $\pm$ 0.002 \\
Cheat SNGP & 0.009 $\pm$ 0.002 & 0.054 $\pm$ 0.001 & 0.052 $\pm$ 0.001 & 0.022 $\pm$ 0.002 & 94.90 $\pm$ 0.14 & 0.149 $\pm$ 0.001 \\
\bottomrule
\end{tabular}
\\[2em]
\begin{tabular}{l@{\hskip 2em}cc@{\hskip 1.5em}c@{\hskip 1.5em}cc}
\toprule
& \multicolumn{5}{c}{w/ Extra Classes, Scrambled}
\\
\cmidrule(l{-2pt}r{2em}){2-6}
Method & \scriptsize \textbf{ECE-2}& \scriptsize $\E[\hat{v}^\theta]$ & \clap{\scriptsize $\E[(\hat{p}^\theta{-}p)^2]$} & \scriptsize ECE-1 & \scriptsize KL \\
\midrule
Naive NN & 0.521 $\pm$ 0.002 & 0.682 $\pm$ 0.003 & 0.161 $\pm$ 0.002 & 0.068 $\pm$ 0.004 & 0.706 $\pm$ 0.009 \\
NN Ensemble & 0.134 $\pm$ 0.001 & 0.014 $\pm$ 0.000 & 0.148 $\pm$ 0.001 & 0.032 $\pm$ 0.001 & 0.647 $\pm$ 0.003 \\
Evidential DL & 0.387 $\pm$ 0.000 & 0.031 $\pm$ 0.000 & 0.418 $\pm$ 0.000 & 0.794 $\pm$ 0.009 & 2.356 $\pm$ 0.000 \\
SNGP Cov. & 0.112 $\pm$ 0.003 & 0.033 $\pm$ 0.003 & 0.145 $\pm$ 0.001 & 0.057 $\pm$ 0.003 & 0.634 $\pm$ 0.006 \\
Epinet & 0.089 $\pm$ 0.009 & 0.087 $\pm$ 0.009 & 0.163 $\pm$ 0.003 & 0.075 $\pm$ 0.003 & 0.712 $\pm$ 0.010 \\
\midrule
Cheat NN & 0.022 $\pm$ 0.001 & 0.134 $\pm$ 0.001 & 0.154 $\pm$ 0.001 & 0.072 $\pm$ 0.005 & 0.672 $\pm$ 0.005 \\
Cheat SNGP & 0.011 $\pm$ 0.001 & 0.153 $\pm$ 0.002 & 0.150 $\pm$ 0.001 & 0.044 $\pm$ 0.003 & 0.650 $\pm$ 0.006 \\
\bottomrule
\end{tabular}
\end{sc}
\end{small}
\end{center}
\vskip -0.1in
\end{table*}
\begin{table*}[t]
\caption{
Results for the additional CIFAR-10H dataset variants we consider, which include only one of the two increased difficulty types each. Results within 2x of best ECE-2 in bold. All metrics summed across classes except Acc and KL.
}
\label{tab:cifar10h-extra-variants}
\vskip 0.15in
\begin{center}
\begin{small}
\begin{sc}
\begin{tabular}{l@{\hskip 2em}cc@{\hskip 1.5em}c@{\hskip 1.5em}cc@{\hskip 2em}cc@{\hskip 1.5em}c@{\hskip 1.5em}cc}
\toprule
& \multicolumn{5}{c}{Extra Classes}& \multicolumn{5}{c}{Scrambled}
\\
\cmidrule(l{-2pt}r{2em}){2-6}\cmidrule(l{-2pt}r){7-11}
Method & \scriptsize \textbf{ECE-2} & \scriptsize $\E[\hat{v}^\theta]$ & \clap{\scriptsize $\E[(\hat{p}^\theta{-}p)^2]$} & \scriptsize ECE-1 & \scriptsize KL & \scriptsize \textbf{ECE-2}& \scriptsize $\E[\hat{v}^\theta]$ & \clap{\scriptsize $\E[(\hat{p}^\theta{-}p)^2]$} & \scriptsize ECE-1 & \scriptsize KL \\
\midrule
Naive NN & 0.540 & 0.574 & 0.034 & 0.02 & 0.18 & \textbf{0.051} & 0.354 & 0.314 & 0.07 & 0.69 \\
NN Ensemble & 0.020 & 0.007 & 0.027 & 0.03 & 0.15 & 0.261 & 0.028 & 0.289 & 0.04 & 0.64 \\
Evidential DL & 0.386 & 0.031 & 0.417 & 1.14 & 2.34 & 0.561 & 0.068 & 0.629 & 0.97 & 1.59 \\
SNGP Cov. & \textbf{0.017} & 0.030 & 0.026 & 0.03 & 0.15 & 0.271 & 0.014 & 0.285 & 0.06 & 0.63 \\
Epinet & 0.054 & 0.083 & 0.041 & 0.04 & 0.20 & 0.271 & 0.044 & 0.314 & 0.08 & 0.69 \\
\midrule
Cheat NN & \textbf{0.010} & 0.029 & 0.037 & 0.04 & 0.19 & \textbf{0.056} & 0.252 & 0.303 & 0.08 & 0.66 \\
Cheat SNGP & \textbf{0.009} & 0.032 & 0.028 & 0.02 & 0.15 & \textbf{0.029} & 0.280 & 0.286 & 0.05 & 0.62 \\
\bottomrule
\end{tabular}
\end{sc}
\end{small}
\end{center}
\vskip -0.1in
\end{table*}

\subsubsection{Evaluation metrics}

\textbf{ECE-2:} Our primary evaluation metric is the expected second-order calibration error between the predicted epistemic variance and the actual squared difference between the predicted probability and $p(Y|X)$. In other words, we wish to evaluate how closely the following correspondence holds:
\begin{align*}
\E\Big[\big( p(y | X) - \phaty(y | X) \big)^2 ~\Big|~ \vhat(y | X)\Big] \overset{?}{\approx} \vhat(y | X).
\end{align*}
(We expect this correspondence to hold because it is a special case of \cref{thm:conf_cheat_grouping} where the event $A$ is $\vhat(y | X) = c$ for each possible $c \in \R$.)
Specifically, we focus on expected calibration error, which has the form
\begin{align*}
\E\Bigg[
\bigg|
\E\Big[\big( p(y | X) - \phaty(y | X) \big)^2 ~\Big|~ \vhat(y | X)\Big] - \vhat(y | X)
\bigg|
\Bigg]
\end{align*}
We estimate this using expected calibration error over 100 equal-probability bins, based on the quantiles of the predicted probability $\vhat(y | X)$, which we compute as follows:

\begin{enumerate}
    \item For each example $x$, for each class $y$, compute the variance estimate $\vhat(y | x)$, and an unbiased estimate of the squared error $p(y | X) - \phaty(y | X) \big)^2$ using the $K$ annotations for this image $x$:
    \begin{align*}
        \textsc{SqErrEst}(y, \phaty(y|x), [y_i]_{i=1}^K)
        &= \phaty(y|x)^2 - 2 \phaty(y|x) \frac{|\{i : y_i = y\}|}{K} + \frac{|\{(i, j) : i \ne j, y_i = y, y_j = y\}|}{K(K-1)}
        \\&\approx \phaty(y|x)^2 - \phaty(y|x)p(y|x) + p(y|x)^2 = (\phaty(y|x)^2 - p(y|x))^2
    \end{align*}
    $K$ is the number of annotator labels for this image, which is always at least 50 but is sometimes greater.
    When evaluating for the ``extra classes'' variant, we compute the modified estimate
    \begin{align*}
    \textsc{SqErrEst}(y', \phaty(y'|x), [y_i]_{i=1}^K)
    &= \phaty(y'|x)^2 - 2 \phaty(y'|x) p(y'|y)\frac{|\{i : y_i = y\}|}{K} \\&\qquad\qquad+ p(y'|y)^2\frac{|\{(i, j) : i \ne j, y_i = y, y_j = y\}|}{K(K-1)}
    \end{align*}
    where $y'$ is one of the three sub-classes corresponding to the original class $y$.
    \item Sort all of the $(x, y)$ pairs in ascending order of $\vhat(y | x)$. Note that each $x$ appears multiple times in this ordering, due to the different possible labels $y$.
    \item Divide the examples into 100 evenly-sized bins, each of which correspond to an empirical quantile range of 1\%.
    \item Compute the average $\overline{v}_{B_i}$ of $\vhat(y | x)$ for all examples $(x, y)$ in each bin ${B_i}$. Also compute the average $\overline{\textsc{SqErrEst}}_{B_i}$ of the error estimates $\textsc{SqErrEst}(y, \phaty(y|x), [y_i]_{i=1}^K)$ for those same examples.
    \item Let \[T = \frac{1}{N} \sum_{B_i} |B_i| \cdot \left|\overline{v}_{B_i} - \overline{\textsc{SqErrEst}}_{B_i}\right|, \]
    where $N$ is the total number of test set examples and $|B_i|$ is the size of the $i$th bin (approximately N/100).
    \item Return $\textsc{ECE-V} = C \cdot T$ where $C$ is the number of classes.
\end{enumerate}

We scale up the expected calibration error by the number of classes in the last step so that our final metric represents a sum over classes instead of an average over classes, since we usually care about the full vector of predictions rather than the prediction for a single random class.
Note that some papers evaluate expected calibration error for the most likely predicted class only, which avoids this problem \citep{perez2022beyond}. However, it is not obvious that this criterion makes sense when evaluating epistemic uncertainty, especially in the presence of significant aleatoric uncertainty. In principle, we could also compute a separate calibration score for each class and then add them together at the end, instead of averaging over them and then scaling the average, but we choose not to do this because of the small size of our dataset (which would potentially make per-class calibration error estimates very noisy).

We also note that, in general, the binning procedure will under-estimate the exact expected calibration error, although this can be avoided by using the binned outputs instead of the original ones \citep{kumar2019verified}.

\textbf{$\bm{\E[\hat{v}^\theta]}$, $\bm{\E[(\hat{p}^\theta{-}p)^2]}$:}
To compute these values, we use the same estimates from the ECE-V computation, but instead of averaging differences over each bin, we simply compute separate sums of $\vhat(y | x)$ and $\textsc{SqErrEst}(y, \phaty(y|x), [y_i]_{i=1}^K)$. We divide by the total number of examples, so that these numbers again reflect sums over all classes (30 classes in this case).

\textbf{ECE-1:} We also estimate the expected calibration error for the ordinary predictions, again using 100 quantile bins, combining classes together, and scaling up by the number of classes. We apply the same procedure as for ECE-2, except that we use the predicted probability estimates $\phaty(y|x)$ instead of the variance estimates $\vhat(y | x)$, and we use the empirical probability $\frac{|\{i : y_i = y\}|}{K}$ instead of the squared error measurement $\textsc{SqErrEst}(y, \phaty(y|x), [y_i]_{i=1}^K)$.

\textbf{KL:} Finally, we compute the average KL divergence between the empirical probability distribution of the annotator labels and the model predictions
\[
D_{KL}\Big(\, p_{\mathcal{D}}(y|x) \,\Big\|\, \phaty(y|x) \,\Big)
= \frac{|\{i : y_i = y\}|}{K} \left( \log \phaty(y|x) - \log \frac{|\{i : y_i = y\}|}{K} \right)
,
\]
where $p_{\mathcal{D}}(y|x)$ is the distribution of annotator labels for image $x$, e.g.
\[
p_{\mathcal{D}}(y|x) = \frac{|\{i : y_i = y\}|}{K}
\]
where $[y_i]_{i=1}^K$ is the collection of annotator labels for image $x$.

We use this KL divergence metric to tune the hyperparameters of each method.

\clearpage
\subsection{Digits of $\pi$}\label{appendix:sec:digits-of-pi}

\subsubsection{Training data}\label{appendix:sec:digits-of-pi:pcfg}
To construct the queries $X$, we sample digit offsets $I$ according to a mixture of geometric random variables: we sample
\[
Q \sim \text{Uniform}(0.001, 0.1),
\qquad
I \sim \text{Geometric}(Q),
\]
then keep $I$ if it is less than 10,000 and reject it otherwise. We then embed this as a tokenized sequence of the form ``Tell me about digit 0 0 1 4 of pi.'', where $I$ is zero-padded to four digits long.

Given $I$, we then look up the actual $I$th digit after the decimal point in $\pi$ (so digit 1 is 1, digit 2 is 4, digit 3 is 1, digit 4 is 5, etc.).

Depending on the value $d$ of this digit, we then sample responses $Y$ from the following hand-written probabilistic context-free grammar (with mostly randomly-chosen weights):

\begin{empheq}[box={\fboxsep=6pt\fbox}]{align*}
\text{STATEMENT}(d) &\to \text{INTRO}~\text{VALUE}(d) && \text{with probability 0.99}
\\&\to \text{``Reply hazy, try again''} && \text{with probability 0.01}
\\\\
\text{INTRO} &\to \text{``It's''} && \text{with probability 0.138}
\\&\to \text{``It is''} && \text{with probability 0.086}
\\&\to \text{``That's''} && \text{with probability 0.218}
\\&\to \text{``That is''} && \text{with probability 0.185}
\\&\to \text{``Sure, it's''} && \text{with probability 0.096}
\\&\to \text{``Sure, it is''} && \text{with probability 0.02}
\\&\to \text{``Sure, that's''} && \text{with probability 0.17}
\\&\to \text{``Sure, that is''} && \text{with probability 0.087}
\\\\
\text{VALUE}(d) &\to \text{SAY-DIGIT}(d) && \text{with probability 0.56}
\\&\to \text{``an''}~\text{EVEN-ODD}(d)~\text{``number''} && \text{with probability 0.19}
\\&\to \text{``spelled''}~\text{SPELL}(d) && \text{with probability 0.13}
\\&\to \text{``spelled with''}~\text{SPELL-LENGTH}(d)~\text{``letters''} && \text{with probability 0.9}
\\\\
\text{SAY-DIGIT}(d) &\to \text{DIGIT}(d) && \text{with probability 0.616}
\\&\to \text{``the number''}~\text{DIGIT}(d) && \text{with probability 0.384}
\\\\
\text{DIGIT}(d) &\to \text{DIGIT-NAME}(d) && \text{with probability 0.323}
\\&\to \text{DIGIT-VAL}(d) && \text{with probability 0.677}
\end{empheq}

The nonterminals DIGIT-NAME$(d)$, DIGIT-VAL$(d)$, EVEN-ODD$(d)$, SPELL$(d)$, and SPELL-LENGTH$(d)$ depend on the digit:

\begin{itemize}
    \item DIGIT-NAME takes values ``zero'', ``one'', ``two'', \dots
    \item DIGIT-VAL takes values ``0'', ``1'', ``2'', \dots
    \item EVEN-ODD takes values ``even'', ``odd'', ``even'', \dots
    \item SPELL takes values ``Z E R O'', ``O N E'', ``T W O'', \dots
    \item SPELL-LENGTH takes values ``4'' (the number of letters in ``zero''), ``3'' (the number of letters in ``one''), ``3'' (the number of letters in ``two''), \dots
\end{itemize}

Using this context free grammar, we can both sample sequences and look up the true probability of any given sequence. In particular, given a statement, we can look up whether it is in the support of the grammar given the true digit value.

We also maintain a lookup table of semantically-equivalent sentences. When doing this, we treat as equivalent any two sentences that differ only based on how they expanded INTRO, SAYDIGIT, and DIGIT. So, for instance, ``Sure, it's the number 3'' and ``That is three'' are equivalent, and ``It's spelled T H R E E'' and ``Sure, it's spelled T H R E E'' are equivalent, but ``Sure, it's the number three'' and ``It's spelled T H R E E'' are \emph{not} judged equivalent (one is about the value and the other is about the spelling). Similarly ``It's spelled F O U R'' and ``It's spelled with 4 letters'' are not equivalent.

We train our model by concatenating $X$ and two samples $Y_1$, $Y_2$. Before concatenation, we pad $Y_1$ and $Y_2$ out to a constant length, producing examples of the form

\begin{Verbatim}[frame=single,samepage=true]
<BOS> Tell me about digit 0 1 2 6 of pi. <SEP>
That's the number four _ _ _ _ _ _ That's spelled with 4 letters _ _ _ _ _
\end{Verbatim}
\begin{Verbatim}[frame=single,samepage=true]
<BOS> Tell me about digit 0 0 4 8 of pi. <SEP>
Sure, that is an odd number _ _ _ _ It's 5 _ _ _ _ _ _ _ _
\end{Verbatim}
\begin{Verbatim}[frame=single,samepage=true]
<BOS> Tell me about digit 0 0 1 2 of pi. <SEP>
Sure, that's 9 _ _ _ _ _ _ _ It is the number nine _ _ _ _ _
\end{Verbatim}
\begin{Verbatim}[frame=single,samepage=true]
<BOS> Tell me about digit 0 0 1 5 of pi. <SEP>
Sure, that is 3 _ _ _ _ _ _ That is spelled T H R E E _ _
\end{Verbatim}

We use a tabular vocabulary, where each word or letter that could possibly be generated by the above process has its own token index.

\subsubsection{Model architecture and training details}

Our model architecture is a 6-layer causally-masked pre-LayerNorm Transformer \citep{vaswani2017attention,xiong2020layer}, with 8 attention heads, an embedding dimension of 512, an MLP dimension of 2048, a per-head embedding dimension of 64, and fixed sinusoidal positional embeddings.

We train this model for 50,000 training iterations at batch size 1024 using the AdamW optimizer \citep{loshchilov2017decoupled} with 1,000 warmup steps, a maximum learning rate of $2 \times 10^{-5}$, and a cosine decay schedule.
We use the ordinary maximum-likelihood objective, and apply masking so that only the tokens after \texttt{<SEP>} are scored (so the model does not have to learn to predict $X$). Our implementation is in JAX \citep{jax2018github} and uses Optax for optimization \citep{deepmind2020jax}.

\subsubsection{Sampling and filtering}

After training our model, we iterate through all of the digit offsets from 1 to 3000, and sample 120 statements from the model's approximate conditional $\phatym(\cdot | x)$.  We sample at temperature 1 but mask out any tokens with predicted probability less than 0.005 during sampling.
We note that the model has learned to sample pairs $(Y_1, Y_2)$, but we interrupt generation after it has generated $Y_1$; thus each of the 120 statements are drawn independently and identically distributed from $\phatym(\cdot | x)$, not $\phatyc$).

We first check whether each sample was correct, where we judge a sample as correct if it had a nonzero probability under $p(Y|X)$ (the context-free grammar described in \cref{appendix:sec:digits-of-pi:pcfg}).
We then assign scores to each sample:

\begin{itemize}
\item For the total probability ranking, we rank by the probability $\phatym(y_1 | x)$ of the sample under the model.
\item For average token log probability, we divide $\log \phatym(y_1 | x)$ by the length $|y_1|$ of the sample; this is thus an average of the log probabilities of each token, and approximates a ``rate'' of log probability \citep{malinin2020uncertainty}.
\item When clustering into groups of 10, we split the 120 samples for each digit into 12 groups of size 10 each, then assign a score to each sample based on the number of other samples in the same group that were semantically equivalent (according to the criterion in \cref{appendix:sec:digits-of-pi:pcfg}). So, if 4 out of the first 10 samples were semantically equivalent, each of those samples would get a score of 4. (Malformed samples that could not possibly appear under the data distribution are given a score 1.)
\item When clustering into groups of 120, we assign a score to each sample based on the number of other samples among all 120 that were semantically equivalent.
\item When ranking based on our epistemic confidence measure $\Ccheat$, we evaluate $\log \phatym(y | x)$ and $\log \phatyc(y|y, x)$ by concatenating the padded output $y$ with itself when scoring using the model, and separately summing the log-probabilities for $Y_1$ and $Y_2$. We then exponentiate the difference of these probabilities to evaluate $\Ccheat$, and finally assign a score of $-|1 - \Ccheat|$ so that examples closer to 1 have a higher score.
\end{itemize}

\begin{figure*}[p]
    \centering
\includegraphics[width=\linewidth]{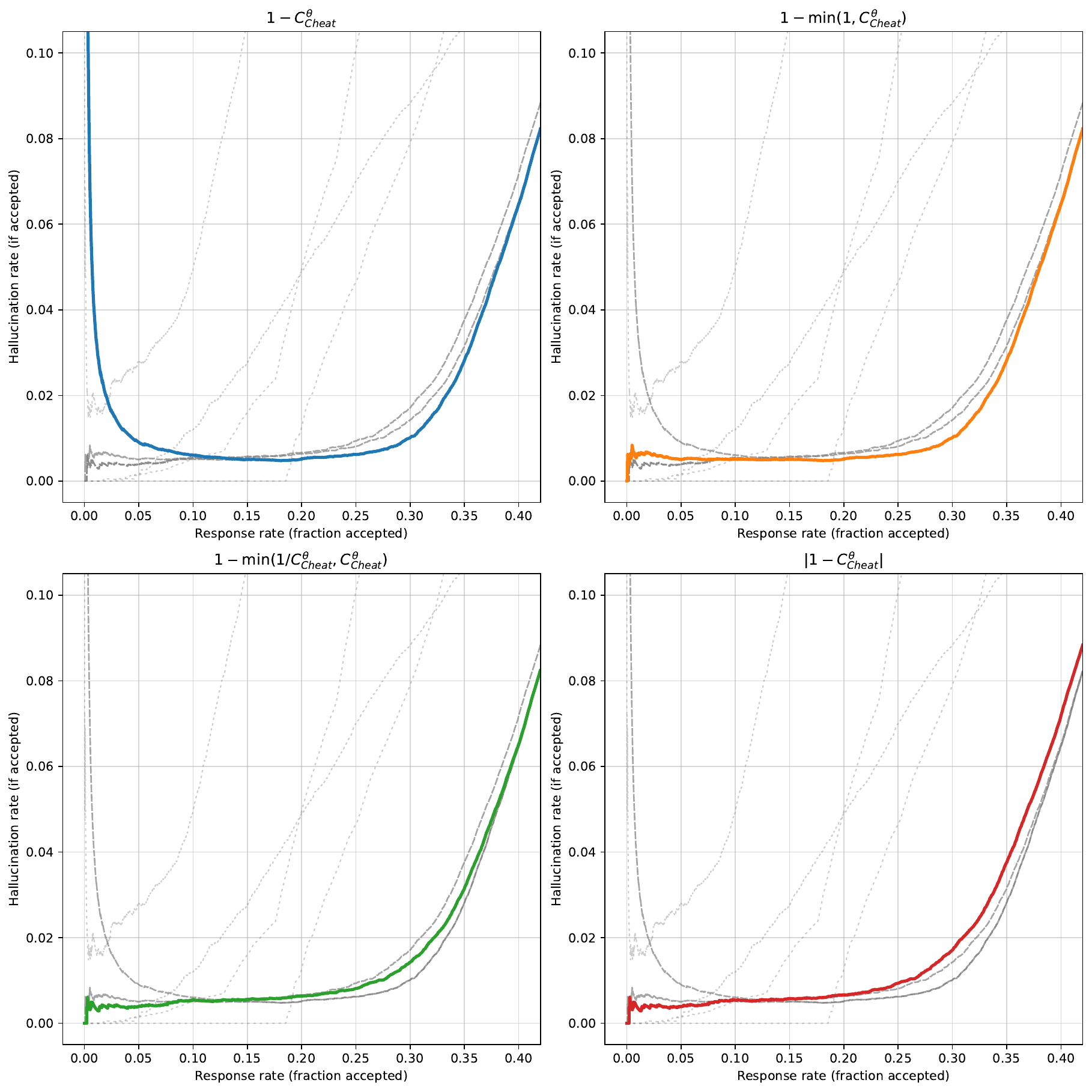}
    \caption{Hallucination rates when ranking by alternative versions of $\Ccheat$, each of which agree with $\Ccheat$ when it is less than 1 but handle $\Ccheat > 1$ differently. For comparison, the original baselines from \cref{fig:digitsofpi-precision-recall} are shown as dotted lines.}
    \label{appendix:fig:pi_rank_alternatives}
\end{figure*}

We additionally explored a number of alternative ranking strategies for the case where $\Ccheat > 1$. According to \cref{thm:conf_bounded}, $\Ccheat$ will never be greater than 1 if $\phatyy$ is calibrated, so our theoretical results do not provide any particular guidance for how to rank these samples. We plot four options in \cref{appendix:fig:pi_rank_alternatives}: $1 - \Ccheat$ (the simplest, but nonsensical when $\Ccheat$ is very large), $|1 - \Ccheat|$ (used in the main paper results), $1 - \min\{1, \Ccheat\}$, and $1 - \min\left\{\Ccheat, \frac{1}{\Ccheat}\right\}$. Using $1 - \Ccheat$ alone leads to high hallucination rates when using very strict thresholds, because the only samples that are kept are the outliers.

\subsubsection{Investigating model samples}

For each score type, we sort the samples (randomizing in the case of ties), and then compute the running hallucination rate (fraction of samples seen so far that actually had $\pgt(y|x) = 0$) and response rate (total fraction of samples seen so far) over prefixes of the sorted ordering; the results are shown in \cref{fig:digitsofpi-precision-recall} (right) in the main paper.

To get a better understanding of the relationship between the model's behavior, its accuracy, and its cheat-corrected epistemic confidence, we conduct a deeper study of the scores and accuracies assigned to each of the digits.

We start by visualizing some samples drawn from the model directly, colored based on the log-probability of each token. In \cref{appendix:fig:pi_samples_independent_1,appendix:fig:pi_samples_independent_2} (in \cref{appendix:sample_visualizations}), we show samples of $Y_1$ and $Y_2$ generated by the model. Note that the samples of $Y_2$ are not actually used under our uncertainty-quantification scheme. However, visualizing these samples reveals an interesting behavior: when querying digits that the model does not know, the samples $Y_1$ often show ``hallucinations'' of plausible facts, but the corresponding $Y_2$ is almost always \emph{consistent} with $Y_1$; the two samples do not contradict one another. This means that the model has learned to ``cheat'' well; it is able to condition on the information in $Y_1$ to make a more consistent prediction of $Y_2$. There are, however, a few exceptions where $Y_1$ and $Y_2$ are inconsistent or incoherent; in this case we find that $Y_2$ tends to be more correct than $Y_1$.

In \cref{appendix:fig:pi_samples_scored_1,appendix:fig:pi_samples_scored_2} (in \cref{appendix:sample_visualizations}), we next visualize the log-probabilities of each token when reuse the sampled $Y_1$ sequences as $Y_2$, by concatenating each $Y_1$ with itself; this is how we compute our confidence scores $\Ccheat(y|x)$. We see that conditioning on $Y_1$ does not usually significantly alter the probability of tokens with aleatoric variation (e.g. the initial stylistic tokens), but usually raises the probability of tokens with epistemic prediction error (e.g. tokens that state facts about the digit). This means the difference between the two log probabilities is usually a good indicator of the unknown parts of the original samples.

There are a few samples which show the opposite pattern, where the likelihood \emph{decreases} in the second sample (in the last row of each figure). This seems to occur when the originally sampled $Y_1$ was incorrect and had a very low probability, whereas the prediction for $Y_2$ was more accurate.
This pattern of an incorrect $Y_1$ but correct $Y_2$ is an indication of \emph{miscalibration} with respect to the paired outcomes $(Y_1, Y_2)$: if the model was truly conditioning on some property $\Phi(X)$ of the input query, its predictions on $Y_1$ should be just as accurate as its prediction of $Y_2$, and it should never predict an inconsistent $Y_1, Y_2$ pair that cannot occur under the data distribution. These samples tend to interfere with our epistemic confidence metric, and result in predicted confidences greater than 1 (and negative predicted variances).

\begin{figure*}[t]
    \centering
\includegraphics[width=\linewidth]{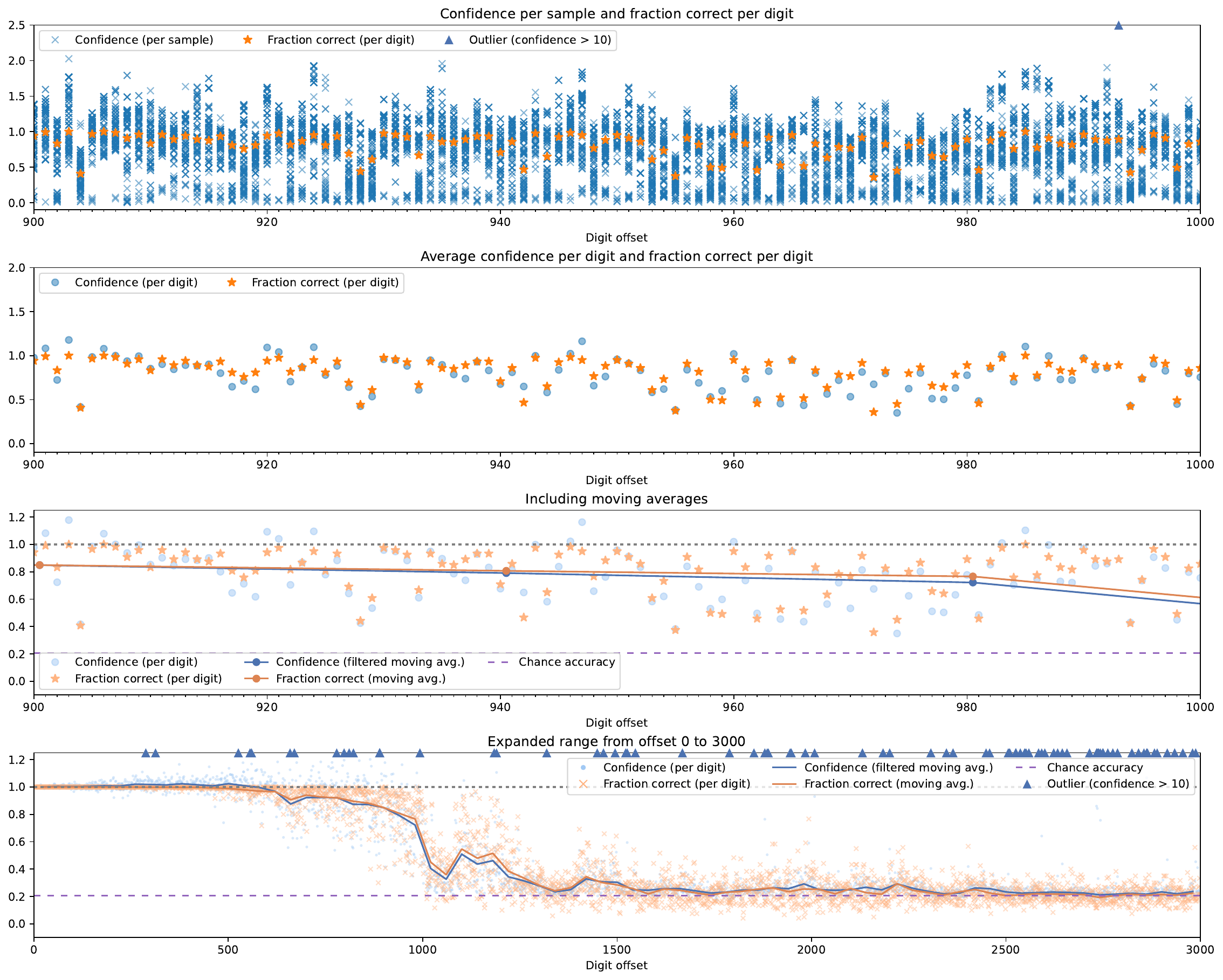}
    \vspace{-2em}
    \caption{Predicted epistemic confidence closely tracks the model's actual accuracy after removing outliers.
Row 1: We compute the confidence for each of the 120 samples for each digit, and also compute the fraction of those 120 samples that are correct. Row 2: We take the average confidence over all of the samples for each digit, ignoring outliers. Row 3: We divide digit offsets into groups of 40 digits, and compute the average of both confidence and fraction correct. Row 4: Zoomed out version of row 3, showing the full sequence.}
    \label{appendix:fig:pi_by_digit_detailed_steps}
\end{figure*}

\subsubsection{Relationship between confidence and correctness}

We also investigate the relationship between the epistemic confidence and the correctness of generated samples. This is somewhat complex, because each individual sample is either correct or not correct, and the model may assign different confidences to each sample. Determining whether the model's confidence was appropriate thus requires some sort of aggregation.

\begin{figure*}[t]
    \centering
\includegraphics[width=\linewidth]{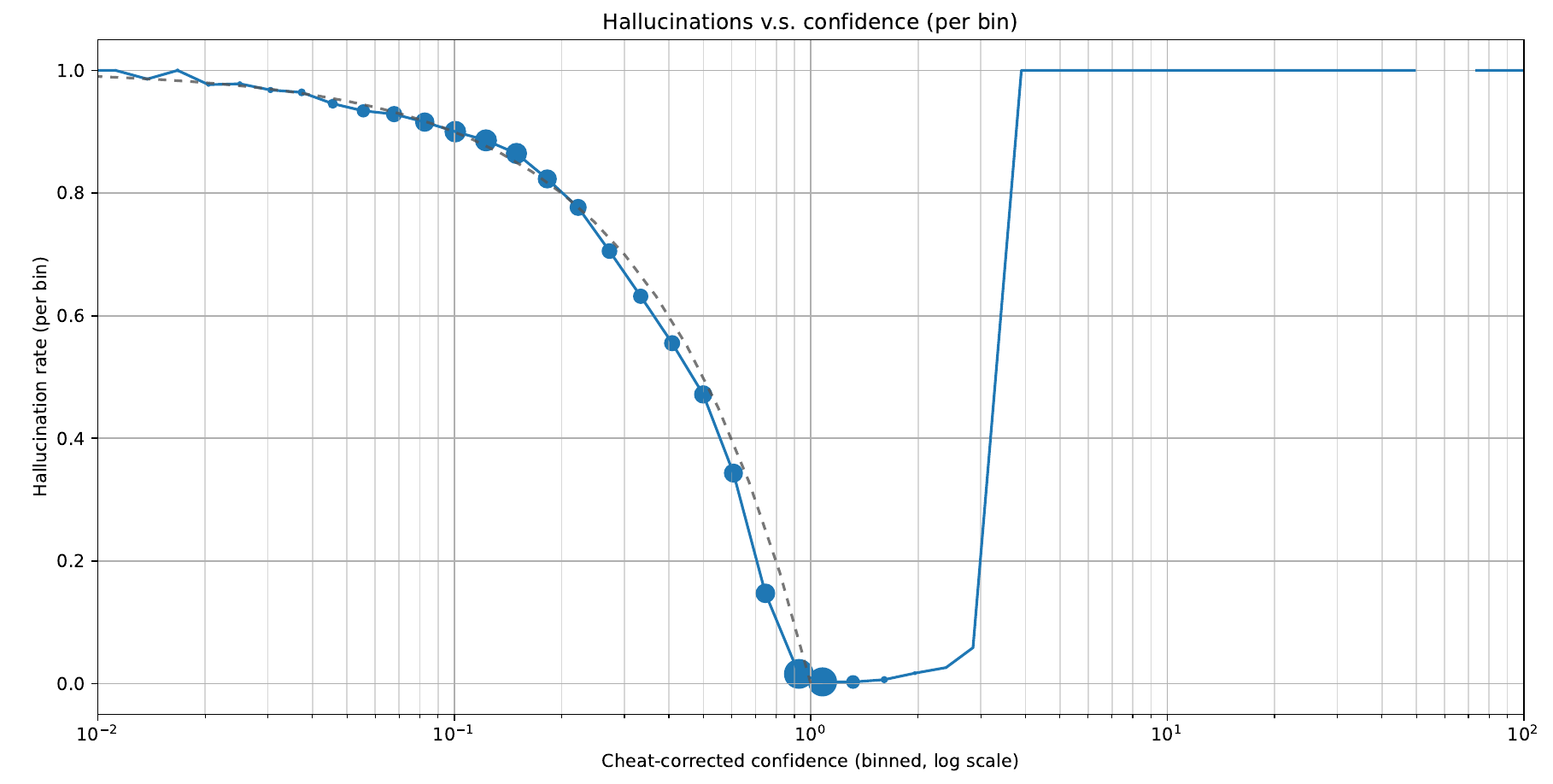}
    \vspace{-2em}
    \caption{Expanded version of \cref{fig:digitsofpi-precision-recall} (left), which shows the hallucination rate of the samples in different confidence bins, with dot size proportional to the number of samples with that confidence. We use a logarithmic scale for $\Ccheat$ to include the outliers which are much larger than 1; the prediction of \cref{thm:conf_hallucination} is shown as a dashed line. We see that confidences between 1 and 2 usually indicate a correct answer, but this quickly falls off, and most samples with extremely large $\Ccheat$ values are usually incorrect.}
    \label{appendix:fig:pi_reliability_log}
\end{figure*}

\cref{appendix:fig:pi_reliability_log} shows an expanded version of the confidence-vs-hallucination-rate plot in \cref{fig:digitsofpi-precision-recall} (left). Overall, when the confidence is less than 1, the quantity $1 - \Ccheat$ is a good estimate of hallucination rate. Samples with confidence close to 1 tend to be correct, but samples with extremely large confidence values tend to be wrong. 
We find that most of these outliers are due to having an incorrect $Y_1$ but fixing the mistake in $Y_2$, as detailed in the previous section.

To show how confidence varies based on digit offset, we additionally aggregate these across nearby digits, since we know from the data distribution that nearby digits appear with similar frequency and should thus be similarly difficult. We compute the fraction of samples that are correct for each digit, and compare it to the average confidence of the samples. To get a meaningful aggregate confidence measurement, we need to throw out samples with extremely large confidences due to miscalibration. The results are shown in \cref{appendix:fig:pi_by_digit_detailed_steps}.

\clearpage
\subsection{Offline RL - Frozen Lake}

\subsubsection{Environment and Expert Policies}

\begin{figure*}[t]
    \centering
\includegraphics[width=0.5\linewidth]{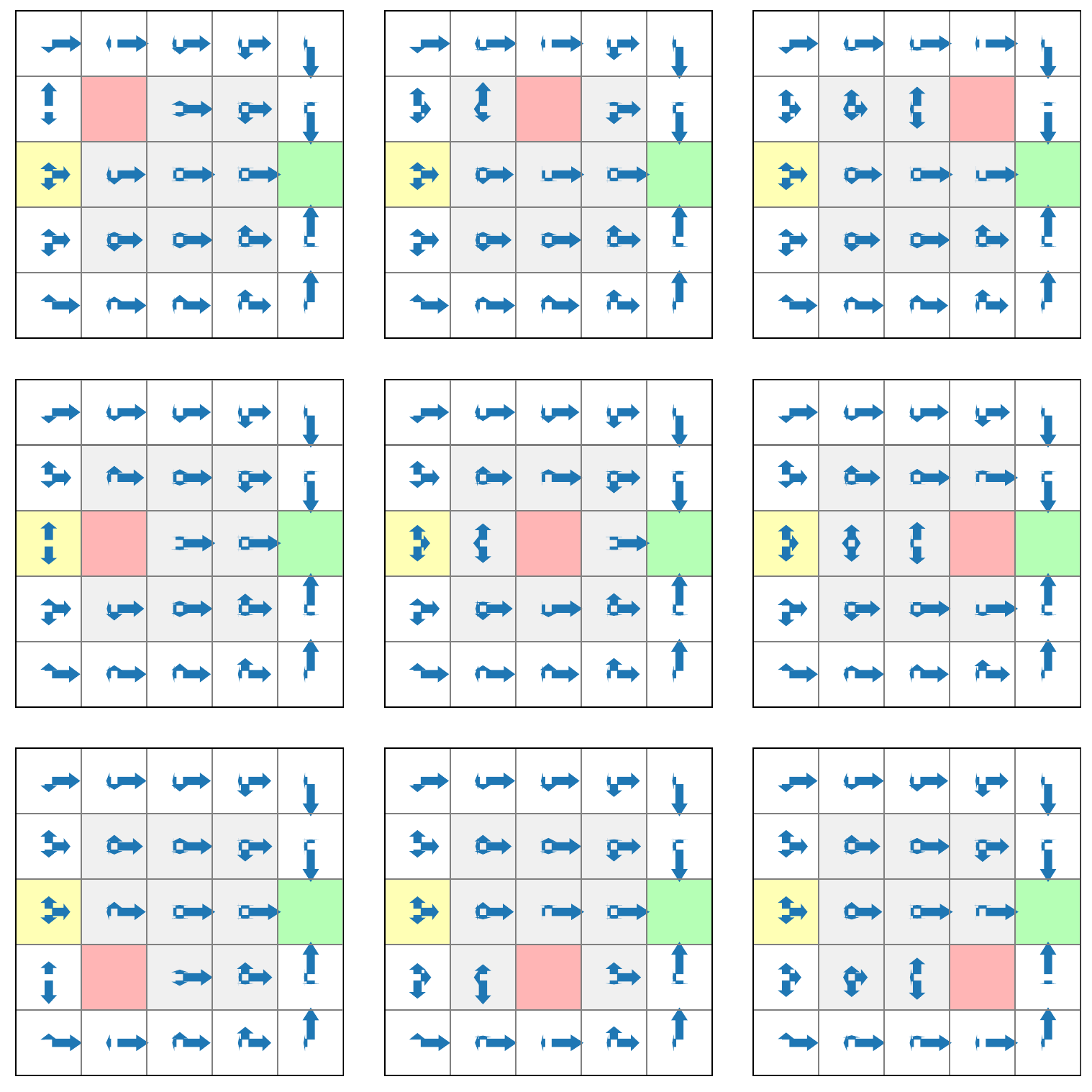}
    \caption{Expert policies for the frozen lake task, depending on the location of the unsafe patch. Arrow length is proportional to probability of taking that action in each state.}
    \label{appendix:fig:frozen_lake_expert}
\end{figure*}

We represent the Frozen Lake environment (\cref{fig:frozenlake}) as a graph, where each gridworld cell is a node in the graph, and there are four outgoing connections from each node to the adjacent nodes in each direction. Costs of each transition are determined by the destination state:
\begin{itemize}
    \item Moving onto the goal state (green square) gives a reward of 40 and ends the episode.
    \item Moving onto one of the non-lake squares (white border) gives a reward of -3.
    \item Moving onto one of the eight non-central lake squares gives a reward of -5.
    \item Moving onto the center lake square gives a reward -10.
\end{itemize}
These costs were chosen so that cutting across the center of the lake gives slightly more reward than going around it ($-5 - 10 - 5 + 40 = 20$ vs $-3 \times 7 + 40 = 19$), but only slightly.

For each of the nine possible locations of the unsafe patch, we then solve for the optimal entropy-regularized tabular policy by iterating the soft Q-learning Bellman backup operator \citep{schulman2017equivalence} with a discount rate of 0.9 and a temperature of 2.5. We explicitly disallow moving onto the unsafe patch or leaving the bounds of the gridworld by assigning $-\infty$ reward to each; the resulting expert policy never takes those actions.

The resulting expert policies for each of the 9 possible unsafe patch locations are shown in \cref{appendix:fig:frozen_lake_expert}.

\subsubsection{Imitation Learning}

Our model architecture for this task is a 12-layer causally-masked pre-LayerNorm Transformer \citep{vaswani2017attention,xiong2020layer} based on GPT-2 \citep{radford2019language}, with 12 attention heads, an embedding dimension of 768, an MLP dimension of 3072, a per-head embedding dimension of 64, and fixed sinusoidal positional embeddings.

For each example, we first tokenize the model's view of the environment, using a single token per gridworld cell. In 50\% of the examples, we mark the unsafe region with a token "C" identifying it, and the safe parts of the lake with "I". In the other 50\%, we mark all potentially-unsafe regions with a "?". We also include tokens for the start state ("S"), goal state ("G") and border states ("P"). (Since these tokens are constant across all examples, they are likely not important to include, but we include them for simplicity.)

Loosely inspired by the setup of \citet{chen2021decision}, we next represent the expert trajectories as sequences of states and actions. We do not tokenize the rewards, since our objective is simply to imitate the expert trajectories. We also concatenate \emph{two} independent samples together, padding them to a maximum length of 16 steps each.

This produces examples of the following form (with each word mapped to its own token index, and padding denoted by ``\texttt{\_}''):

\begin{Verbatim}[frame=single,samepage=true]
P P P P P 
P I I I P 
S I I C G 
P I I I P 
P P P P P 
<SEP> c0 r2 down c0 r3 up c0 r2 down c0 r3 up c0 r2 down c0 r3 down c0 r4 right
c1 r4 up c1 r3 right c2 r3 right c3 r3 right c4 r3 up FINISH _ _ _ _ _ _ _ _ _ _
_ _ _
<SEP> c0 r2 right c1 r2 right c2 r2 down c2 r3 up c2 r2 down c2 r3 left c1 r3
down c1 r4 right c2 r4 right c3 r4 right c4 r4 left c3 r4 right c4 r4 up c4 r3
up FINISH _ _ _ _ _ _ _
\end{Verbatim}

\begin{Verbatim}[frame=single,samepage=true]
P P P P P 
P ? ? ? P 
S ? ? ? G 
P ? ? ? P 
P P P P P 
<SEP> c0 r2 right c1 r2 right c2 r2 right c3 r2 right FINISH _ _ _ _ _ _ _ _ _ _
_ _ _ _ _ _ _ _ _ _ _ _ _ _ _ _ _ _ _ _ _ _ _ _ _ _ _
<SEP> c0 r2 up c0 r1 right c1 r1 down c1 r2 up c1 r1 down c1 r2 up c1 r1 up c1
r0 down c1 r1 down c1 r2 right c2 r2 right c3 r2 right FINISH _ _ _ _ _ _ _ _ _
_ _ _ _
\end{Verbatim}

We train by maximizing log-likelihood under a standard autoregressive training setup. We only train it to imitate the sequences of states and actions, by masking out all tokens prior to the \texttt{<SEP>} token.
We train this model for 50,000 training iterations at batch size 512 using the AdamW optimizer \citep{loshchilov2017decoupled} with 1,000 warmup steps, a maximum learning rate of $2 \times 10^{-5}$, and a cosine decay schedule.
Our implementation is in JAX \citep{jax2018github} and uses Optax for optimization \citep{deepmind2020jax}.

We then sample trajectories from the model at temperature 0.9, conditioning on either a full or partial view of the environment, and compute the cheat-corrected epistemic confidence for each. For our ``cheat-corrected rejection sampling'' decoding strategy, we reject any sample with $|1 - \Ccheat| > 0.05$ and resample it; rejected samples are shown as dashed lines in \cref{fig:frozenlake} in the main paper. For our ``cheat-corrected top-1 search'' decoding strategy, we sample 6400 samples, then identify the sample $y$ with the largest predicted probability $\phatym(y|x)$ subject to the constraint $|1 - \Ccheat| \le 0.05$.

We show additional trajectories sampled by our model, along with their confidences, in \cref{appendix:fig:frozen_lake_samples_r1c1,appendix:fig:frozen_lake_samples_r3c2,appendix:fig:frozen_lake_samples_hidden} (in \cref{appendix:sample_visualizations}).

\end{document}